\documentclass[twoside,11pt]{article}

\usepackage{blindtext}

\usepackage{jmlr_style}
\usepackage[nohyperref]{jmlr2e}

\usepackage{lastpage}
\jmlrheading{27}{2026}{1-\pageref{LastPage}}{5/25}{8/26}{25-1020}{
Ming Xiang,
Stratis Ioannidis,
Edmund Yeh,
Carlee Joe-Wong
and Lili Su}

\ShortHeadings{Efficient and Unbiased Federated Learning}{Xiang, Ioannidis, Yeh, Joe-Wong and Su}
\firstpageno{1}

\usepackage{booktabs}       %
\usepackage{amsfonts}       %
\usepackage{nicefrac}       %
\usepackage{microtype}      %
\usepackage{xcolor}         %
\usepackage{colortbl}
\usepackage{float}%
\usepackage{titletoc}
\usepackage{standalone} %
\usepackage{adjustbox} %
\usepackage{overpic}    %
\usepackage{wrapfig}
\usepackage{anyfontsize}
\usepackage[colorlinks=true,
    linktoc=all]{hyperref}       %

\hypersetup{ hidelinks }
\newcommand{\FedAPM}{{\tt FedSWE}}
\newcommand{\MIFA}{{\tt MIFA}}
\newcommand{\FedAvg}{{\tt FedAvg}}

\newcommand{\gFedAvg}{{\tt gFedAvg}}
\newcommand{\FedVARP}{{\tt FedVARP}}
\newcommand{\SCAFFOLD}{{\tt SCAFFOLD}}
\newcommand{\Wt}[1]{W^{(#1)}}
\newcommand{\tF}{\tilde F}
\usepackage{pifont}

\newcommand{\TDt}[1]{{\tilde \Delta}^{#1}}
\newcommand{\Dt}[1]{\Delta^{#1}}
\newcommand{\DFt}[1]{\nabla \bm{F}_{\x}^{#1}}
\newcommand{\variance}{\sf var}

\newcommand{\tP}{P_{\delta}}
\newcommand{\rp}{\rho_k}
\newcommand{\rrp}{{\tilde \rho}}
\newcommand{\FedAU}{\texttt{FedAU}}
\newcommand{\FAST}{\texttt{F3AST}}

\usepackage{enumitem}
\setitemize{noitemsep,topsep=0pt,parsep=0pt,partopsep=0pt}

\SetKwComment{MyComment}{$\triangleright$~}{}  %
\newcommand{\COMMENT}[1]{\MyComment*[r]{\textcolor{blue}{\small #1}}}
\SetCommentSty{rmfamily}       %
\SetKw{KwComment}{}              %
\SetAlgoNlRelativeSize{0}        %
\SetSideCommentLeft{}            %
\SetSideCommentRight{}           %
\newcommand{\LCOMMENT}[1]{\KwComment{\textcolor{brown}{\bf $\bigstar$~{\small #1}}}}

\begin{document}
\title{
Resilience Beyond Stationary Client Unavailability: \\
Unlocking Efficient and Unbiased Federated Learning
\thanks{A preliminary version of this work \citep{xiang2024efficient} was presented at the 38th Annual Conference on Neural Information Processing Systems, Vancouver, Canada.}}

\author{\name Ming Xiang$^1$ \email xiang.mi@northeastern.edu \\
        \name Stratis Ioannidis$^1$ \email e.ioannidis@northeastern.edu \\
        \name Edmund Yeh$^1$ \email e.yeh@northeastern.edu \\
        \name Carlee Joe-Wong$^2$ \email cjoewong@andrew.cmu.edu \\
        \name Lili Su$^1$ \email l.su@northeastern.edu \\
        {\addr $^1$Northeastern University, Boston, MA USA};\\
        {\addr $^2$Carnegie Mellon University, Pittsburgh, PA USA}.
        }

\editor{Zhaoran Wang}

\maketitle

\begin{abstract}%
Due to resource constraints or external and internal uncertainties, clients in real-world federated learning systems are often intermittently available edge devices.
In highly dynamic environments, the parameter server lacks prior real-time knowledge of clients' availability, making it challenging to adapt traditional federated learning algorithms to be resilient to uncertainties in client availability.
If not carefully addressed, complex client availability can introduce significant bias, potentially harming the performance of the trained model.  
Most prior work either fails to account for non-stationary client availability dynamics or demands significant memory and computational overhead.
This paper aims to develop efficient federated learning algorithms that are provably resilient to heterogeneous and non-stationary stochastic client availability.  
We propose~\FedAPM, which admits novel algorithmic structures to
(i) compensate for missed computations, 
(ii) stabilize and diffuse the global updates over rounds,
and (iii) evenly mix the local updates through implicit gossiping, despite being agnostic to non-stationary dynamics.
Compared with the standard~{\tt FedAvg},~\FedAPM~introduces light additional memory and computation overhead. 
We show that~\FedAPM~converges to a stationary point of non-convex objectives while achieving the desired linear speedup property in certain special cases.
We corroborate our analysis with numerical experiments over diversified client unavailability dynamics on real-world data sets. 
\end{abstract}


\section{Introduction}
\label{sec: intro}
Federated learning is a distributed machine learning framework that enables training global models without disclosing raw local data 
\citep{mcmahan2017communication,kairouz2021advances}. 
It has been adopted in commercial applications such as autonomous vehicles 
\citep{chen2021bdfl,zeng2022federated,peng2023privacy}, %
internet of things \citep{nguyen2019diot}, and natural language processing \citep{yang2018applied,ramaswamy2019federated}. 

Heterogeneous data and massive client populations are two of the defining characteristics of cross-device federated learning systems \citep{mcmahan2017communication,kairouz2021advances,mclaughlin2025pfda}.
Despite intensive efforts \citep{mcmahan2017communication,Li2020,yuan2022,ruan2021towards,kairouz2021advances}, 
several key challenges that arise from the involvement of large-scale client populations are often overlooked in the existing literature \citep{perazzone2022communication}.
One of the primary hurdles is the issue of intermittent client unavailability.
Intuitively, more active clients drive the global model to their local optima, biasing the training.  %
In addition, the higher the uncertainty in client unavailability, the larger the performance degradation. 
Concrete examples that confirm these intuitions 
can be found in Section \ref{sec: obj inconsistency}.  
Client unavailability issues can arise from internal factors such as
different working schedules %
and heterogeneous hardware/software constraints.
External factors, such as poor network coverage and frequent handovers of base stations due to fast movements, only exacerbate these problems \citep{tse2005fundamentals,wen2024communication,Ye2022TSP,bonawitz2019towards,kairouz2021advances}. The intricate interplay of internal and external factors results in non-stationary and heterogeneous client unavailability.

There is a recent surge in the study of client unavailability \citep{li2020federated,yang2022anarchic,wang2022,wang2023lightweight,pmlr-v202-cho23b,gu2021fast,yan2023federated,crawshaw2024federated}. 
Despite their solid foundation in this direction, 
most prior work either assumes exact knowledge of the clients' availability or requires their dynamics to be benignly stationary \citep{mcmahan2017communication,li2020federated,perazzone2022communication,wang2022,wang2023lightweight,crawshaw2024federated}. 
The non-stationarity in client unavailability remains largely underexplored.
A related line of work studies asynchronous federated learning wherein clients are vulnerable to delays in message transmission, and the reported model updates may be stale \citep{xie2019asynchronous,nguyen2022federated,toghani2022unbounded,koloskova2022sharper}.
However, the proposed methods assume the availability of all clients or uniformly sampled clients, making them infeasible for dynamic and complex client availability in practice. 
A handful of other works \citep{gu2021fast,jhunjhunwala2022fedvarp,yan2023federated} memorize the old gradients of unavailable clients. 
However, the added memory burdens the federated learning system with substantial memory proportional to the product of the number of clients and the model dimension.

\begin{wrapfigure}[11]{r}{0.45\textwidth} 
\centering
\includegraphics[width=\linewidth, trim=2cm 6.6cm 3cm 4.4cm,clip]{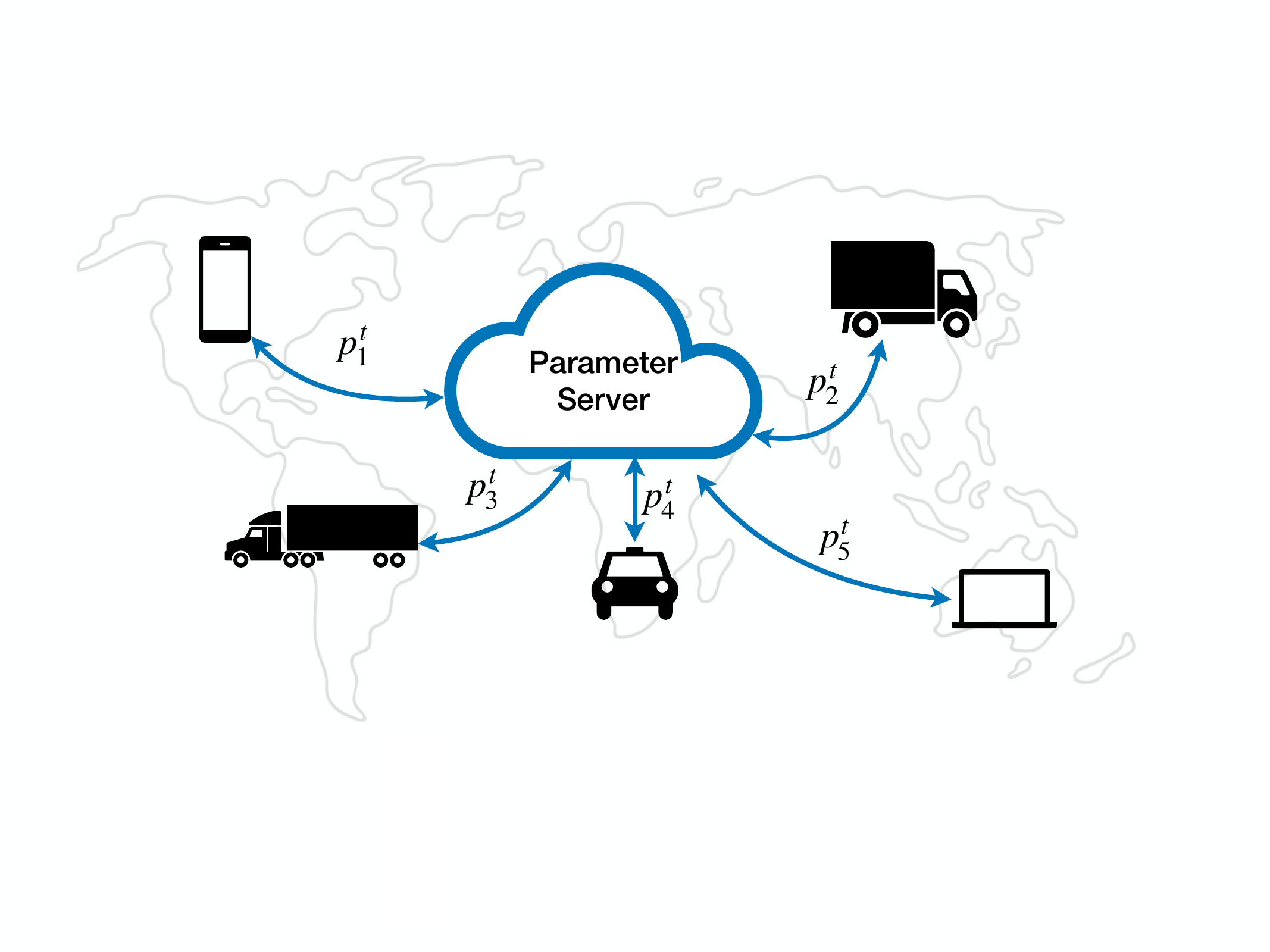}
\vskip -0.5\baselineskip 
\caption{\footnotesize 
An illustration of heterogeneous and non-stationary availability. 
}
\label{fig: system setup}
\end{wrapfigure} 
We adopt the commonly-used %
stochastic client unavailability model~\citep{mcmahan2017communication,wang2022,jhunjhunwala2022fedvarp,perazzone2022communication,wang2023lightweight},
where each client $i$ is available for federated learning training with probability $p_i^t$ in round $t$. The $p_i^t$'s are   
heterogeneous across clients and are subject to unknown and non-stationary dynamics.  
An example can be found in~\prettyref{fig: system setup}. 
\citep{xiang2025empowering} marks our first move towards understanding heterogeneity and non-stationarity in $p_i^t$'s but focuses on a significantly simpler problem where $p_i^t$'s are used to describe the uplink communication failures---it requires that clients be capable of continuous local optimization regardless of failures. 
In addition, it imposes the assumption that $p_i^t \ge \delta$, where $\delta>0$ is an absolute constant.

Relaxing the requirement of continuous local computation, 
a preliminary version of this work 
\citep{xiang2024efficient} studies %
client unavailability yet still imposes the technical assumption that $p_i^t\ge \delta$. 
In this extended version, we generalize
the dynamics of $p_i^t$ by allowing it to take zero values occasionally. 
Our generalized setup is motivated by the real-world scenario in which clients located in different geographical regions may experience availability issues due to time zone differences, naturally causing $p_i^t$'s to drop to zero from time to time~\citep{zhu2022diurnal}. 
When clients participate independently, our generalized model of $p_i^t$ covers some popular existing models as special cases such as  
$p_i^t = p$~\citep{li2020federated,yang2022anarchic}, 
$p_i^t = p_i$~\citep{wang2023lightweight}, 
regularized participation~\citep{wang2022,crawshaw2024federated} 
and cyclic participation~\citep{pmlr-v202-cho23b}.
Details can be found in~\prettyref{sec: problem formulation}.

\noindent{\bf Contributions}. 
Our contributions are four-fold:
\begin{itemize}[leftmargin=*]
\item 
We demonstrate in~\prettyref{sec: obj inconsistency}, using concrete examples in the context of~\FedAvg~- the most widely adopted federated learning algorithm, that both heterogeneity and non-stationarity of $p_i^t$ can result in bias and thus significant performance degradation of~\FedAvg. 
\item 
We propose a computational and memory-efficient algorithm~\FedAPM~in~\prettyref{sec: algorithm FedAPM}. 
At a high level, the design of~\FedAPM~introduces three novel algorithmic components: 
\begin{itemize}
    \item[(i)] {\bf adaptive innovation echoing}, 
    which helps clients catch up on missed computation;
    \item[(ii)] {\bf global moving average}, 
    which stabilizes and diffuses the global updates over rounds;
    \item[(iii)] {\bf implicit gossiping}, which facilitates a balanced information mixture through implicit client-client gossip, ultimately correcting residual bias.
\end{itemize}
Notably, no direct neighbor information exchanges are involved, and the client unavailability dynamics remain undisclosed to all clients and the parameter server. 
\item 
In~\prettyref{sec: convergence analysis}, we show the convergence of~\FedAPM, %
which exhibits the desired linear speedup property in certain special cases. 

\item 
In~\prettyref{sec: numerical}, 
we validate our analysis with numerical experiments over diversified client unavailability dynamics on real-world data sets.
\end{itemize}

\section{Related Work}
\label{sec: related work} 
\subsection{\bf Dynamical client availability} 
There is a recent surge of efforts to study client availability
\citep{ruan2021towards,ribero2022federated,chen2022optimal,jhunjhunwala2022fedvarp,wang2022,wang2023lightweight,perazzone2022communication,xiang2023towards,crawshaw2024federated}, which can be roughly classified into two categories depending on whether the parameter server can unilaterally determine the participating clients.

\noindent(i) {\bf Controllable availability.}
Earlier research~\citep{mcmahan2017communication,Li2020,jhunjhunwala2022fedvarp} presumes that, in each round, the parameter server could recruit a small set of clients either uniformly at random or in proportion to the volume of local data held by clients.  
More recently, 
\cite{cho2022towards} design adaptive and non-uniform client sampling to accelerate learning convergence, albeit at the cost of introducing a non-zero residual error.
In another work, \cite{pmlr-v202-cho23b} study the convergence of~\FedAvg~with cyclic client participation.
Yet, the set of available clients is sampled uniformly at random per cyclic round and is chosen unilaterally and deliberately by the parameter server.
\cite{perazzone2022communication} consider heterogeneous and time-varying response rates $p_i^t$ under the assumptions that $p_i^t$ is known a priori and that the stochastic gradients are bounded in expectation. 
Furthermore, the values of $p_i^t$ are determined by the parameter server by solving a stochastic optimization problem.
\cite{chen2022optimal} propose a client sampling scheme wherein only the clients with the most ``important" updates communicate back to the parameter server. 
This sampling method can achieve performance comparable to that of full client participation, provided that $p_i^t$ is globally known to both the parameter server and the clients. 
Departing from this line of literature, our setup neither assumes any side information or prior knowledge of the probability $p_i^t$ nor assumes that the parameter server has any influence on $p_i^t$'s. 

\noindent(ii) {\bf Uncontrollable availability.}
There is a handful of work on building resilience against arbitrary client availability \citep{ribero2022federated,wang2022,yan2023federated,gu2021fast,yang2022anarchic,wang2023lightweight,crawshaw2024federated}.  
\cite{ribero2022federated} consider random client availability whose underlying probabilities are also heterogeneous and time-varying with unknown dynamics. 
However, the underlying dynamics of $p_i^t$'s in \citep{ribero2022federated} are assumed to follow a homogeneous Markov chain. 
\cite{wang2022} propose a generalized~\FedAvg~that amplifies parameter updates every $P$ rounds for some carefully tuned $P$. 
Despite its elegant unified analysis and potential to accommodate non-independent unavailability dynamics, 
to reach a stationary point, $p_i^t$ needs to satisfy some assumptions to ensure roughly equal availability of all clients over every $P$ rounds.
Sharing a similar spirit, \cite{crawshaw2024federated} propose a~\SCAFFOLD~variant that amplifies global parameter and local gradient updates every $P$ round.
In spite of its communication efficiency, ability for correlated participation, and resilience to data heterogeneity, the rolling average of $p_i^t$ over every $P$ round is assumed to be the same constant for all clients.
\cite{yang2022anarchic} analyze a setting where clients participate in the training at their will. 
Yet, their convergence is shown to be up to a non-zero residual error. 
The algorithms proposed in \citep{gu2021fast,yan2023federated} share the same idea of using the memorized latest gradient updates from unavailable clients for global aggregation. 
Despite superior numerical performance, 
both algorithms demand a substantial amount of additional memory \citep{wang2023lightweight}. 
For non-convex objectives, both \citep{yan2023federated} and \citep{gu2021fast} require an absolute bounded inactive period, and share similar technical assumptions such as almost surely bounded stochastic gradients \citep{yan2023federated} or almost surely bounded gradient noise \citep{gu2021fast}. 
Though bounded inactive periods are relevant for applications wherein the sensors wake up on a periodic schedule, this assumption is not satisfied even for the simple stochastic setting when clients are selected uniformly at random.  
A recent work~\citep{wang2023lightweight} considers unknown heterogeneous $p_i$'s yet assumes $p_i$'s are fixed over time. 
Another concurrent work \citep{sun2025debiasing} characterizes periodic client participation through the lens of the Markov chain.
In spite of its resilience to non-uniform and correlated availability, $p_i^t$'s are also assumed to be stationary over the window of participation.

\subsection{\bf Asynchronous federated learning}
Another related line of work is asynchronous federated learning. 
To the best of our knowledge, \cite{xie2019asynchronous} initialize the study of asynchronous federated learning, 
wherein the parameter server adjusts the global model every time it receives an update from a client.  
Convergence is shown under some technical assumptions such as weakly-convex global objectives, bounded delay, and bounded stochastic gradients.  
\cite{nguyen2022federated} propose FedBuff, which uses additional memory to buffer asynchronous aggregation to achieve scalability and privacy. 
Convergence is shown under bounded gradients and bounded staleness assumptions. 
In fact, 
most convergence guarantees in the asynchronous federated learning literature rely on bounded staleness \citep{xie2019asynchronous,nguyen2022federated,toghani2022unbounded,koloskova2022sharper}, or bounded gradients \citep{xie2019asynchronous,nguyen2022federated,koloskova2022sharper}.
Recently, arbitrary delay is considered in the context of distributed SGD with bounded stochastic gradients and $(0, \zeta)$-bounded inter-client heterogeneity \citep{mishchenko2022asynchronous} (see Assumption~\ref{ass: bounded similarity} therein for the definition).  
The convergence suffers from a non-zero residual term $O(\zeta^2)$. 
In contrast, our convergence guarantee is free from non-zero residual terms and does not require gradients or staleness to be bounded.

\vspace{\baselineskip}
\noindent{\bf Notations.}
Let $\norm{\bm{v}}$, $\fnorm{A}$ and $\lambda_2(B)$ define the $l_2$ norm of a vector $\bm{v}$, 
the Frobenius norm of a matrix $A$,
and the second largest eigenvalue of a squared matrix $B$, respectively.
Denote $\calF^t$ the sigma algebra generated by randomness up to round $t$,
$\reals^d$ a $d$-dimensional vector space,
and $[m]$ a set $\sth{k : k\in \naturals, 1 \le k \le m]}$.
$\indc{\calE}$ is an indicator function of an event $\calE$, \ie, $\indc{\calE} = 1$ when event $\calE$ occurs, but $\indc{\calE} = 0$ otherwise.
For two functions $f(n)$ and $g(n)$,
we have $f(n) \lesssim g(n)$, if there exists a constant $c_{o} > 0$ and an integer $n_{o} \in \naturals$ such that $f(n) \le c_{o} g(n)$ for all $n \ge n_{o}$,
while $f(n) \asymp g(n)$, if there exists a constant $c_{\theta} > 0$ and an integer $n_{\theta} \in \naturals$ such that $f(n) = c_{\theta} g(n)$ for all $n \ge n_{\theta}$.

\section{Problem Formulation}
\label{sec: problem formulation}
A federated learning system consists of a parameter server and $m$ clients to  collaboratively minimize
\begin{align}
\label{eq: global obj}
\min\limits_{\x\in\reals^d} F(\x) = \frac{1}{m}\sum_{i=1}^m F_i(\x),
\end{align} 
where $F_i(\x) \triangleq \expects{\ell_i(\x;\xi_i)}{\xi_i \sim \calD_i}$ is the non-convex local objective,
$\calD_i$ is the local distribution,
$\xi_i$ is a stochastic sample that client $i$ has access to,
$\ell_i$ is the local loss function,
and $d$ is the model dimension. 

We use~\prettyref{ass: prob lower bound} to %
formally describe the non-stationary and heterogeneous client availability that we consider in this paper. 
Let $\calA^t $ denote the set of %
available clients,
and $T$ be the number of total training rounds.

\begin{assumption}
\label{ass: prob lower bound}
Define $p_i^t \triangleq \Expect[\Indc_{\{i\in\calA^t\}}]$.
For any given $\delta \in (0,1]$, there exists a window of $P$ rounds such that 
\begin{align}
    \label{eq: prob lower}
    \frac{1}{P} \sum_{t = n P}^{(n+1) P - 1} p_i^t \ge \delta, ~~ \forall n \in \naturals,
\end{align}
where the events $\{i\in\calA^t\}$ are independent across clients $i$ and across rounds $t$. 
Let $\calP_{\delta}$ be the collection of all $P$'s that satisfy~\eqref{eq: prob lower}, and $P_{\delta} \triangleq \min \calP_{\delta}$. 
\end{assumption}
\prettyref{ass: prob lower bound} requires that the averaged $p_i^t$'s over $P$ consecutive rounds are non-trivially lower bounded away from zero while allowing  
$p_i^t$'s occasionally drop to zeros.
Note that if \eqref{eq: prob lower} holds for some $P_0$, then it must also hold for $c \cdot P_0$, where $c\in \naturals$. 
Hence, we will focus on %
$P_{\delta}$ to eliminate ambiguity. 
Next, we elaborate on the generality of Assumption \ref{ass: prob lower bound} in~\prettyref{rmk: arbitrary dynamics}. 
\begin{remark}
\label{rmk: arbitrary dynamics}
Within the context that $\{i\in \calA^t\}$ are independent across agents $i$ and across rounds $t$, 
\prettyref{ass: prob lower bound} generalizes many existing client availability assumptions. 
\begin{itemize}[leftmargin=*]
\item 
When $P_{\delta}=1$, \eqref{eq: prob lower} reduces to $p_i^t \ge \delta$,
which encompasses the cases of uniform availability $p_i^t=p$~\citep{li2020federated,yang2022anarchic}, 
stationary availability  $p_i^t = p_i \ge \min_{i\in [m]} p_i$~\citep{wang2023lightweight}
and non-stationary availability $p_i^t \ge \delta$~\citep{xiang2024efficient,xiang2025empowering}.
\item 
When $P_{\delta}>1$,~\prettyref{ass: prob lower bound} encompasses the following dynamics:  
\begin{itemize}[leftmargin=*]
\item[(i)] {\em Regularized participation}~\citep{wang2022,crawshaw2024federated}: There exists $\mu_{t_0} \ge 1/P_{\delta}$ such that
    \begin{align}
    \label{eq: regularized dynamics}
    \frac{1}{P_{\delta}}
    \sum_{t=n \tP}^{(n+1) \tP -1}
    \prob{i\in\calA^t}
    =
    \mu_{n},~ \text{for} ~ i\in [m]~\text{and}~ \forall n \in \naturals,  
    \end{align}
    where $\tP$ is some carefully chosen integer.
    \eqref{eq: regularized dynamics} says that every client becomes available equally often within each window.
    By contrast, our~\prettyref{ass: prob lower bound} does not require such ``balance'' in $\mu_{n}$.
    On the other hand, outside the restriction that $\{i\in \calA^t\}$ are independent across $i$ and $t$, the regularized participation assumption~\citep{wang2022,crawshaw2024federated} allows clients' participation to be correlated within $P_{\delta}$ consecutive rounds. 
    \item[(ii)] {\em Cyclic participation}~\citep{pmlr-v202-cho23b,crawshaw2024federated}: Suppose that the $m$ clients can be equally partitioned into $\bar{K}$ groups. %
    Denote the group index of client $i$ as $G(i) \in \{1, \cdots, \bar{K}\}$. 
    In each round $t$, $p_i^t$ is determined as 
    \begin{align}
    \label{eq: cyclic dynamics}
    p_i^t & = 
    \begin{cases}
        p,~& \text{if}~ (t ~ \text{mod}~ \bar{K}) = G(i);\\
        0,~& \text{otherwise}.
    \end{cases}
    \end{align}
    Cyclic participation characterizes the scenario where client groups become available in order.
    For example, if clients are populated in different time zones around the globe, their availability naturally exhibits cyclic patterns due to time differences.
    On the other hand, outside the restriction that $\{i\in \calA^t\}$ are independent across $i$ and $t$, cyclic participation also admits correlated participation. %
    \cite{sun2025debiasing} generalize~\eqref{eq: cyclic dynamics} by allowing clients to have different $p_i$'s rather than a homogeneous $p$.
\end{itemize}
\end{itemize}

\vspace{\baselineskip}
Independent client unavailability is a widely adopted assumption in federated learning research \citep{li2020federated,Li2020,karimireddy2020scaffold,yang2021achieving,yang2022anarchic,wang2023lightweight}, and it is also the focus of this paper. We aim to extend our approach to address correlated client unavailability in future work. 
Analyzing non-independent availability with uncertain probabilistic trajectories in~\prettyref{ass: prob lower bound} is in general challenging. 
For example, the involved entanglement of stochastic gradient and availability statistics fundamentally complicates the theoretical analysis.
We conjecture that independent participation may be only for the technical convenience of our analysis.
Our experiments in~\prettyref{sec: numerical} suggest that the proposed algorithms offer notable improvements even when the clients' participation is correlated.
\end{remark}

\section{Heterogeneity and Non-stationarity May Lead to Significant Bias} %
\label{sec: obj inconsistency}
\begin{wrapfigure}[13]{r}{0.38\textwidth} 
\centering
\vspace*{-\baselineskip}
\resizebox{\linewidth}{!}{
\includegraphics[width=\linewidth,trim=0 0 0 0.2cm, clip]{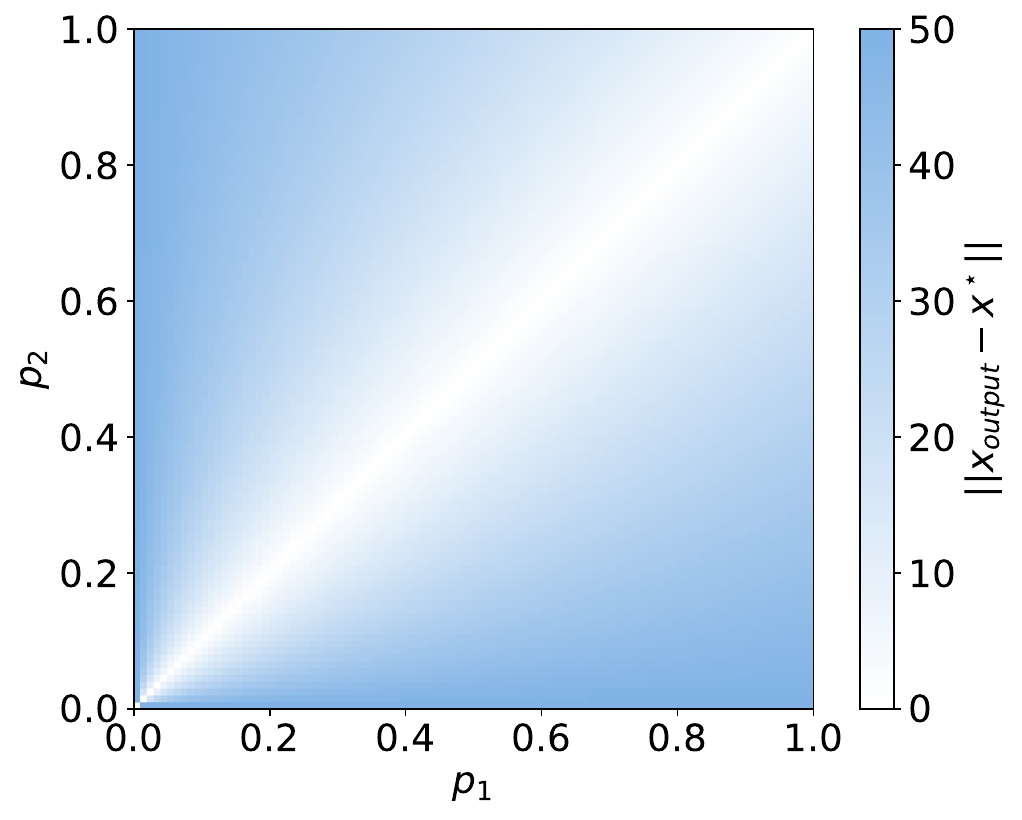} }
\vskip -0.6\baselineskip
\caption{\footnotesize 
Let $x_{\text{output}}\triangleq \lim_{t\diverge} \expect{x^t}$.
Under most of the choices of $p_1, p_2$, $x_{\text{output}}$ is far from $x^*$. %
}
\label{fig: obj inconsistency quadratic}
\end{wrapfigure}
In this section, we illustrate the impacts of heterogeneity and non-stationarity of client availability under the classic~\FedAvg. 
We use two examples to showcase 
the significant bias incurred.

\begin{example}[Heterogeneity]
\label{example: obj shift}
Suppose that $m=2$ and $p_i^t = p_i$ for $i \in [2]$. 
Let $F_i\pth{x} \triangleq \norm{x - u_i}^2 / 2$, where $x, u_i\in \reals$.   
The global objective~\eqref{eq: global obj} is  
\begin{align}
F\pth{x} = \frac{1}{2} (\norm{x - u_1}^2 + \norm{x - u_2}^2),
\label{eq: counterexample global objective}
\end{align}
with unique minimizer 
$x^\star = (u_1+u_2)/ 2$.  
Let $u_1=0$ and $u_2=100$. 
\prettyref{fig: obj inconsistency quadratic} 
illustrates %
how the heterogeneity 
in $p_i$ affects the expected output of~\FedAvg.
\end{example}

\begin{example}[Non-stationarity]
\label{example: non stationarity}
In~\prettyref{fig: motivating example non-stationary},
a total of $m = 100$ clients perform an image classification task on the SVHN data set \citep{netzer2011readingdigits} under the~\FedAvg~algorithm,
whose local data set distribution follows $\mathsf{Dirichlet}(0.1)$ \citep{hsu2019measuring}.
Clients become available with probability $p_i^t = p \cdot [\gamma \cdot \sin (0.1 \pi \cdot t) + (1 - \gamma)],~\forall i \in [m]$.
The hyperparameter details are deferred to~\prettyref{app: numerical}.
Observations can be found in the caption.
\begin{figure}[H]
    \centering
    \begin{subfigure}[b]{.45\textwidth}
    \includegraphics[width=\linewidth]{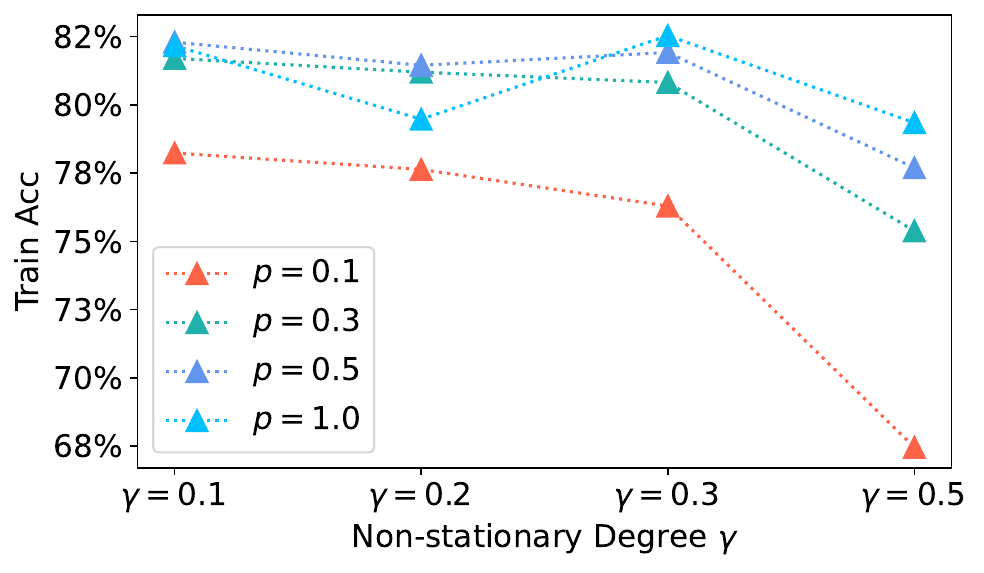}
    \vspace{-1.5em}
    \caption{\footnotesize Train accuracy.}
    \end{subfigure}
    \begin{subfigure}[b]{.45\textwidth}
    \includegraphics[width=\linewidth]{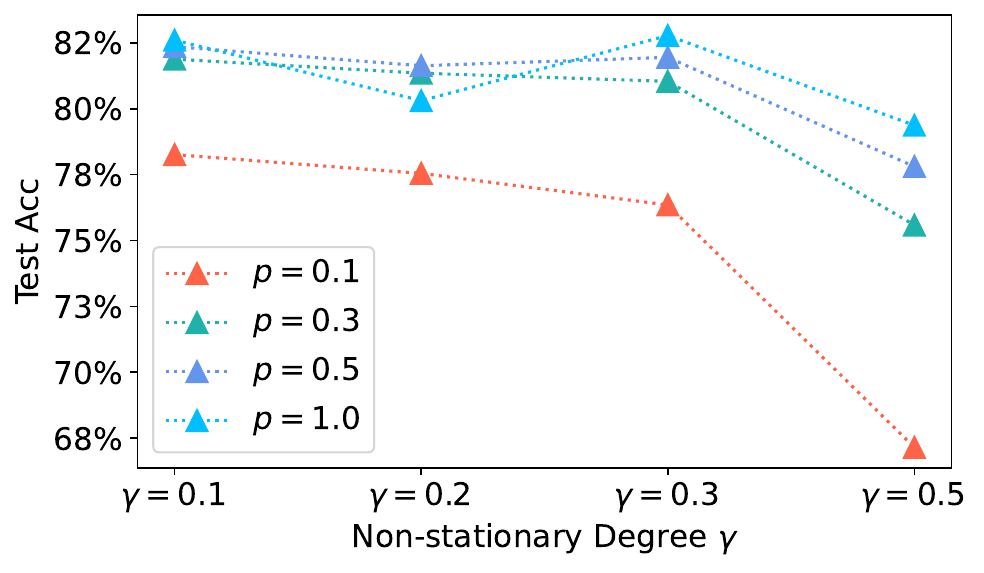}    
    \vspace{-1.5em}
    \caption{\footnotesize Test accuracy}
    \end{subfigure}
    \vspace{-0.5em}
    \caption{\footnotesize
    Train and test accuracy results in percentage (\%).
    In particular,
    the parameter $\gamma$ signifies the degree of non-stationary.
    Notice that,
    as the client availability becomes more non-stationary
    (a larger $\gamma$),
    \FedAvg~experiences a significant drop in accuracy.
    For example,
    both the train and test accuracies drop by over $10\%$ when $p=0.1$,
    and $\gamma$ increases from $0.1$ to $0.5$.
    }
    \label{fig: motivating example non-stationary}
\end{figure}
\end{example}

\section{Federated Stabilized Agile Weight Re-Equalization (\FedAPM)}
\label{sec: algorithm FedAPM}
In order to minimize~\eqref{eq: global obj}, it is natural to have the entire client population perform the same number of local updates and mix these updates carefully to ensure that they are weighted equally.
However, due to intermittent availability, clients may miss computations in certain rounds and, as a result of the heterogeneity in $p_i^t$, are unable to contribute an equal number of local updates.  
An alternative approach to equalizing the number of local updates is to have clients catch up by performing their missed local computations immediately when they become available. 
However, this approach requires a daunting amount of resources and may not be feasible due to hardware or software constraints. 
Formally, recall that $\calA^t$ is the set of available clients at time $t$.  
Let $\tau_i(t) := \{t^{\prime}:~ t^{\prime}<t ~\text{and}~ i\in \calA^{t^{\prime}}\}$ denote the most recent (with respect to time $t$) round that client $i$ is available.
Compared with standard~\FedAvg, the naive ``catch-up'' procedure will consume $\pth{t-\tau_i(t)-1} \cdot s$ local stochastic gradient descent updates and $(t - \tau_i(t)-1)$ additional stochastic samples, where $s$ is the number of local updates per global round when a client is available in standard~\FedAvg.
In this work, we target computation-light algorithms that, compared with~\FedAvg, 
adjust local updates by $O(1)$ additional computation per client without additional stochastic samples. 

We propose {\bf Fed}erated {\bf S}tabilized Agile {\bf W}eight Re-{\bf E}qualization 
(\FedAPM), which is formally described in~\prettyref{alg: fedpbc+}.
It involves three novel algorithmic structures:
{\em adaptive innovation echoing}, {\em global moving average} and {\em implicit gossiping}. 
At a high level, these novel algorithmic structures
(i) help clients catch up on the missed computation,
(ii) stabilize and diffuse the global updates over rounds by
interpolating between the fresh local updates and the most recent global update,
and (iii) enable a balanced information mixture through implicit client-client gossip, ultimately correcting the remaining bias.

\begin{figure}[!t]
\centering
\begin{minipage}{\textwidth}
\resizebox{\linewidth}{!}{\setlength{\columnsep}{1.8cm}
\begin{algorithm}[H]
\caption{
{\bf Fed}erated {\bf S}tabilized Agile {\bf W}eight Re-{\bf E}qualization (\FedAPM)
}
\label{alg: fedpbc+}
\textbf{Inputs:} 
$T$, 
$s$, 
$k$,
$\eta_l$, 
$\eta_g$, 
$\x^0$. 

\LCOMMENT{Initializations.}

\lFor{$i\in [m]$}
{
$\x_i^{0} \gets \x^0$ and $\tau_i (0) \gets -1$
}
\For{$t=0, \cdots, T-1$}
{
\LCOMMENT{On the available clients.}

\For{$i\in \calA^t$}
{
    $\x_i^{(t,0)} \gets \x_i^{t}$\;
    \For{$k=0, \cdots, s-1$}
    {$\x_i^{(t, k+1)} \gets \x_i^{(t, k)} -  \eta_l \nabla \ell_i(\x_i^{(t, k)}; \xi_{i}^{(t,k)})$
    \COMMENT{Client local SGD}   
    } 
        $\bm{G}_i^{t} \gets \x_i^{(t,0)} - \x_i^{(t, s)}$\;
        $\x_i^{t\dagger} \gets \x_i^{(t, 0)} - \eta_g
        (t - \tau_i(t)) \bm{G}_i^t$
        \COMMENT{Client adaptive innovation echoing (\prettyref{sec: adaptive innovation echoing})}
        $\tau_i(t+1) \gets t$
        \COMMENT{Client round index counter update (\prettyref{sec: adaptive innovation echoing})}
    Report $\x_i^{t\dagger}$ to the parameter server
    } 
\LCOMMENT{On the parameter server.}

$\x^{t+1} \gets \frac{1}{\abth{\calA^t} + k} 
\pth{\sum_{i\in \calA^t}\x_i^{t\dagger} + k \x^t}
$
\COMMENT{Global moving average (\prettyref{sec: moving global average})}
Multicast $\x^{t+1}$ to clients $i \in \calA^t$
\COMMENT{Postponed multicast (\prettyref{sec: implicit gossiping})}
\LCOMMENT{On all clients.}

\For{$i\in [m]$}{
    \uIf{$i\in \calA^t$}{$\x_i^{t+1} \gets \x^{t+1}$\;}
    \uElse{$\x_i^{t+1} \gets \x_i^{t}$\; $\tau_i(t+1) \gets \tau_i(t)$\;}}
    }
\end{algorithm}
}
\end{minipage}
\end{figure}
In~\prettyref{alg: fedpbc+}, 
each client keeps two local variables $\x_i$ and $\tau_i$, along with a few auxiliary variables used in updating $\x_i$ and $\tau_i$. 
The server keeps tracking the most recent global update $\x^t$.
The algorithm's inputs are rather standard: 
total training rounds $T$, 
local and global learning rates $\eta_{l}$ and $\eta_g$, 
the number of local updates per round $s$, 
the interpolation coefficient $k$,
and the initial model $\x^0$.  
In each round $t$, in lines 6-10, similar to~\FedAvg, an available client $i\in \calA^t$ 
performs $s$ steps of stochastic gradient descent
on its local model $\x_i^t$, where $\nabla \ell_i (\cdot; \xi_i^{(t,k)})$ is the stochastic gradient of sample $\xi_i^{(t,k)}$. 
Next, we describe the novel algorithmic structures used in~\FedAPM. 

\subsection{Adaptive innovation echoing}
\label{sec: adaptive innovation echoing}
Departing from~\FedAvg,~wherein the local estimate $\x_i^t$ is updated as 
$\x_i^{t\dagger} \gets \x_i^{(t, 0)} - \eta_g
\bm{G}_i^t$.   
\FedAPM~``echos'' the local innovation $\bm{G}_i^t$ by multiplying it by $(t - \tau_i(t))$ (lines 12-13).
Intuitively, this simple echoing helps approximately equalize the number of local improvements, as formally stated in~\prettyref{prop: similar speed}. 
It says that the total number of innovations echoed is the same for all active clients for any given round. 
\begin{proposition}
\label{prop: similar speed}
For any $R\in \naturals$, if $i \in \calA^{R-1}$, then $\sum_{t=0}^{R-1} \indc{i \in \calA^t} \pth{t - \tau_i(t)} = R.$
\end{proposition}

\subsection{Global moving average}
\label{sec: moving global average}
When client availability is highly dynamic, the set of active clients $\calA^t$ can vary greatly from round to round, leading to unsteady global updates.
Our extensive experiments in~\prettyref{sec: numerical P>1} confirm that most of the state-of-the-art methods experience significant fluctuations during training. 
The early version of our algorithm (\ie,~\prettyref{alg: fedpbc+} with $k=0$) presented in the conference paper~\citep{xiang2024efficient} is no exception and has a similar level of variability.

When $k>0$,~\prettyref{alg: fedpbc+} adaptively interpolates between the fresh local updates $\sum_{i\in\calA^t} \x_i^{t\dagger}$ and the most recent global update $\x^t$ in line 17, where the interpolation coefficients are jointly decided by the number of active clients $\abth{\calA^t}$ and the parameter $k$.  
For ease of exposition, we restate the interpolation as follows: 
\begin{align}
    \label{eq: interpolation}
    \x^{t+1} =  
    \frac{1}{\abth{\calA^t} + k} 
    \sum_{i\in \calA^t} \x_i^{t\dagger}
    +
    \pth{1 - \frac{1}{\abth{\calA^t} + k}}
    \x^{t}.
\end{align}
Intuitively, as $k$ increases, the interpolation produces a smoother curve by emphasizing more on the most recent global update.
However, the budget for increasing $k$ is not unlimited.
In particular, when $k \rightarrow \infty$, the global update $\x^{t}$ duplicates the global model from the last round, preventing effective learning from occurring. 
Specifically, unrolling the recursion, we have $\x^{t+1} = \x^{0}$ for all $t \ge 0$, i.e., the global model $\x^{t}$ is not updated at all.
In~\prettyref{sec: numerical}, we will show that $k = \Theta(m)$ is generally a reasonable empirical choice.
We conjecture that discrete $k$'s are only necessary for the technical convenience of our analysis, and leave it as a future work on how to analyze continuous $k$ theoretically.

Furthermore, we want to note that~\eqref{eq: interpolation} is closely related to global momentum and model exponential moving average, yet the interpolation coefficient is neither static nor decaying over rounds. 
Therefore, the existing theoretical analysis for momentum with static coefficient~\citep{reddi2019convergence,li2023convergence,cheng2024momentum} and model exponential moving average with decaying coefficient~\citep{ahn2024adam} is inapplicable to our problem.
Beyond stabilizing training, we will show in~\prettyref{sec: convergence analysis} that interpolation is also necessary to strengthen global information diffusion across rounds. Details can be found therein.

\subsection{Implicit gossiping}
\label{sec: implicit gossiping}
In~\FedAPM, the parameter server does not send the most recent global model to the active clients at the beginning of a global round. 
Instead, the parameter server aggregates the locally updated models $\x_i^{t\dagger}$ through \eqref{eq: interpolation} and sends the new global model $\x^{t+1}$ to all active clients $\calA^t$ (lines 20-26).
By postponing multicasting the shared global model, 
the active clients in $\calA^t$ {\em implicitly} gossip
their updated local models with each other through the parameter server \citep{xiang2023towards,xiang2024efficient}.
Though the postponed multicasting brings in staleness, 
we will show that the staleness is bounded in~\prettyref{lmm: geo second moment main text}. %
In addition, our empirical results (\prettyref{tab: slowdown supp} in~\prettyref{app: numerical}) suggest that there is no significant slowdown when compared to vanilla~\FedAvg.

Gossip-type algorithms were originally proposed for peer-to-peer networks and are well-known for their agility to communication failures and asynchronous information exchange in achieving average consensus  
\citep{degroot1974reaching,boyd2006randomized,kempe2003gossip,Hajnal58,Lynch:1996:DA:2821576,nedic2009distributed}. 
Intuitively, the clients' local estimates are eventually equally weighted in the final algorithm output.
Note that, departing from the standard gossiping protocols therein \citep{kempe2003gossip,shah2009gossip}, information exchange in~\FedAPM~does not involve direct client-client communication.

\section{Convergence Analysis}
\label{sec: convergence analysis}
\subsection{Assumptions}
In this section,
we analyze the convergence of~\FedAPM.
We start by stating regulatory assumptions that are common in federated learning analysis \citep{li2020federated, wang2020tackling, karimireddy2020scaffold}.

\begin{assumption}
\label{ass: 2 smmothness}
Each local objective function
$\nabla F_{i}(\x)$ is $L$-Lipschitz,
\ie,
\[
    \norm{\nabla F_{i}(\x_1)-\nabla F_{i}(\x_2)}\le L \norm{\x_1-\x_2},
    ~\forall \x_1,~\x_2,~\text{and}~\forall~i\in [m].
\] 
\end{assumption}
\begin{assumption}
\label{ass: bounded variance client-wise}
Stochastic gradients $\nabla \ell_i(\x;\xi)$ are unbiased with bounded variance, \ie,  
\[  
    \expect{\nabla \ell_i(\x;\xi) \mid \x}=\nabla F_i(\x)
    ~\text{and}~
    \expect{\norm{\nabla \ell_i(\x;\xi)-\nabla F_i(\x)}^2 \mid \x} \le \sigma^2,
    ~\forall~
    i\in[m].
\] 
\end{assumption}

\begin{assumption}
\label{ass: bounded similarity}
The divergence between local and global gradients is bounded
for $\beta,~\zeta\ge 0$
such that 
\begin{align}
\label{eq: BG condition}
    \frac{1}{m}\sum_{i =1}^m \norm{\nabla F_i(\x)- \nabla F(\x)}^2 \le \beta^2 \norm{\nabla F(\x)}^2+ \zeta^2. 
\end{align}
\end{assumption}
When the local data sets are homogeneous, 
$\nabla F_i(\x) = \nabla F(\x)$ holds for any client $i \in \calR$, 
resulting in $\beta = \zeta = 0$. 
\prettyref{ass: bounded similarity} and its variants in~\prettyref{tab: limitation of existing work} are often referred to as bounded gradient divergence to characterize data heterogeneity across clients.  
It can be easily checked that our~\prettyref{ass: bounded similarity} is more relaxed or equivalent to the variants therein. 
\begin{table}[!t]
    \centering
    \caption{\small Popular variant assumptions on gradient dissimilarity.}
    \label{tab: limitation of existing work}
    \centering
    \begin{tabular}{ c  | m{8cm} }
    \toprule
    {\bf Bounded Gradient Dissimilarity} &  
    \multicolumn{1}{c}{\bf References}  \\
    \midrule     
     $\max_{\x}\norm{\nabla F_i(\x)}^2 \le \zeta^2, ~ \forall ~i\in [m]$ & \citep{Li2020,yu2019parallel,cho2022towards,cho2023communication,yan2023federated}. \\ 
    \midrule 
    $ \frac{1}{m}\sum_{i=1}^m \norm{\nabla F_i(\x)}^2 \le \beta^2 \norm{\nabla F(\x)}^2 $ & 
    \citep{li2019feddane,li2020federated} \\ 
    \midrule
    $ \frac{1}{m}\sum_{i=1}^m \norm{\nabla F_i(\x) - \nabla F(\x)}^2 \le \zeta^2$ & \citep{yu2019linear,wang2019adaptive,huang2022lower,karimireddybyzantine22,wang2022,wang2023lightweight,yang2022anarchic,wang2022matcha,allouah2023fixing}. \\
    \midrule
    $\frac{1}{m}\sum_{i=1}^m \norm{\nabla F_i(\x)}^2 \le \beta^2 \norm{\nabla F(\x)}^2 + \zeta^2$ & \citep{karimireddy2020scaffold,wang2020tackling,wang2021cooperative,gu2021fast,yuan2022}. \\
    \bottomrule
    \end{tabular}
\end{table}

\subsection{Augmented Learning Systems}
\label{subsec: prelim}
For ease of analysis, it is technically convenient to consider an augmented learning system with $k$ virtual clients, and to show convergence of $\x$ through this system. 
Observing that the update in \eqref{eq: interpolation} can be rewritten as 
\begin{align*}
\x^{t+1} =  
\frac{1}{\abth{\calA^t} + k} 
\sum_{i\in \calA^t} \x_i^{t\dagger}
+
\frac{1}{\abth{\calA^t} + k} \sum_{i = m+1}^{m+k}   
\x^{t}.    
\end{align*} 
More specifically, we construct the augmented learning system as follows:
let $\calV = \{m+1, \cdots, m+k\}$ with each element representing a virtual client; it holds that $\x_i^{t\dagger} = \x^t$ for each $i\in \calV$.
Let $F_i(\x) \triangleq 0$ for $i\in \calV$ and $\forall \x \in \reals^d$. 
It is easy to see that, with this local objective function, we have 
\begin{align*}
\x^{t+1} =  
\frac{1}{\abth{\calA^t} + k} 
\sum_{i\in \calA^t \cup \calV} \x_i^{t\dagger},    
\end{align*}  
In addition, unlike regular clients in $[m]$, each virtual client is always available, i.e., $\tau_i(t) = t-1$ for $i\in \calV$.    
To distinguish, let $\calR = [m]$ denote the regular clients
and  $M = m + k$.
We define an auxiliary global objective function $\tilde F$ as in~\eqref{eq: adj glob obj}:
\begin{align}
\label{eq: adj glob obj}
\tilde F(\x) 
&\triangleq 
\frac{1}{M} 
\sum_{i \in \calR \cup \calV} 
F_i(\x)
=
\frac{1}{M} 
\sum_{i \in \calR} 
F_i(\x) 
=
\pth{\frac{m}{M}}
\frac{1}{m} \sum_{i \in \calR}
F_i(\x)
=
\pth{\frac{m}{M}}
F(\x). 
\end{align}
Note that the auxiliary local and global objectives are only used to facilitate our analysis of~\prettyref{alg: fedpbc+}; they do not affect the computation at the regular clients.  
Hence, to show that $\x$ converges to a stationary point of $F(\cdot)$ with $k>0$ on the regular client population $\calR$ is equivalent to showing that $\x$ converges to $\tilde{F}(\x)$ with $k=0$ on the augmented client population $\calR\cup \calV$ up to rescaling. 
The latter case can be analyzed by adapting our road map from the conference version~\citep{xiang2024efficient}, but with non-trivial characterizations to account for the generalized~\prettyref{ass: prob lower bound}.

\subsubsection{Information Mixing on the Augmented Learning System.}
\label{sec: information mixing}
We construct a doubly stochastic information mixing matrix $W^{(t)}$ in~\eqref{eq: W matrix elements} that characterizes the information diffusion in~\FedAPM.
\begin{align}
    W_{ij}^{(t)} \triangleq 
    \begin{cases}
        \frac{1}{\abth{\calA^t} + k},~&\text{if } \{i,j \in \calA^t \cup \calV\};\\
        1,~&\text{if } \{ i=j \} ~\text{and } \{i \in \calR \setminus \calA^t\};\\
        0,~&\text{otherwise.}
    \end{cases}
\label{eq: W matrix elements}
\end{align} 
Let 
$W^{(n, \tP)} \triangleq \prod_{t = n \tP}^{(n+1) \tP -1}W^{(t)}$ and
$\rho (n, k) \triangleq \lambda_2 (\Expect[(W^{(n, \tP)})^2])$,
where $\lambda_2(\cdot)$ denotes the second largest eigenvalue,
$n \in \integers^+$,
$\allones = \Indc \Indc^{\top} / M$,
and $\rho_k \triangleq \max_n \rho(n, k)$. 
The information mixing errors, \ie, consensus errors, are quantified through~\prettyref{lmm: spectral norm}.
\begin{lemma}\citep{boyd2005gossip,koloskova2020unified}
\label{lmm: spectral norm}
For any matrix $B \in \reals^{d \times M}$, it holds that
$
\mathbb{E}_{W}[\fnorm{B \pth{W^{(n,\tP)} - \allones}}^2] \le \rho_k \fnorm{B(\identity - \allones)}^2,
$
where the expectation is taken w.r.t. randomness in $W$ matrices.
\end{lemma}

\subsection{Imaginary Update Sequence Construction}
\label{sec: auxiliary sequence}
Directly analyzing the evolution of $\x^t$ and $\x_i^t$ is challenging 
due to the fact that different clients update at different rounds, 
and that different active clients echo their 
local innovation $\bm{G}_i^t$ (line 12 in~\prettyref{alg: fedpbc+}) with different strength $(t-\tau_i)$. 
As such, we construct an imaginary update sequence $\bm{z}_i^t$ for client $i\in [m]$, whose evolution is closely coupled with $\x^t$ and $\x_i^t$
yet is easier to analyze.  
Note that the imaginary update sequence is never actually computed by clients but acts as a necessary tool in building up the analysis.
\begin{definition}
\label{def: auxiliary sequence}
The auxiliary sequence $\{\bz_i^t\}$ of client $i\in \calR \cup \calV$ is defined as 
\begin{align}
\label{eq: auxiliary definition form}
    \bz_i^t ~ \triangleq ~
    \begin{cases}
    \x_i^t,~&\forall~ i \in \calV; \\
    \x_i^t - \eta_l \eta_g s (t-\tau_i(t) - 1)
        \nabla F_i(\x_i^{\tau_{i}(t)+1}),~&\forall~ i \in \calR.
    \end{cases}
\end{align}    
\end{definition}
We know that the virtual clients in $\calV$ are 
(i) always available and 
(ii) with zero-valued local gradient updates since their local models are duplicates of the global model from the immediate previous round.
Therefore, \eqref{eq: auxiliary definition form} can be simplified as \eqref{eq: auxiliary simplified} by using the convention that $\nabla F_i (\x_i^{\tau_{i}(t)+1}) = \bm{0}$ for any virtual client $i \in \calV$ in any round $t \in [T]$:
\begin{align}
\label{eq: auxiliary simplified}
    \bz_i^t ~ \triangleq ~
    \x_i^t - \eta_l \eta_g s (t-\tau_i(t) - 1)
        \nabla F_i(\x_i^{\tau_{i}(t)+1}),~&\forall~ i \in \calR \cup \calV.
\end{align}
    
    \paragraph{When 
    $i\in \calA^{t-1}$, the iterate of $\bm{z}_i$ is a bit more involved:} 
    \begin{align}
    \bz_i^t 
    &\overset{(\ref{eq: auxiliary active line 2}.a)}{=} 
    \x_i^t 
    \overset{(\ref{eq: auxiliary active line 2}.b)}{=} 
    \frac{\sum_{j\in \calA^{t-1}}}{|\calA^{t-1}|}
    \pth{
    \bz_j^{t-1}
    +
    \underbrace{(\x_j^{t-1} - \bz_j^{t-1})}_{(\ref{eq: auxiliary active line 2}.c)}
    - 
    \eta_g (t - 1 - \tau_{j}(t-1) )
    \bm{G}_j^{t-1}} 
    \label{eq: auxiliary active line 2} 
    \\
    \notag
    & = \frac{1}{\abth{\calA^{t-1}}}
    \sum_{j \in \calA^{t-1}}
    \pth{
    \bz_j^{t-1}
     -
    \eta_l \eta_g
    \sum_{r=0}^{s-1}
    \nabla \ell_j(\x_j^{(t-1,r)};\xi_i^{(t,r)})} \\
    \label{eq: z iterate last active}
    &~~~+ 
    \frac{\eta_l \eta_g}{\abth{\calA^{t-1}}}
    \sum_{j \in \calA^{t-1}}
    \pth{t - 2 - \tau_j(t-1)}
    \sum_{r=0}^{s-1}
    \pth{
    \nabla F_j(\x_j^{\tau_j(t-1)+1})
    -
    \nabla \ell_j(\x_j^{(t-1,r)};\xi_i^{(t,r)})
    },
    \end{align}    
    where 
    $(\ref{eq: auxiliary active line 2}.a)$ holds because of~\prettyref{def: auxiliary sequence} and $i \in \calA^{t-1}$,
    $(\ref{eq: auxiliary active line 2}.b)$ because of line 12 in~\prettyref{alg: fedpbc+},
    addition and subtraction,
    and we can get~\eqref{eq: z iterate last active} by replacing $(\ref{eq: auxiliary active line 2}.c)$ with~\eqref{eq: auxiliary definition form}.
    Recall that, for client $j \in \calV$, we have 
    (i) $\x_j^t = \bz_j^t$ and
    (ii) $\bm{G}_{j}^{t} = \bm{0}$.
   
   \paragraph{When 
    $i \in \calR \setminus \calA^{t-1}$, $\bm{z}_i^t$ has a simple iterative relation:}
    \begin{align}
    \label{eq: auxiliary inactive update}
    \bz_i^t ~ = ~ \bz_i^{t-1}  - \eta_l\eta_g s \nabla F_i(\x_i^{\tau_i(t-1)+1}).      
    \end{align}
At a high level, 
the sequence $\bm{z}_i^t$ 
approximately
mimics the ideal descent evolution at a client as if the client performs local optimizations on its local model $\x_i$ per round regardless of its availability. 
Mathematically,
the idea is that, 
if the progress per iteration of the auxiliary sequence $\bz_i^t$ is bounded,
we can show the convergence of $\x_i^t$
when $\x_i^t$ and $\bz_i^t$ are close to each other.

It is worth noting that imaginary sequences are used in peer-to-peer distributed learning literature \citep{spiridonoff2020robust,avdiukhin2021federated,lian2017can,yuan2016convergence,stich2018local,nedic2018network}. 
Yet, existing constructions are not applicable to our problem due to the 
(i) non-convexity of the global objectives,  
(ii) multiple local updates per round, 
(iii) possibly unbounded gradients, and the 
(iv) general form of bounded gradient dissimilarity. 
Departing from the use of staled stochastic gradients for auxiliary updates therein,
we adopt the true gradient $\nabla F_i(\cdot)$ 
to avoid the complications from the involved interplay between randomness in stochastic samples
and randomness in $\tau_i(t)$.
On the technical front, it follows from~\prettyref{def: auxiliary sequence} that
 $\|\x_i^t - \bz_i^t\|_2^2 \le \eta_l^2 \eta_g^2 s^2 (t - \tau_i(t)-1)^2 \|\nabla F_i(\x_i^{\tau_{i}(t)+1})\|_2^2$, 
 whose bound appears to be quite challenging to derive due to the coupling of different realizations of $\tau_i(t)$ and gradients.
 As such, we bound the average of $\|\x_i^t - \bz_i^t\|^2 $ across clients and rounds in~\prettyref{prop: client dis}.

\begin{lemma}[Unavailability statistics]
\label{lmm: geo second moment main text}
Under~\prettyref{ass: prob lower bound}
and $\delta$ defined therein.
It holds for $t\ge 0$ that

\noindent
\begin{minipage}{0.31\textwidth}
\begin{small}
\begin{align}
\label{eq: stale exp}
\expect{t - \tau_i(t)} &\le P_{\delta} + \frac{1}{\delta}; 
\end{align}    
\end{small}
\end{minipage}
\begin{minipage}{0.7\textwidth}
\begin{small}
\begin{align}
\label{eq: stale var}
\expect{\pth{t - \tau_i(t)}^2} &\le 
\frac{\pth{(P_{\delta}-1) \delta + 1}^2 + (P_{\delta} - 1) \delta^2 + 1 }{\delta^2}
.
\end{align}
\end{small}
\end{minipage}
\end{lemma}
\begin{remark}
    \label{rmk: unavail stats}
    \prettyref{lmm: geo second moment main text} yields an upper bound on the first and second moments of a client $i$'s unavailable duration.
    Its proof can be found in~\prettyref{app: unavail stats}, where we leverage the tools from probability theory~\citep{gut2006probability}.
    Here, we remark on some special cases:
    \begin{itemize}[leftmargin=*]
        \item When $\delta = 1$, all clients are available during all training rounds, suggesting a static unavailable duration $t - \tau_i(t)$ of length 1 and thus a static second moment of value 1.
        Recall that we require $\tP$ to be the minimum window size when given a budget $\delta$ in~\prettyref{ass: prob lower bound}, so it implies $\tP = 1$. 
        In this case,
        both of our bounds are loose by only a constant offset 1.
        \item When $P_{\delta} = 1$, the dynamics~\eqref{eq: prob lower} becomes $p_i^t \ge \delta$, which follows a heterogeneous geometric distribution.
        Our bounds~\eqref{eq: stale exp} and~\eqref{eq: stale var} reduce to $1 + 1/\delta$ and $2 / \delta^2$, respectively. 
        Compared with \citep[Lemma 2]{xiang2024efficient}, the first moment is loose by only a constant offset 1, while the second moment matches the result therein.
        When we further relax the condition and consider the special case where clients are available with the same probability $\delta$, the unavailable duration $t - \tau_i(t)$ simply follows a homogeneous geometric distribution. It can be checked that our bound trivially holds. 
    \end{itemize}
\end{remark}

\subsection{Main Convergence Results}
\label{sec: main results}
Let $\bar{\bm{z}}_t \triangleq \frac{1}{M} \sum_{i=1}^{M} \bm{z}_i^t$,
$\tilde F^\star \triangleq \min_{\x} \tilde F(\x)$,
and $\delta_{\max} \triangleq \max_{i\in \calR, t \in [T]} p_i^t$.
Recall that $M = m + k$. 
\begin{lemma}[Descent Lemma]
\label{lmm: descent lemma}
Let $\calF^t$ define the sigma algebra generated by randomness up to round $t$.
Suppose that Assumptions~\ref{ass: 2 smmothness},
\ref{ass: bounded variance client-wise} hold and 
$\eta_l \eta_g \le 1 / (8 s L)$. 
It holds that 
\begin{align*}
\expect{\tF(\bar{\bz}^{t+1}) - \tF(\bar{\bz}^t)~|~\calF^t}
&\le
- \frac{\eta_l \eta_g s}{3} \norm{\nabla \tF(\bar{\bz}^t)}^2 \\
&~~~+\frac{
2 \eta_l \eta_g s L \sigma^2
\pth{\eta_l \eta_g \delta_{\max} + 9 m \eta_l^2 s L}}
{M^2} 
\sum_{i=1}^m
(t - \tau_i(t))^2
\\
&~~~+
\frac{65 \eta_g \eta_l^3 s^3 L^2}{M} 
\sum_{i=1}^m 
(t - \tau_i(t))^2
\norm{\nabla \tF_i(\x_i^{\tau_i(t)+1})}^2 \\
&~~~+ 
\frac{4 \eta_l \eta_g s L^2 }{M}
\sum_{i=1}^m
\underbrace{\norm{\x_i^t - \bz_i^t}^2}_{\textnormal{Approximation Error}} 
+
\frac{\eta_l \eta_g s L^2}{2 M}
\sum_{i=1}^m
\underbrace{\norm{\bz_i^t - \bar{\bz}^t}^2}_{\textnormal{Consensus Error}}.
\end{align*}
\end{lemma}

The proof of~\prettyref{lmm: descent lemma} follows from the standard analysis for non-convex smooth objectives but
with non-trivial adaptation to account for
{\em adaptive innovation echoing}
and {\em implicit gossiping}.
In particular,
it highlights two terms unique in our derivation:
the approximation error from the auxiliary sequence
and the consensus error from the implicit gossiping procedure.
\begin{proposition}[Approximation error]
\label{prop: client dis}
Suppose that Assumptions \ref{ass: 2 smmothness} and \ref{ass: bounded similarity} holds, we have
\begin{small}
\begin{align}
\notag
&\frac{1}{M T}\sum_{t=0}^{T-1}
\sum_{i=1}^M
\expect{\norm{\x_i^t - \bz_i^t}^2} 
\overset{(\ref{eq: approx error}.a)}{=}
\frac{1}{M T}\sum_{t=0}^{T-1}
\sum_{i=1}^m
\expect{\norm{\x_i^t - \bz_i^t}^2} 
\\
\notag
&\le 
3 \eta_l^2 \eta_g^2 s^2 
\frac{\qth{(P_{\delta} - 1) \delta + 1}^2 + \qth{(P_{\delta} - 1) \delta^2 + 1}}{\delta^2}
\frac{M \pth{\beta^2 + 1}}{m}
\frac{1}{T}\sum_{t=0}^{T-1} 
\expect{\norm{\nabla \tF(\bar{\bz}^{t})}^2} \\
&~~~
+ 3 \eta_l^2 \eta_g^2 s^2 
\frac{\qth{(P_{\delta}-1) \delta + 1}^2 + \qth{(P_{\delta} - 1) \delta^2 + 1}}{\delta^2}
\pth{
\frac{m \zeta^2}{M}
+
\frac{L^2}{M}
\sum_{i=1}^m \frac{1}{T}\sum_{t=0}^{T-1} 
\expect{\norm{\bz_i^{t} - \bar{\bz}^t}^2} }
,
\label{eq: approx error}
\end{align}    
\end{small}
where $(\ref{eq: approx error}.a)$ holds because of \eqref{eq: auxiliary definition form}.
\end{proposition}
The proof of~\prettyref{prop: client dis} starts from~\prettyref{def: auxiliary sequence}.
Although in general it is difficult to bound the error,
Assumptions~\ref{ass: 2 smmothness} and~\ref{ass: bounded similarity} 
allow us to break down the problem into 
bounding the averaged gradient norm of $\bar{\bz}^t$ and the consensus error over all randomness instead.
Next,
we analyze the consensus error.
Note that
although implicit gossiping
takes place in~\prettyref{alg: fedpbc+} for $\x_i^t$,
its analysis is technically challenging
as discussed before.
So, 
we adopt the auxiliary $\bz_i^t$ as an intermediary and apply Young's inequality to bound the actual consensus error.
Formally,
the auxiliary models can be expressed in a compact matrix form as 
$\bm{Z}^{(t)} \triangleq [\bz_1^t, \ldots, \bz_m^t]$.
Their local parameter innovation matrix ${ \tilde{{\bm{G}}^t}}$~\eqref{eq: auxiliary update detail main text}
is formulated by combing~\eqref{eq: auxiliary active line 2} and~\eqref{eq: auxiliary inactive update}.
\begin{align}
    \tilde{\bm{G}}^t_{i}
    &=
    \indc{i \in \calA^t}
    \pth{t - \tau_i(t)}
    \sum_{r=0}^{s-1}
    \pth{
    \nabla \ell_i(\x_i^{(t,r)})
    -
    \nabla F_i(\x_i^{t})}
    +
    s \nabla F_i(\x_i^{t}).
    \label{eq: auxiliary update detail main text}
\end{align}
Unrolling the recursion,
it holds for the consensus error $\sum_{i=1}^M \|\bz_i^t - \bar{\bz}^t\|_2^2 / M$
that
\begin{align}
\frac{1}{M}
{\lnorm{\pth{\bm{Z}^{(t-1)} - \eta_l \eta_g \tilde{\bm{G}}^{(t-1)}} W^{(t-1)} \pth{\identity - \allones}}{\rm F}^2} 
\label{eq: consensus derivation}
\overset{(\ref{eq: consensus derivation}.a)}{=} 
\frac{\eta_l^2 \eta_g^2}{M} 
\lnorm{\sum_{q=0}^{t-1} \tilde{\bm{G}}^{(q)} \pth{\prod_{l=q}^{t-1} W^{(q)} - \allones}}{\rm F}^2,
\end{align}
where equality $(\ref{eq: consensus derivation}.a)$ holds because all clients are initiated at the same weight. 
Recall that $\rho_k$ is the spectral norm of the information mixing matrix square over a rolling window of $P_{\delta}$ rounds in~\prettyref{lmm: spectral norm}.
To ensure an exponential decay of the consensus error, it is crucial to have $\rho_k < 1$, which is confirmed by~\prettyref{lmm: rho upper bound main text}.
\begin{lemma}
\label{lmm: rho upper bound main text}
Suppose that~\prettyref{ass: prob lower bound} and $k > 0$. 
\paragraph{When $P_{\delta}=1$, it holds that}
\begin{align}
\label{eq: single spectral}
\rho_k \le 1 - 
\frac{\qth{m^2 \delta + (k^2 + 2 m k)}^2 \delta^2}{8 (m + k)^2}. 
\end{align}
\paragraph{When $P_{\delta}>1$, it holds that}
\begin{align}
\label{eq: block spectral}
\rho_k & \le 
1 - \frac{\qth{m^2 \delta + (k^2 + 2 m k)}^2 \delta^2}{8 (m + k)^{4 P_{\delta} + 2}}.    
\end{align}
\end{lemma}
\begin{proof}{Proof Sketch.}
The full proof is deferred to~\prettyref{app: consensus}.
Coarsely, we study the conductance of a hypothetical Markov chain, 
whose transition matrix is a matrix square $(W^{(n, P_{\delta})})^2$, 
as we are interested in a sliding window of $P_{\delta}$ rounds.
The unique challenge of our analysis arises from the involved information mixing between volatile regular clients in $\calR$ and always-on virtual clients in $\calV$.
Instead, we represent the federated learning systems as a graph, where the nodes are clients, and the edges are defined by the mixing matrix $W$ weights, capturing the interactions between clients.
As such, we can bound the conductance of the hypothetical Markov chain by studying the product of the edge weights
and use Cheeger's inequality to close the gap between the spectral norm and the conductance.
\end{proof}
\begin{remark}
\label{rmk: block spectral norm}
\prettyref{lmm: rho upper bound main text} is divided into two parts: $P_{\delta} = 1$ and $P_{\delta} > 1$, where \eqref{eq: single spectral} is tighter than \eqref{eq: block spectral} when $P_{\delta}=1$.
A similar dependence on the number of clients in the spectral norm bound has been noted in the prior fully decentralized learning literature, \eg, in \citep{nedic2014distributed}, the clients therein are assumed to form a $P_{\delta}$-strongly-connected graph.
Mapping to our setup, it means that all clients are available at least once over $P_{\delta}$ consecutive rounds.
In contrast, we only need the average available probability of each client to be non-zero across every $P_{\delta}$ rounds, which is more general and makes direct applications of their results inapplicable.
It remains an open question whether the worst-case bound can be improved, and we would like to leave this as future work.
\end{remark}
We now proceed to present the convergence rates.
In the sequel, we assume
it holds for
$\eta_g$ and 
$\eta_l$ that
\begin{small}
\begin{align}
\label{eq: lr condition main text}
\eta_l \eta_g 
\le
\frac{\delta (1 - \rho_k) \sqrt{m}}{96 s L \sqrt{\tP M
\pth{(P_{\delta} - 1)^2 \delta^2 + 1}
\pth{\beta^2 + 1}}}
~ ; ~
\eta_l 
\le
\frac{\delta}{216 s L
\sqrt{
\pth{(P_{\delta} - 1)^2 \delta^2 + 1}
(\beta^2 + 1)}}.
\end{align}
\end{small}
The proof of the consensus error borrows insights from the analysis of the gossip algorithm~\citep{nedic2017achieving,wang2022matcha} but with substantial adaptation to accommodate the novel auxiliary formulation and multi-step local updates.
Under the learning rate conidtions in~\eqref{eq: lr condition main text} and Assumptions~\ref{ass: prob lower bound}, \ref{ass: 2 smmothness}, \ref{ass: bounded variance client-wise} and \ref{ass: bounded similarity},
we can show that
\begin{small}
\begin{align}
    \frac{1}{T} 
    \sum_{t=0}^{T-1}
    \frac{1}{M}
    \sum_{i=1}^M 
    \expect{\norm{\x_i^t - \bz_i^t}^2} 
    &\asymp
    \frac{1}{T} 
    \sum_{t=0}^{T-1}
    \frac{1}{M} 
    \sum_{i=1}^M 
    \expect{\norm{\bz_i^t - \bar{\bz}^t}^2}
    \asymp
    \frac{1}{T}
    \sum_{t=0}^{T-1}
    \expect{\norm{\nabla F(\bar{\bz}^t)}^2}.
    \label{eq: approx and consensus for z}
\end{align}
\end{small}
It remains to bound the full convergence error of $\bz_i^t$,
which is presented in~\prettyref{thm: z bar rate}.
\begin{theorem}[Convergence error of $\bz_i^t$]
\label{thm: z bar rate}
Suppose that Assumptions~\ref{ass: prob lower bound}, \ref{ass: 2 smmothness}, \ref{ass: bounded variance client-wise} and \ref{ass: bounded similarity} hold.
Choose learning rates $\eta_l$ and $\eta_g$ such that the conditions in~\eqref{eq: lr condition main text} are met for $T\ge 1$ and $k > 0$. 
It holds that
\begin{align}
    \nonumber
    \frac{1}{T}\sum_{t=0}^{T-1}
    &\expect{\norm{\nabla F(\bar{\bz}^t)}^2}
    \lesssim
    \pth{\frac{m + k}{m}}
    \frac{\pth{F(\bar{\bz}^0) - F^\star}}{\eta_l \eta_g s T}
    +
    \pth{\frac{m + k}{m}}
    \frac{\delta_{\max} \eta_l \eta_g L \sigma^2}{m + k}
    \qth{
    \pth{P_{\delta} - 1}^2
    +
    \frac{1}{\delta^2}
    }
    \\
    &\qquad +
    \pth{\frac{m + k}{m}}
    \eta_l^2 \eta_g^2 s^2 L^2 \tP
    \pth{\sigma^2 +  \zeta^2}
    \qth{
    \pth{P_{\delta} - 1}^2
    +
    \frac{1}{\delta^2}
    }
    \qth{
    1
    +
    \frac{1}{(1 - \rp)^2}
    }. 
    \label{eq: z bar rate}
\end{align}
\end{theorem}
By addition, subtraction, and Young's inequality,
\eqref{eq: x bar consensus} and~\eqref{eq: x bar relation} hold under~\prettyref{ass: 2 smmothness}.
\begin{small}
\begin{align}
\frac{1}{T} \sum_{t=0}^{T-1}
\frac{1}{M} \sum_{i=1}^M 
\expect{\norm{\x_i^t - \bar{\x}^t}^2}
&\asymp
\frac{1}{T} \sum_{t=0}^{T-1}
\frac{1}{M} \sum_{i=1}^M 
\expect{\norm{\x_i^t - \bz_i^t}^2}
+
\frac{1}{T} \sum_{t=0}^{T-1}
\frac{1}{M} \sum_{i=1}^M 
\expect{\norm{\bz_i^t - \bar{\bz}^t}^2}
;
\label{eq: x bar consensus} \\
\frac{1}{T}
\sum_{t=0}^{T-1}
\expect{\norm{\nabla F(\bar{\x}^t)}^2}
&\asymp
\frac{1}{T}\sum_{t=0}^{T-1} 
\frac{1}{M} \sum_{i=1}^M
\expect{\norm{\x_i^t - \bz_i^t}^2
}
+
\frac{1}{T}\sum_{t=0}^{T-1} \expect{\norm{\nabla F(\bar{\bz}^t)}^2}.
\label{eq: x bar relation}
\end{align}
\end{small}
Moreover, 
from~\eqref{eq: approx and consensus for z},~\eqref{eq: x bar consensus} and~\eqref{eq: x bar relation},
it can be seen that~\eqref{eq: x consensus} holds.
\begin{small}   
\begin{align}
    \frac{1}{T} 
    \sum_{t=0}^{T-1}
    \frac{1}{M} 
    \sum_{i=1}^M 
    \expect{\norm{\x_i^t - \bar{\x}^t}^2} 
    &\asymp
    \frac{1}{T}
    \sum_{t=0}^{T-1}
    \expect{\norm{\nabla F(\bar{\bz}^t)}^2}
    \asymp
    \frac{1}{T}\sum_{t=0}^{T-1}\expect{\norm{\nabla F(\bar{\x}^t)}^2}
    .
    \label{eq: x consensus}
\end{align}
\end{small}
Combining~\eqref{eq: approx and consensus for z},
\eqref{eq: z bar rate},
\eqref{eq: x bar consensus}
and \eqref{eq: x bar relation},
we are ready for~\prettyref{cor: x bar rate}.
\begin{corollary}[Convergence rate of $\x_i^t$]
\label{cor: x bar rate}
Suppose that Assumptions \ref{ass: prob lower bound}, \ref{ass: 2 smmothness}, \ref{ass: bounded variance client-wise} and \ref{ass: bounded similarity} hold.
Choose learning rates as
$\eta_l = \frac{1}{\sqrt{T} s L}$,
$\eta_g = \sqrt{s \delta m}$. 
Let $T\ge 1$ be sufficiently large so that the conditions of $\eta_l$ and $\eta_g$ in~\eqref{eq: lr condition main text} are satisfied.  
For $k>0$, it holds that
\begin{align}
    \nonumber
    \frac{1}{T}\sum_{t=0}^{T-1}
    &\expect{\norm{\nabla F(\bar{\x}^t)}^2}
    \lesssim
    \pth{\frac{m + k}{m}}
    \frac{L \pth{F(\bar{\x}^0) - F^\star}}{\sqrt{s \delta m T}}
    +
    \frac{\delta_{\max}}{\delta^{\frac{3}{2}} \sqrt{s m T}}
    \qth{\pth{P_{\delta} - 1}^2 \delta^2 + 1}
    \sigma^2
    \\
    &\qquad \qquad \qquad
    +
    \pth{\frac{m + k}{m}}
    \frac{s m \tP}{T}
    \pth{\sigma^2 +  \zeta^2}
    \qth{\pth{P_{\delta} - 1}^2 \delta^2 + 1}
    \qth{1 + \frac{1}{(1 - \rp)^2}}
    . 
    \label{eq: x bar rate}
\end{align}
\end{corollary}
\begin{remark}[Linear speedup]
\label{rmk: convergence results}
\prettyref{cor: x bar rate} establishes the full convergence rate for~\FedAPM~algorithm.
It can be seen that the second term dominates when $T$ is sufficiently large,
which relates to stochastic gradient noise $\sigma^2$.
The non-stationary client unavailability results in
the third term, which relates to gradient divergence $\zeta^2$ and also to $\sigma^2$. 
The proof of~\prettyref{cor: x bar rate} follows from~\eqref{eq: x bar relation} by plugging in~\prettyref{prop: client dis} and~\prettyref{thm: z bar rate}.

In the special case where $k = \Theta (m)$ and $\calA_t=[m]$,
we simply have $\delta_{\max} = \delta = 1
$ and $P_{\delta} = 1$.
Our convergence bound 
reduces to $O(1/\sqrt{s m
T})$.
In other words, we achieve the desired linear speedup property with respect to the number of local steps $s$ and the number of clients $m$,
matching rates in the established literature \citep{yu2019linear,yu2019parallel,yang2021achieving,wang2022,wang2023lightweight}. 
The linear speedup property enables a large cross-device federated learning system to take advantage of the massive scale of parallelism.
Notice that the consensus error~\eqref{eq: x consensus} and the convergence rate~\eqref{eq: x bar rate} have the same asymptotic order with respect to the parameters mentioned above.
Hence, the consensus error also enjoys the desired linear speedup property when $T$ is sufficiently large in this special case.
\end{remark}

\subsection{Impacts of $k$}
\label{sec: impact of k}
In this section, 
we elaborate on the scaling of spectral norm $\rho_k$ and convergence upper bound in \eqref{eq: x bar rate} w.r.t.\,$k$.
Next, we explore the necessity of interpolation ($k>0$) under~\prettyref{ass: prob lower bound}.

\subsubsection{On the Scaling of Convergence Results}
\label{sec: scaling of k}
\paragraph{On spectral norm $\rho_k$.}
It is easy to show that the upper bound in~\eqref{eq: single spectral} decreases monotonically in $k$ by taking partial derivatives; yet, the monotonicity is not that straightforward in~\eqref{eq: block spectral}.
On the other hand, the upper bounds characterize only the worst-case scenario, so they cannot directly inform the monotonicity analysis of $\rho_k$ w.r.t.\,$k$.
Intuitively, a greater $k$ implies a more connected client population because virtual clients are always available.
When $k \diverge$, we have $\rho_k \approx 0$ since the $W$ matrix will be dominated by the virtual clients, approaching a scaled all-one matrix. 
We hypothesize that the spectral norm $\rho_k$ would decrease w.r.t.\,$k$,
which is numerically demonstrated in Example~\ref{example: rho monoto} by explicit realizations of~\prettyref{ass: prob lower bound}.
\begin{example}
    \label{example: rho monoto}
    In~\prettyref{fig: rho mono}, a total of $m=10$ clients are available under the dynamics shown in~\prettyref{fig: prob mono P=1} with $P_{\delta} = 1$ and~\prettyref{fig: prob mono P>1} with $P_{\delta}>1$.
    We expect to see a smaller spectral norm when clients are more frequently available.
    The results match our hypothesis that the spectral norm $\rho_{k}$ would decrease w.r.t. $k$. 
    It is expected that a larger $P_{\delta}$ leads to worse information fusion, \ie, a smaller spectral norm $\rho_k$.
    Details can be found in the captions. 
    \begin{figure}[!htb]
        \centering
        \begin{subfigure}[b]{.45\textwidth}
            \centering
            \includegraphics[width=\linewidth]{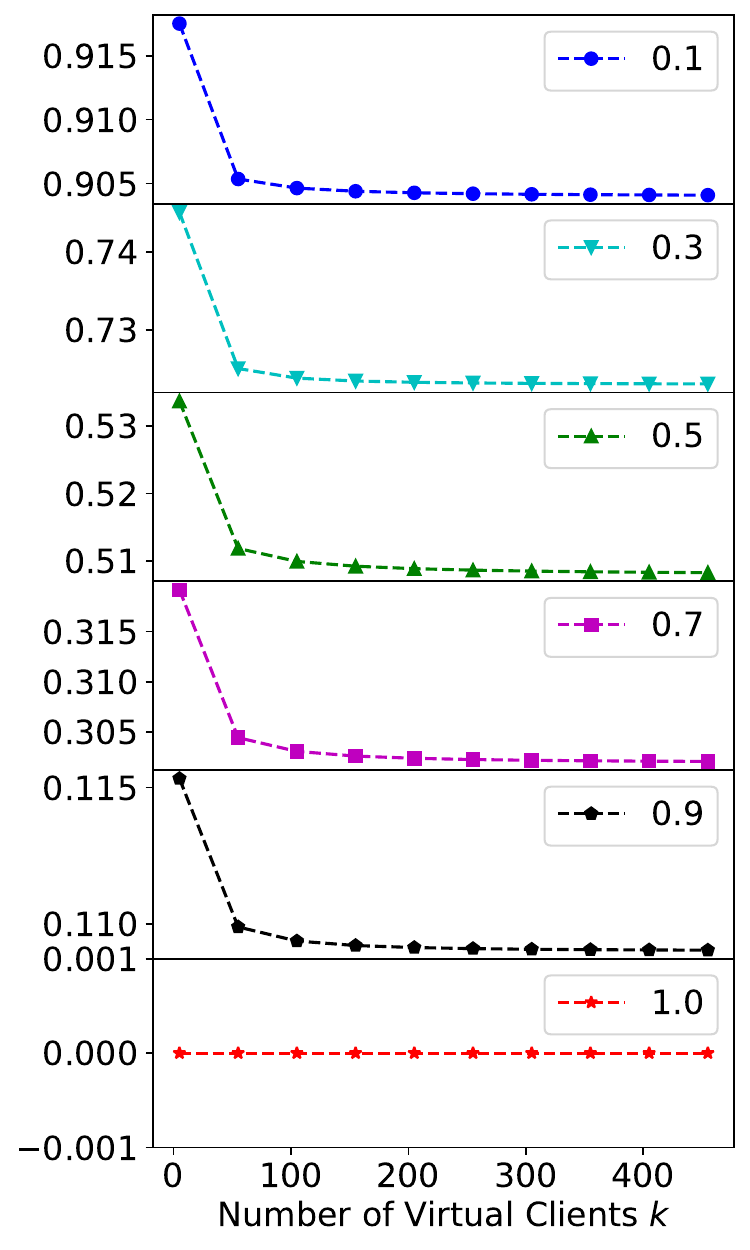}
            \includegraphics[width=.85\linewidth]{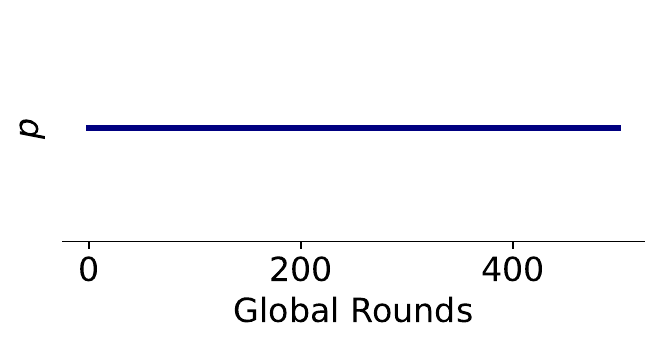}
            \caption{\footnotesize
            Regular clients (\ie, clients in $\calR$) are available with probability $p$ in any global round. 
            That is, the period is $P = 1$. 
            The values of $p$ can be found in the legends.
            $y$-axis is the calculated spectral norm of $(W^{(0,1)})^2$ over $1000$ repetitions.
            We can see that the spectral norm decreases in $k$ except when $p=1$.
            In that case, clients are always available, leading to a zero-valued spectral norm.
            }
            \label{fig: prob mono P=1}    
        \end{subfigure}
        \hfill
        \begin{subfigure}[b]{.45\textwidth}
            \centering
            \includegraphics[width=.9\linewidth]{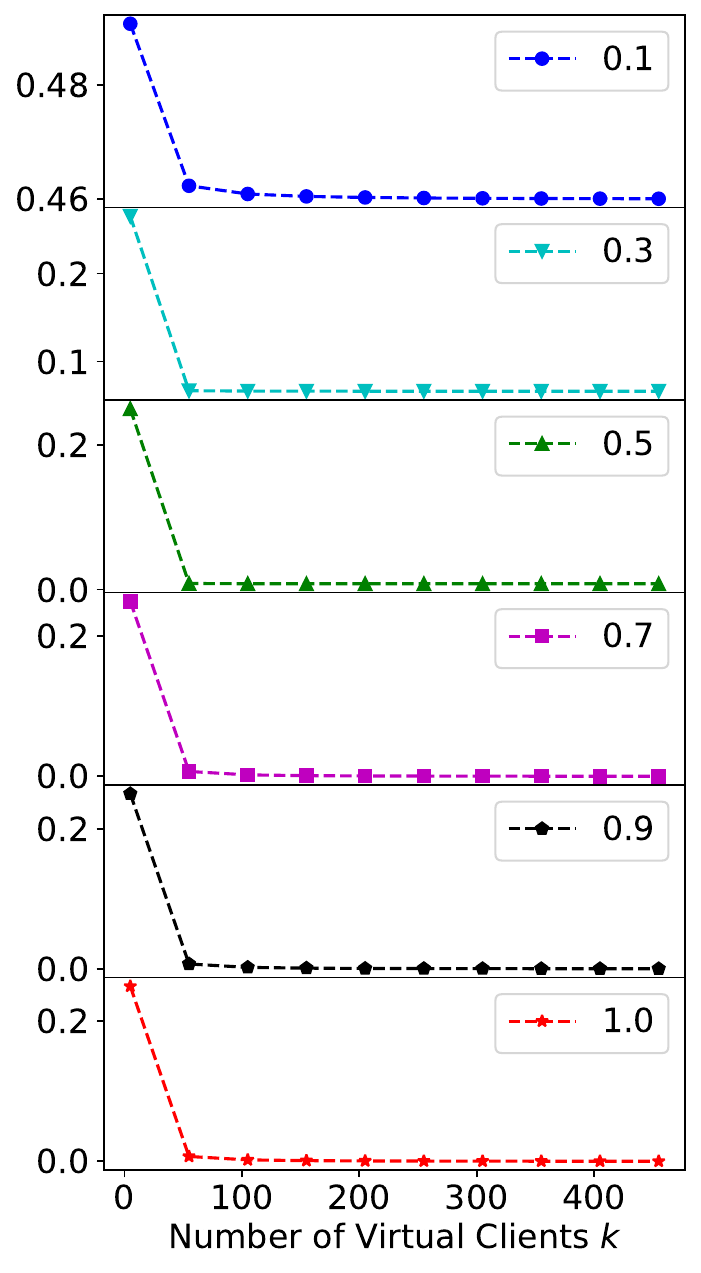}
              \includegraphics[width=.8\linewidth]{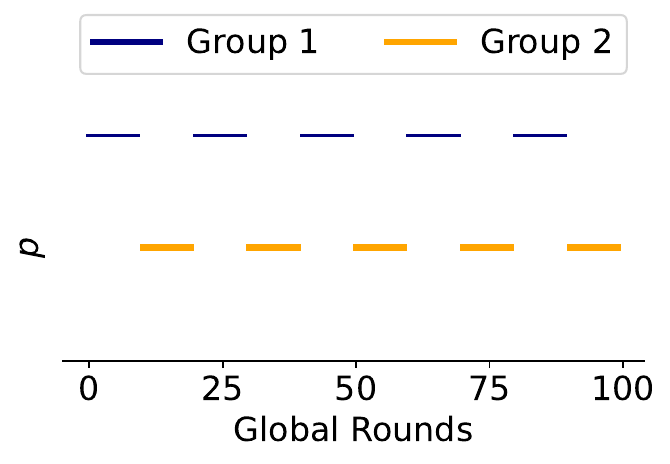}
            \caption{\footnotesize
            Regular clients (\ie, clients in $\calR$) are divided into two groups.
            Each group $i$ is available with probability $p_i$ for 10 rounds. 
            That is, the period is $P = 20$. 
            $y$-axis is the calculated spectral norm of $(W^{(0,20)})^2$ over $1000$ repetitions.
            We can see that the spectral norm decreases in $k$. 
            }
            \label{fig: prob mono P>1}
        \end{subfigure}
        \caption{\footnotesize
        Calculated spectral norm results with varying numbers of $k$.
        We consider $m=10$ clients with availability dynamics described in the plots.
        Note that the $y$-axis of each subplot is of a different range.
        }
        \label{fig: rho mono}
    \end{figure}
\end{example}

\paragraph{On convergence upper bound in \eqref{eq: x bar rate}.}
We have shown in~\prettyref{lmm: rho upper bound main text} that $\rho_k < 1$, which holds independently of $k$ as long as $k>0$.
Intuitively, 
as we have discussed in~\prettyref{sec: moving global average}, there exists a sweet spot $k^\ast$ for $k$ that balances training stability and convergence speed.
Yet, analytically obtaining the exact $k^\ast$ is fundamentally challenging, if not impossible at all.
Specifically, the value of $k$ affects the convergence upper bound in~\eqref{eq: x bar rate} by 
(i) explicitly showing up in the numerator and by (ii) implicitly influencing the spectral norm $\rho_k$.
Recall that we hypothesize in~\prettyref{example: rho monoto} that the spectral norm $\rho_{k}$ monotonically decreases w.r.t.\,$k$.
However, which term will ultimately dominate the monotonicity of~\eqref{eq: x bar rate} as $k$ increases remains unclear.
In~\prettyref{sec: numerical}, we empirically find that $k = \Theta(m)$ strikes the best balance between convergence speed and test accuracy.

\subsubsection{Necessity of Interpolation ($k > 0$) when $P_{\delta}>1$}
\label{sec: necessity of interpolation}
Observe that we require the coefficient $k > 0$ in~\prettyref{lmm: rho upper bound main text}, which appears to be an artifact in our proofs to improve training stability.
However, we show next in~\prettyref{prop: counter single} that for our algorithm to hold under~\prettyref{ass: prob lower bound}, it is necessary to have $k>0$. 
To see this, we construct a counterexample via a similar quadratic function as in Example~\ref{example: obj shift}.
Let client $i$'s local objective $F_i(x) \triangleq \norm{x - u_i}^2 / 2$, where $x,~u_i \in \reals$ and $i \in [m]$.
The global objective is 
\begin{align}
\label{eq: counter global obj}
F(x) = \frac{1}{2 m} \sum_{i=1}^m \norm{x - u_i}^2,
\end{align}
with unique global minimizer $x^\ast = (\sum_{i=1}^m u_i) / m$.
\begin{proposition}
    \label{prop: counter single}
    For a global objective as per~\eqref{eq: counter global obj}, 
    let $\calM_1$ the first half client population,
    $\calM_2$ the remaining client population,
    and $\calM \triangleq [m] = \calM_1 \cup \calM_2$.
    When all clients in $\calM_1$ are available in even rounds only, 
    while those in $\calM_2$ are available in odd rounds only.
    We have
    \begin{compactitem}
    \item[(i)]
    \prettyref{ass: prob lower bound} holds under $P_{\delta} = 2$ and $\delta = 0.5$;
    \item[(ii)]
    It holds for the output $x^{t}_{\text{out}}$ of~\prettyref{alg: fedpbc+}
    without interpolation $(k=0)$ that
    \begin{align}
    \label{eq: counter res}
        \lim_{t \diverge} \expect{\norm{\nabla F(x^t_{\text{out}}) - \nabla F(x^{\ast})}^2}
        = 
        \frac{1}{m^2}
        \norm{\sum_{i \in \calM_1} u_i - \sum_{j \in \calM_2} u_j}^2. 
    \end{align} 
    \end{compactitem}
\end{proposition}
\begin{proof}{Proof of~\prettyref{prop: counter single}.}
At a high level, our proof suggests an interesting client participation dynamics that prevents two groups of clients from properly mixing global updates.

\noindent{\em Construction of $p_i^t$'s.}
We assume that clients in $\calM_1$ are available in even rounds only with $p_i^{2 r } = 1$,
while the rest clients in $\calM_2$ are exclusively available in odd rounds with $p_i^{2 r +1} = 1$,
where $r \in \integers^+$.
It is easy to extend the proof to the case where $p_i^t < 1$ for the same group of clients.

\noindent{\em Global and local gradients.}
The global gradient $\nabla F(x)$ and local gradient $\nabla F_i(x)$ are %
\begin{align*}
\nabla F(x) = x - \frac{1}{m}\sum_{i=1}^m u_i,~
\nabla F_i(x) = x - u_i.
\end{align*}

\noindent{\em Combining them together.}
Due to the postponed multi-cast procedure and the lack of interpolation, 
the parameter server cannot carry the aggregated global updates from one round to another.
Hence, the available clients depend on their cohorts' updates in that round to prepare their local models for the next availability. 
Since $\calM_1 \cap \calM_2 = \emptyset$ and non-overlap availability between two groups, 
the clients in $\calM_1$ {\bf cannot} exchange models with the clients in $\calM_2$.
Therefore, the global objective alternates between an average objective of clients in $\calM_1$ and of clients in $\calM_2$, \ie,
\begin{align*}
\begin{cases}
    F_{\text{even}}
    &=
    \frac{1}{m}
    \sum_{i \in \calM_1}
    \norm{x - u_i}^2; \\
    F_{\text{odd}}
    &=
    \frac{1}{m}
    \sum_{i \in \calM_2}
    \norm{x - u_i}^2.
\end{cases}
\end{align*}
The expected output $x_{\text{out}}^t$ follows that
\begin{align*}
\expect{x_{\text{out}}^{t}} = 
\begin{cases}
2(\sum_{i \in \calM_1} u_i)/ m & ~\text{if}~t = 2r; \\
2(\sum_{i \in \calM_2} u_i)/ m & ~\text{if}~t = 2r + 1.
\end{cases}
\end{align*}
Consequently,
\begin{align*}
    \expect{\norm{\nabla F(x_{\text{out}}^{t}) - \nabla F(x^{\ast})}^2}
    &=
    \frac{1}{m^2}
    \norm{\sum_{i \in \calM_1} u_i - \sum_{j \in \calM_2} u_j}^2.
\end{align*}
\end{proof}
\begin{figure}[!t]
    \centering
    \begin{subfigure}[b]{.48\textwidth}
        \includegraphics[width=\linewidth, trim=.2cm .2cm .2cm .2cm, clip]{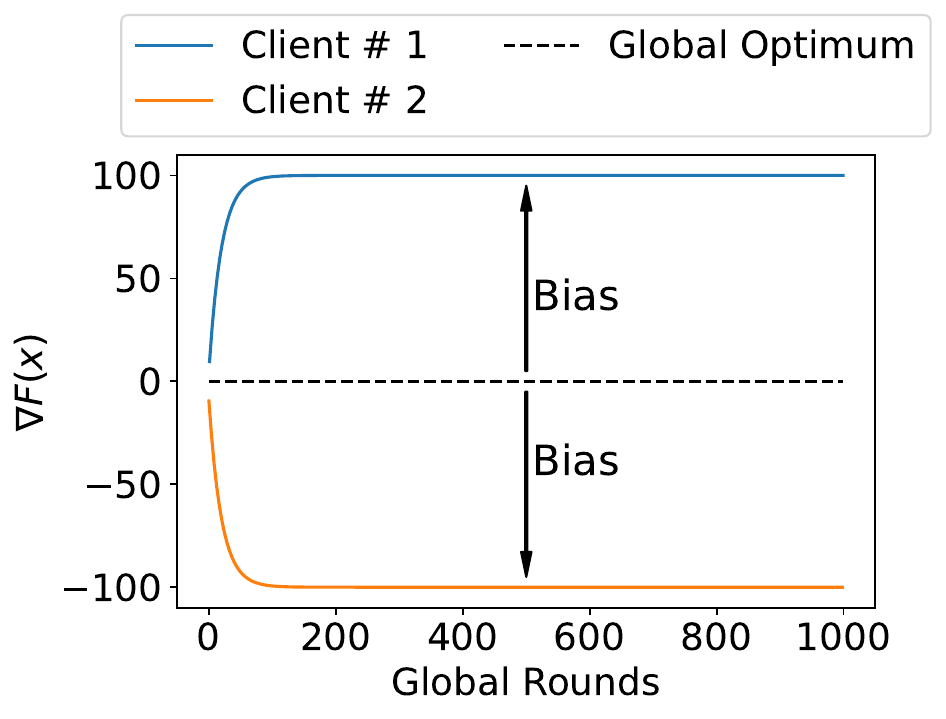}
        \caption{\footnotesize
        A visualization of the scalar gradient of~\prettyref{alg: fedpbc+} without interpolation ($k=0$).
        It can be seen that the biases are significant.
        }
        \label{fig: counter gradient sub}
    \end{subfigure}
    \hfill
    \begin{subfigure}[b]{.48\textwidth}
        \includegraphics[width=.9\linewidth, trim=.2cm .2cm .2cm .2cm, clip]{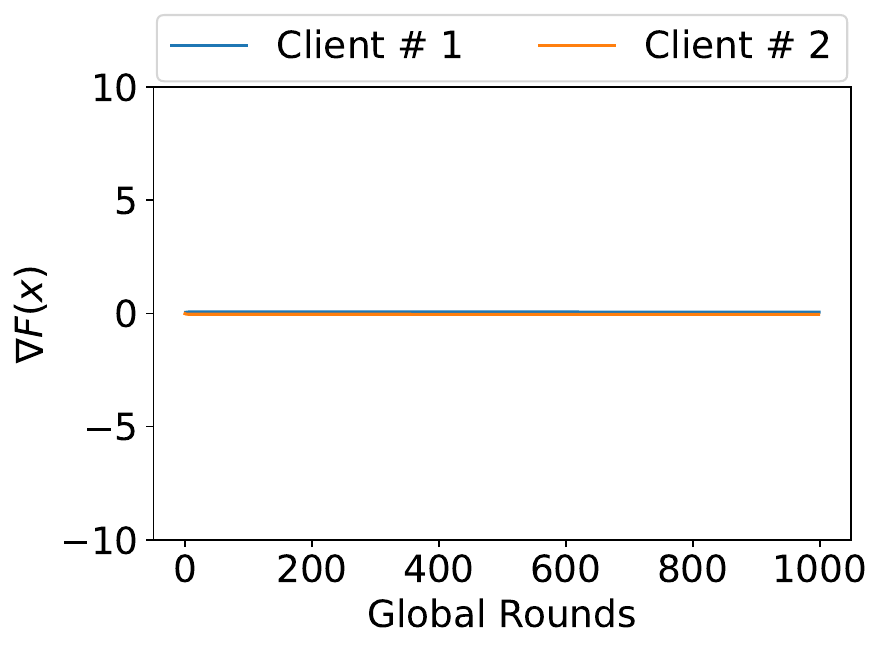}
        \caption{\footnotesize
        A visualization of the scalar gradient of~\prettyref{alg: fedpbc+} {\bf with} interpolation ($k=1$).
        It can be seen that the gradient outputs converge to $0$ under interpolation, recovering the global optimum.
        }
        \label{fig: counter unbiased gradient sub}
    \end{subfigure}
    
    \begin{subfigure}[b]{.48\textwidth}
        \includegraphics[width=\linewidth, trim=.2cm .2cm .2cm .2cm, clip]{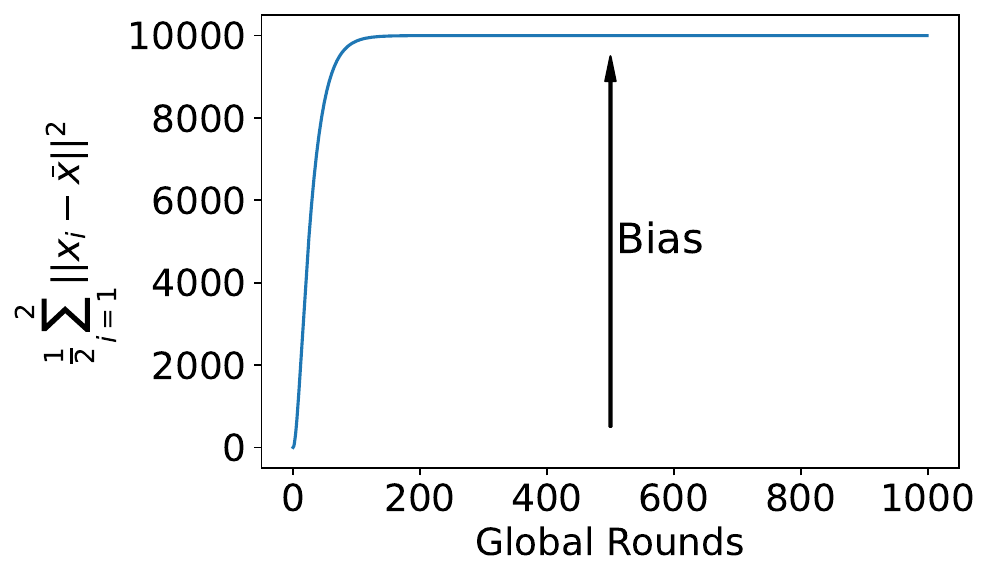}
        \caption{\footnotesize
        A visualization of the consensus of~\prettyref{alg: fedpbc+} without interpolation ($k=0$).
        It can be seen that two clients fail to reach a consensus.}
        \label{fig: counter consensus sub}
    \end{subfigure}
    \hfill
    \begin{subfigure}[b]{.48\textwidth}
        \includegraphics[width=\linewidth, trim=.2cm .2cm .2cm -.5cm, clip]{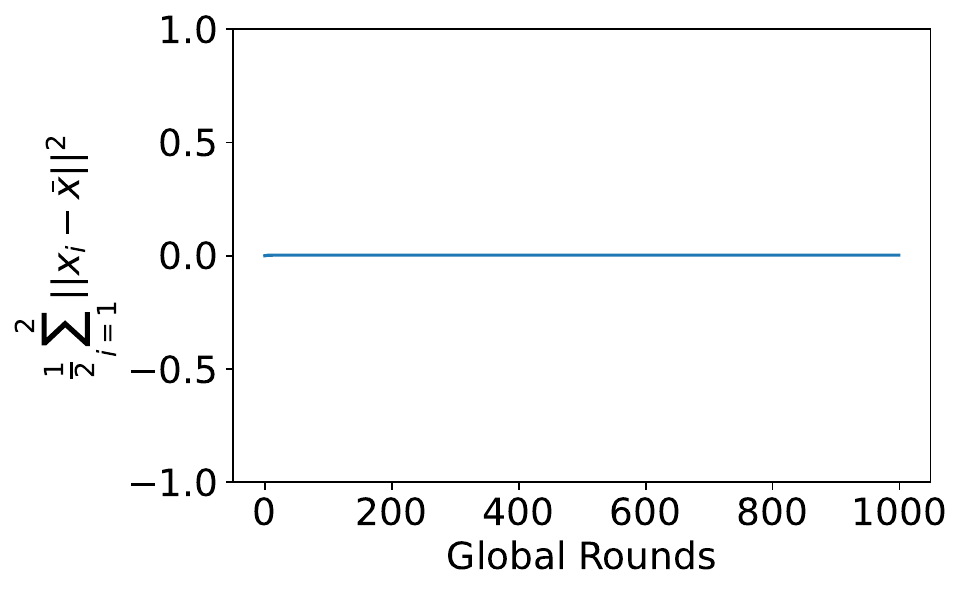}
        \caption{\footnotesize
        A visualization of the consensus of~\prettyref{alg: fedpbc+} {\bf with} interpolation ($k=1$).
        It can be seen that two clients can reach a consensus under interpolation.
        }
        \label{fig: counter unbiased consensus sub}
    \end{subfigure}
    \caption{\footnotesize
    Visualizations 
    of~\prettyref{alg: fedpbc+} with two clients and a global objective~\eqref{eq: counter global obj}, whose $u_1 = 100$, $u_2 = -100$.
    The first client is available exclusively in even rounds,
    while the second client is available in odd rounds only.
    $x$-axis is the global round index, the details of $y$-axis can be found in the subplots.
    Figs.~\ref{fig: counter gradient sub} and \ref{fig: counter consensus sub} plot the results of~\prettyref{alg: fedpbc+} without interpolation.
    It can be seen that the outputs alternate between two clients' local optimums and can deviate far away from the true global minimizer $u_1 + u_2 = 0$.
    Moreover, the clients fail to reach a consensus.
    In sharp contrast, 
    Figs.~\ref{fig: counter unbiased gradient sub} and \ref{fig: counter unbiased consensus sub} indicate that~\prettyref{alg: fedpbc+} recovers the global minimizer with interpolation ($k = 1$) and allows clients to reach a consensus.
    }
    \label{fig: interpolation counter}
\end{figure}

As is empirically verified by the scalar example in~\prettyref{fig: interpolation counter}, $k>0$ may correct the bias.   
Intuitively, interpolation mixes the global models over rounds and enables clients to share information through the parameter server with cohorts in the current round and those from previous rounds.
Mathematically, we guarantee clients to exchange information with each other in expectation over every sliding window of $P_{\delta}$ rounds, provided $k > 0$.
See the proof of~\prettyref{lmm: rho upper bound main text} for details.

\subsubsection{Special case when $P_{\delta}=1$}
Our conference version~\citep{xiang2024efficient} studies a special case where $P_{\delta} = 1$ and interpolation becomes optional.
Our discussions in~\prettyref{sec: necessity of interpolation} are consistent with our results therein.
Informally, this is because clients in~\citep{xiang2024efficient} are never isolated into distinct groups when $P_{\delta} = 1$.
On the technical front, all elements in the information mixing matrix $W^{(t)}$ are strictly positive in expectation in any round $t$ under the unavailability dynamics therein.
Hence, clients can evenly diffuse information with each other in expectation. 
Nevertheless, $k>0$ still helps to reduce fluctuations of the trajectory of $\x^t$.

\section{Numerical Experiments}
\label{sec: numerical}
In this section, we evaluate~\FedAPM~on real-world data sets to corroborate our analysis and compare the performance of~\FedAPM~with the state-of-the-art algorithms.
The missing specifications and additional numerical results can be found in~\prettyref{app: numerical}.
Specifically,
we consider a federated learning system of one parameter server and $m = 100$ clients. 
We focus on image classification tasks, and consider multiple real-world data sets \citep{netzer2011readingdigits,krizhevsky2009learning,darlow2018cinic}.
Each data set contains 10 image classes, but the categories differ.

\subsection{Non-stationary and Heterogeneous Unavailability with $P_{\delta} > 1$}
\label{sec: numerical P>1}
\begin{figure}[!t]
\includegraphics[width=.9\linewidth, trim=0 7.7cm 0 0, clip]{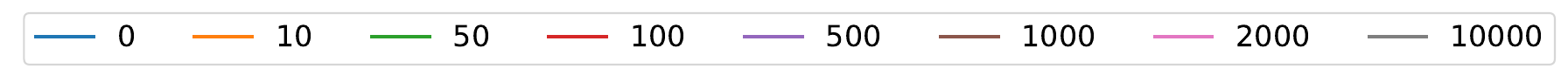}
\begin{subfigure}[b]{.45\textwidth}
    \centering
    \includegraphics[width=\linewidth, trim= 0 0 0 0, clip]{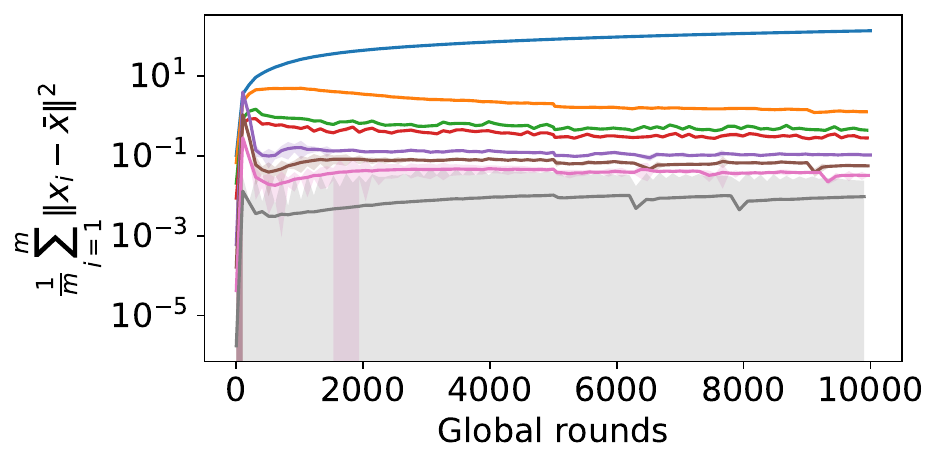}
    \caption{\footnotesize 
    SVHN data set.
    Heterogeneous but stationary intra-period probability, detailed in~\prettyref{fig: static consensus display}.
    }
    \label{fig: svhn static consensus periodic}
\end{subfigure}
\hfill
\begin{subfigure}[b]{.45\textwidth}
    \centering
    \includegraphics[width=\linewidth, trim= 0 0 0 0, clip]{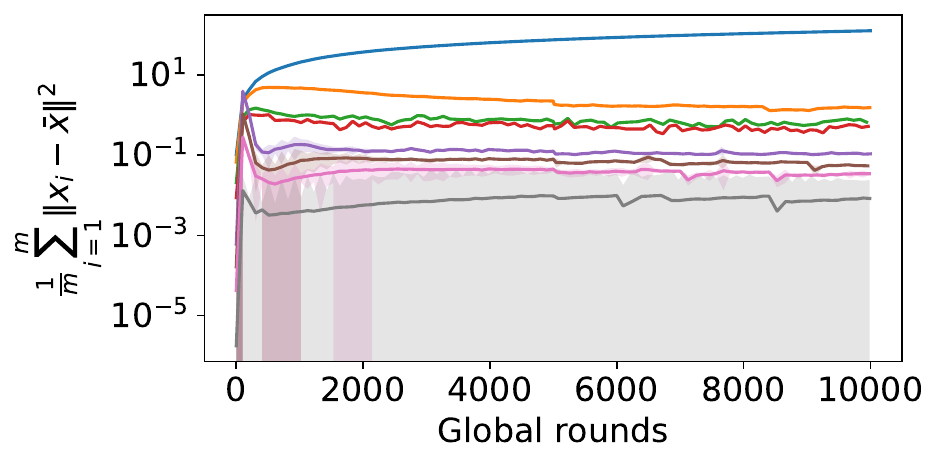}
    \caption{\footnotesize
    SVHN data set.
    Heterogeneous and nonstationary intra-period $p_i^t$, detailed in~\prettyref{fig: staircase consensus display}.  
    }
    \label{fig: svhn stair consensus periodic}
\end{subfigure}
\begin{subfigure}[b]{.45\textwidth}
    \centering
    \includegraphics[width=\linewidth, trim= 0 0 0 0, clip]{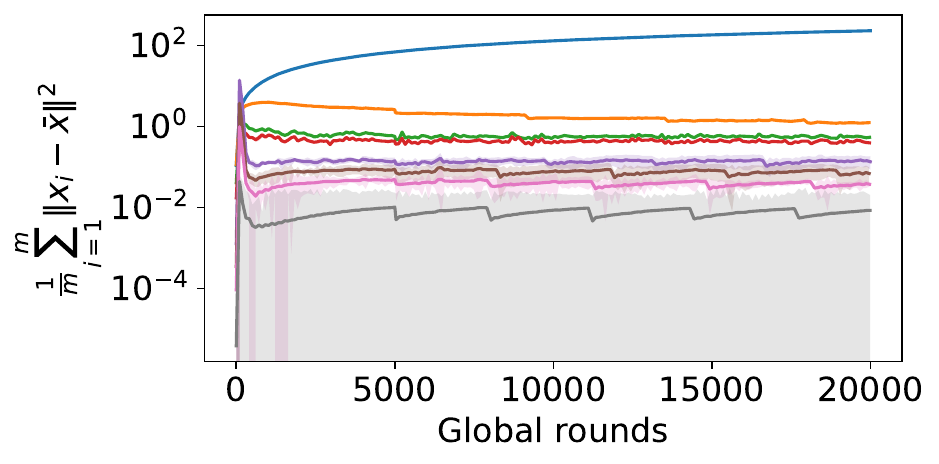}
    \caption{\footnotesize
    CIFAR-10 data set.
    Heterogeneous but stationary intra-period probability, detailed in~\prettyref{fig: static consensus display}.
    }
    \label{fig: cifar10 static consensus periodic}
\end{subfigure}
\hfill
\begin{subfigure}[b]{.45\textwidth}
    \centering
    \includegraphics[width=\linewidth, trim= 0 0 0 0, clip]{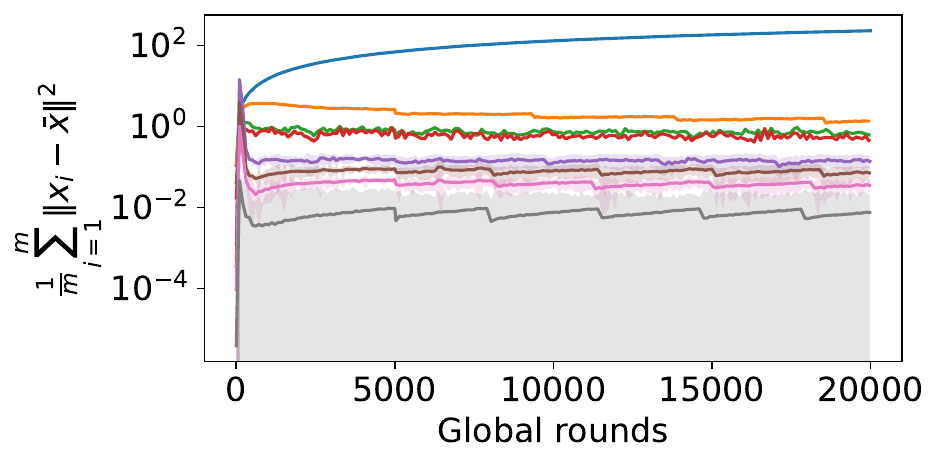}
    \caption{\footnotesize
    CIFAR-10 data set.
    Heterogeneous and non-stationary intra-period $p_i^t$, detailed in~\prettyref{fig: staircase consensus display}. 
    }
    \label{fig: cifar10 stair consensus periodic}
\end{subfigure}
\begin{subfigure}[b]{.45\textwidth}
\centering
\includegraphics[width=\linewidth, trim= 0 0 0 0, clip]{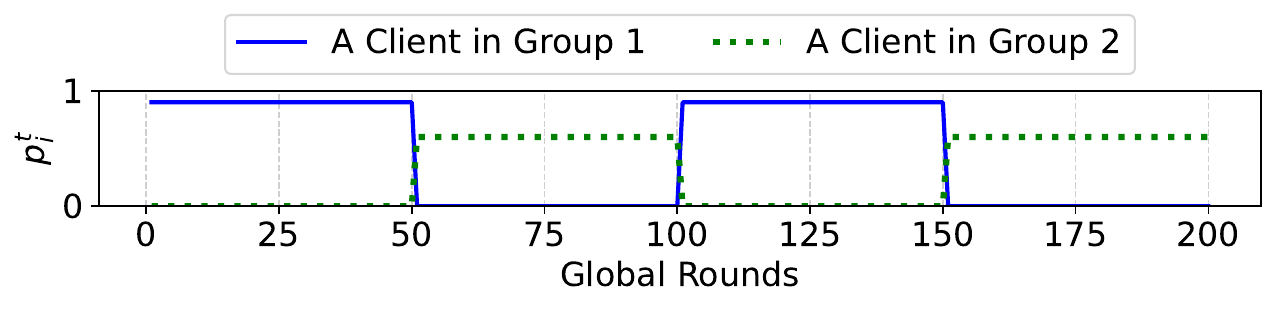}    
\caption{\footnotesize
An illustration of the client unavailability dynamics in Figs.~\ref{fig: svhn static consensus periodic} and \ref{fig: cifar10 static consensus periodic}.
The two client groups $\calM_1$ and $\calM_2$ are available alternately, where $p_i$'s remain static within each available period.
}
\label{fig: static consensus display}
\end{subfigure}
\hfill
\begin{subfigure}[b]{.45\textwidth}
\centering
\includegraphics[width=\linewidth]{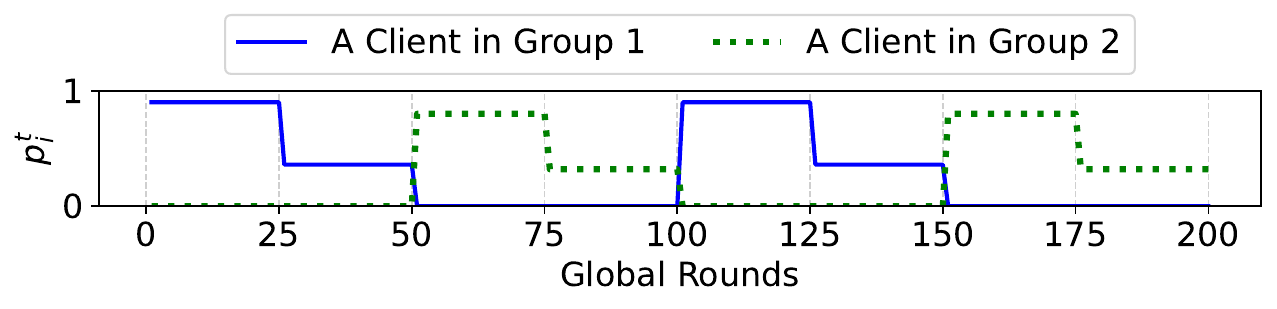}    
\caption{\footnotesize
An illustration of the client unavailability dynamics in Figs.~\ref{fig: svhn stair consensus periodic} and \ref{fig: cifar10 stair consensus periodic}.
The two client groups are available alternately, where $p_i^t$'s have a staircase pattern within each available window.
}
\label{fig: staircase consensus display}
\end{subfigure}
\caption{\footnotesize
Plots of the consensus errors on a logarithmic scale with $m = 100$ clients. 
The 100 clients are evenly divided into two non-overlapping groups $\calM_1$ and $\calM_2$. 
The global rounds can be partitioned into periods, each with length $50$ rounds. 
In all odd periods, the available probability $p^t_i>0$ if $i\in \calM_1$ while $p^t_i = 0$ otherwise. 
In all even periods, the available probability $p^t_i>0$ if $i\in \calM_2$ while $p^t_i = 0$ otherwise. 
We train CNN networks on the SVHN data set and CIFAR-10 data set for $T=10000$ and $T=20000$ rounds, respectively.
When a client becomes available, it performs $s=10$ steps of local computation.
The results are obtained under 3 random seeds and sampled every $\tP$ round. 
The solid curves differ from each other by the values of $k$, and the shaded areas plot the standard deviation.
It can be observed from the plots that clients fail to reach a consensus when $k=0$.
}
\label{fig: consensus periodic}
\end{figure}
\begin{figure}[!t]
\includegraphics[width=\linewidth, trim= 0 9cm 0 0, clip]{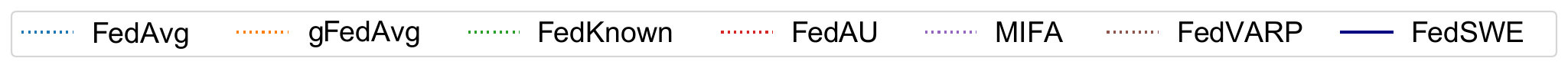}
    \begin{subfigure}[b]{.45\textwidth}
        \centering
        \includegraphics[width=.9\linewidth, trim= 0 0 0 0, clip]{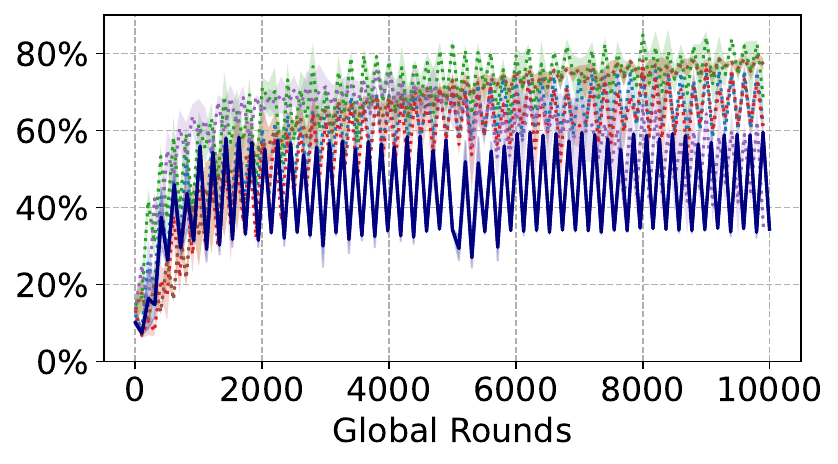}
        \caption{\footnotesize 
        No interpolation $(k=0)$.
        Heterogeneous but stationary intra-period probability $p_i$ (\prettyref{fig: static consensus display 2}).
        }
        \label{fig: svhn static test periodic v0}
    \end{subfigure}
    \hfill
    \begin{subfigure}[b]{.45\textwidth}
        \centering
        \includegraphics[width=.9\linewidth, trim= 0 0 0 0, clip]{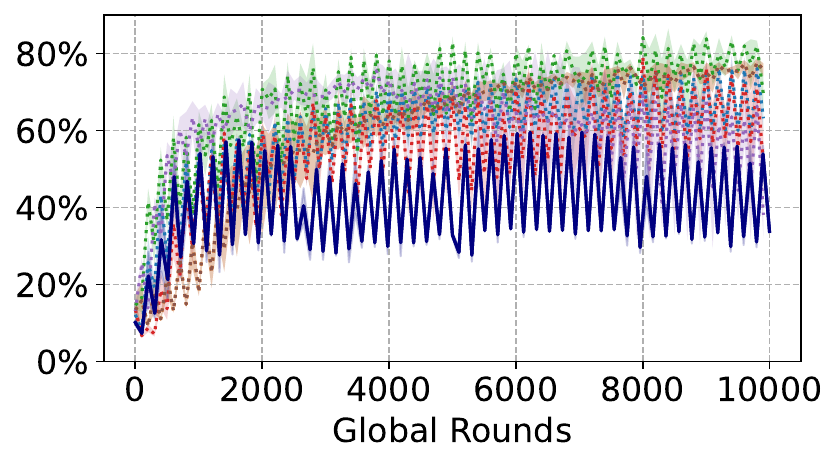}
        \caption{\footnotesize
        No interpolation $(k=0)$.
        Heterogeneous and non-stationary intra-period probability $p_i^t$ (\prettyref{fig: staircase consensus display 2}).
        }
        \label{fig: svhn stair test periodic v0}
    \end{subfigure}

    \begin{subfigure}[b]{.45\textwidth}
        \centering
        \includegraphics[width=.8\linewidth, trim= 0 0 0 0, clip]{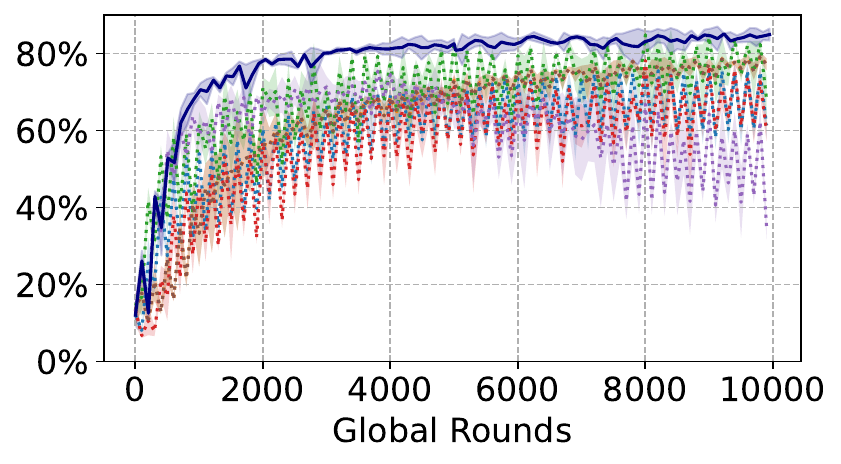}
        \caption{\footnotesize
        With interpolation $(k=100)$.
        Heterogeneous but stationary intra-period probability $p_i$ (\prettyref{fig: static consensus display 2}).
        }
        \label{fig: svhn static test periodic v100}
    \end{subfigure}
    \hfill
    \begin{subfigure}[b]{.45\textwidth}
        \centering
        \includegraphics[width=.8\linewidth, trim= 0 0 0 0, clip]{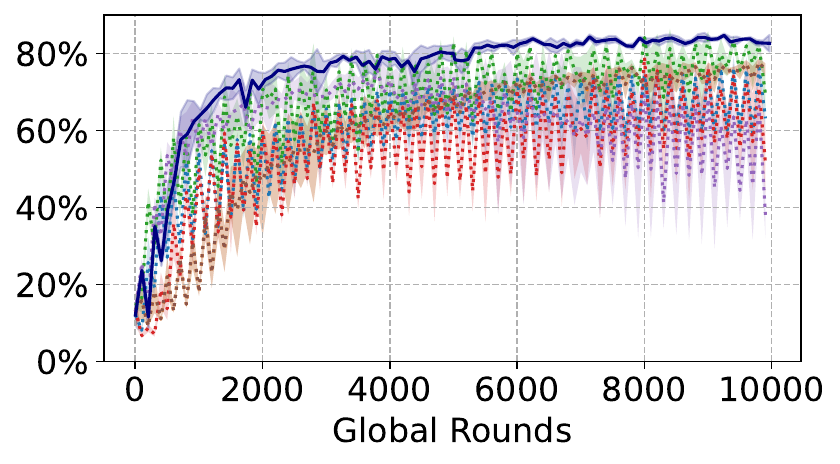}
        \caption{\footnotesize
        With interpolation $(k=100)$.
        Heterogeneous and non-stationary intra-period probability $p_i^t$ (\prettyref{fig: staircase consensus display 2}).
        }
        \label{fig: svhn stair test periodic v100}
    \end{subfigure}
    
    \begin{subfigure}[b]{.45\textwidth}
    \centering
    \includegraphics[width=\linewidth, trim= 0 0 0 0, clip]{elements/figure/periodic_exps/consensus/periodic_static_p.pdf}    
    \caption{\footnotesize
    The two client groups are available alternately; however, the clients' $p_i$'s remain static within each available window.
    }
    \label{fig: static consensus display 2}
    \end{subfigure}
    \hfill
    \begin{subfigure}[b]{.45\textwidth}
    \centering
    \includegraphics[width=\linewidth]{elements/figure/periodic_exps/consensus/periodic_stair_p.pdf}    
    \caption{\footnotesize
    The two client groups are available alternately; however, the clients' $p_i^t$'s show a staircase pattern within each available window.
    }
    \label{fig: staircase consensus display 2}
    \end{subfigure}
    \caption{\footnotesize
    Test accuracy results with $m = 100$ clients, who are divided into two non-overlapping but evenly sized groups. 
    We train CNN networks on the SVHN data set for $T=10000$ rounds.
    When a client becomes available, it performs $s=10$ steps of local computation.
    The results are obtained under 3 random seeds and sampled every $\tP$ round. 
    The curves plot the averaged results, while the shaded areas plot the standard deviation.
    }
    \label{fig: svhn test periodic}
\end{figure}
\begin{figure}[!t]
    \centering
    \begin{subfigure}[b]{.45\textwidth}
    \centering
    \includegraphics[width=.9\linewidth, trim= 0 0 0 0, clip]{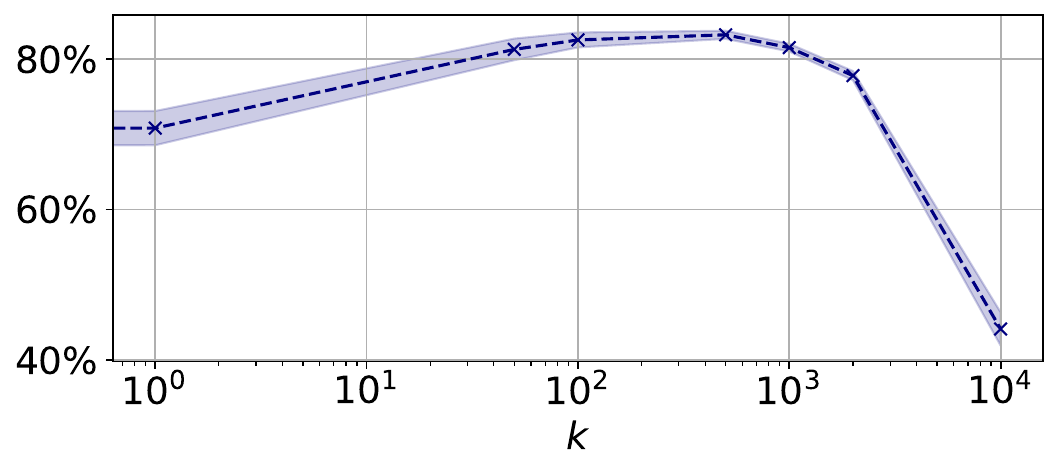}
    \caption{\footnotesize
    Heterogeneous but stationary intra-period probability $p_i$ (\prettyref{fig: static consensus display 2}).
    }
    \label{fig: svhn stair test periodic acc}
    \end{subfigure}
    \hfill
    \begin{subfigure}[b]{.45\textwidth}
    \centering
    \includegraphics[width=.9\linewidth, trim= 0 0 0 0, clip]{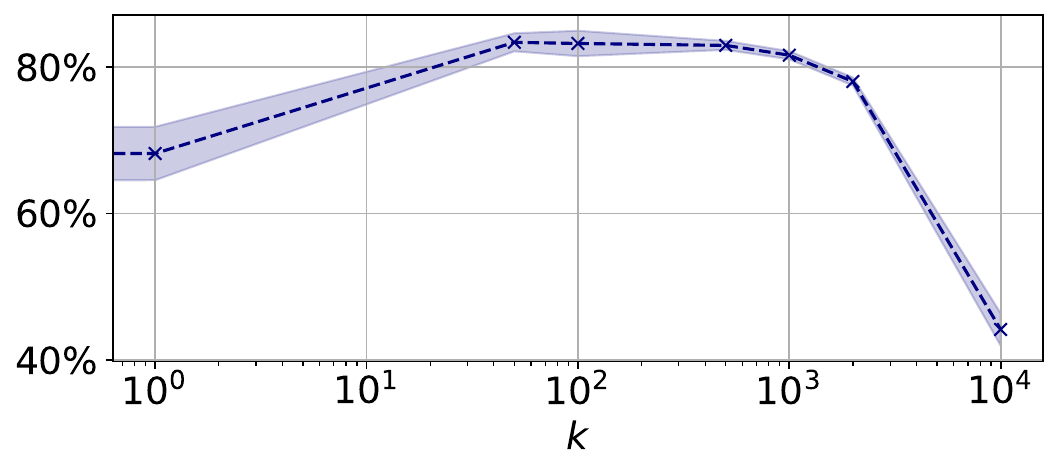}
    \caption{\footnotesize
    Heterogeneous and non-stationary intra-period probability $p_i^t$ (\prettyref{fig: staircase consensus display 2}).
    }
    \label{fig: svhn static test periodic acc}
    \end{subfigure} 
    \caption{\footnotesize
    Test accuracy of~\FedAPM~with different $k$'s on SVHN data set under different unavailability dynamics.
    The reported results are averaged over the last 500 rounds.
    Consistent with~\prettyref{rmk: convergence results}, we observe an initial increase in test accuracy, followed by a decline, peaking around $k = m = 100$.
    }
    \label{fig: svhn scale test acc}
\end{figure}
\begin{figure}[!t]
    \centering
    \begin{subfigure}[b]{.45\textwidth}
    \centering
    \includegraphics[width=.9\linewidth, trim= 0 0 0 0, clip]{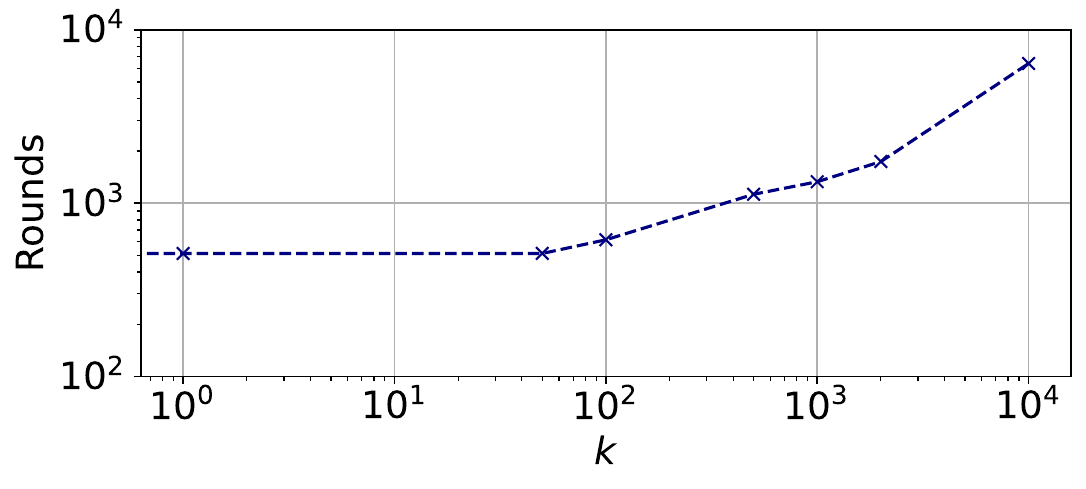}
    \caption{\footnotesize
    Heterogeneous yet stationary intra-period probability $p_i$ (\prettyref{fig: static consensus display 2}). 
    }
    \label{fig: svhn stair test periodic convergence}
    \end{subfigure}
    \hfill
    \begin{subfigure}[b]{.45\textwidth}
    \centering
    \includegraphics[width=.9\linewidth, trim= 0 0 0 0, clip]{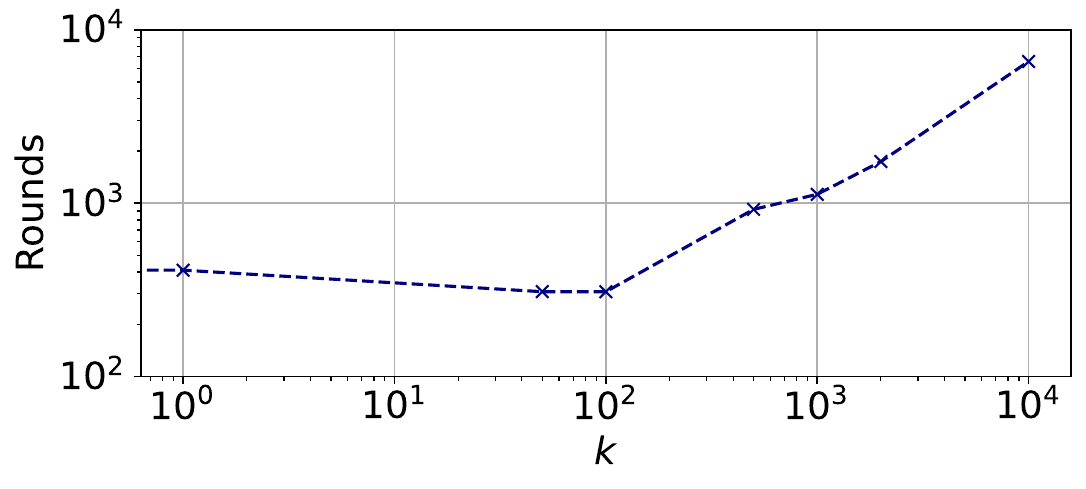}
    \caption{\footnotesize
    Heterogeneous and non-stationary intra-period probability $p_i^t$ (\prettyref{fig: staircase consensus display 2}).
    }
    \label{fig: svhn static test periodic convergence}
    \end{subfigure} 
    \caption{\footnotesize
    The number of rounds needed to achieve 40\% test accuracy of~\FedAPM~with various $k$ values on the SVHN data set under different unavailability dynamics. 
    We can observe a similar trend as in~\prettyref{fig: svhn scale test acc} that a slight speedup in the beginning but a significant slowdown after $k = m$.
    }
    \label{fig: svhn scale convergence}
\end{figure}
\begin{figure}[!t]
\includegraphics[width=\linewidth, trim= 0 9cm 0 0, clip]{elements/figure/periodic_exps/test_acc/svhn/legend.pdf}
    \begin{subfigure}[b]{.45\textwidth}
        \centering
        \includegraphics[width=.9\linewidth, trim= 0 0 0 3.1cm, clip]{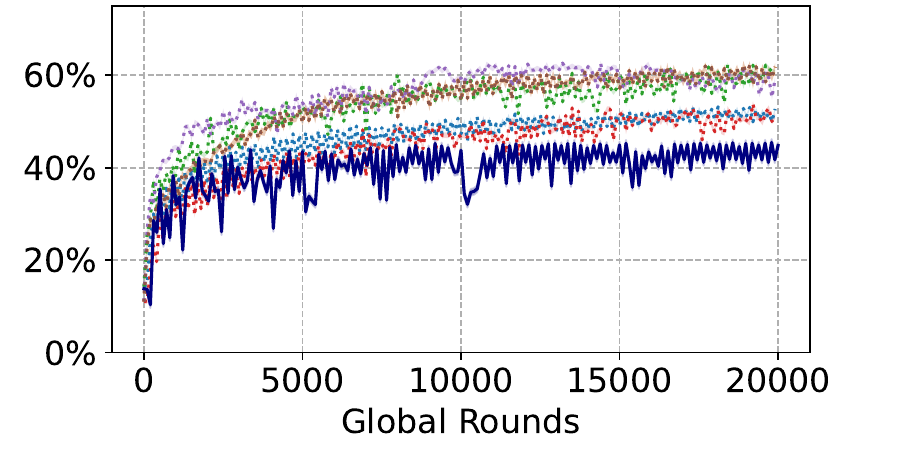}
        \caption{\footnotesize 
        No interpolation $(k=0)$.
        Heterogeneous but stationary intra-period probability $p_i$ (\prettyref{fig: static consensus display 3}).
        }
        \label{fig: cifar10 static test periodic v0}
    \end{subfigure}
    \hfill
    \begin{subfigure}[b]{.45\textwidth}
        \centering
        \includegraphics[width=.9\linewidth, trim= 0 0 0 3cm, clip]{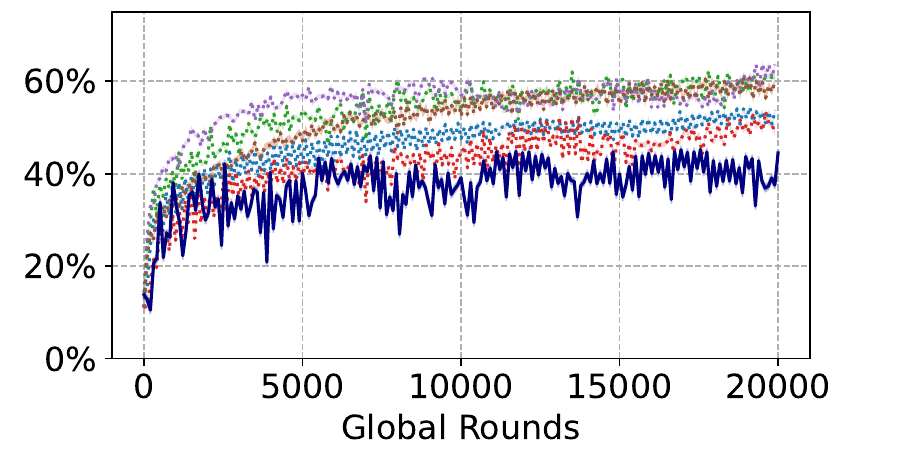}
        \caption{\footnotesize
        No interpolation $(k=0)$.
        Heterogeneous and non-stationary intra-period probability $p_i^t$ (\prettyref{fig: staircase consensus display 3}).
        }
        \label{fig: cifar10 stair test periodic v0}
    \end{subfigure}

    \begin{subfigure}[b]{.45\textwidth}
        \centering
        \includegraphics[width=.8\linewidth, trim= 0 0 0 0, clip]{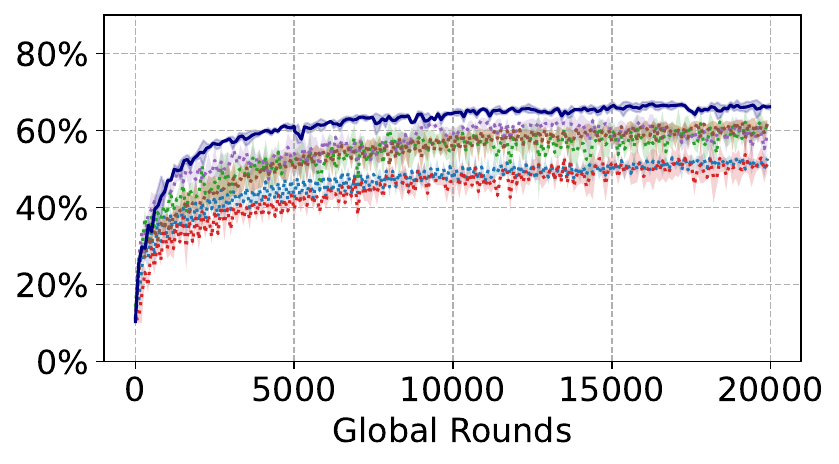}
        \caption{\footnotesize
        With interpolation $(k=100)$.
        Heterogeneous but stationary intra-period probability $p_i$ (\prettyref{fig: static consensus display 3}).
        }
        \label{fig: cifar10 static test periodic v100}
    \end{subfigure}
    \hfill
    \begin{subfigure}[b]{.45\textwidth}
        \centering
        \includegraphics[width=.8\linewidth, trim= 0 0 0 0, clip]{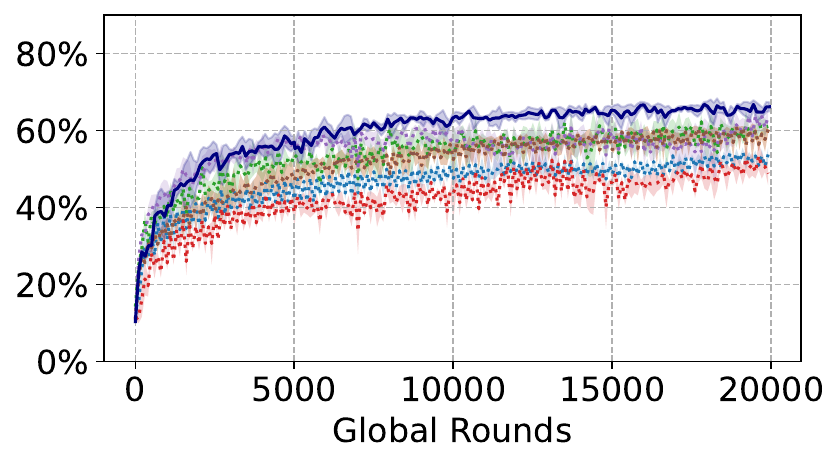}
        \caption{\footnotesize
        With interpolation $(k=100)$.
        Heterogeneous and non-stationary intra-period probability $p_i^t$ (\prettyref{fig: staircase consensus display 3}).
        }
        \label{fig: cifar10 stair test periodic v100}
    \end{subfigure}

    \begin{subfigure}[b]{.45\textwidth}
    \centering
    \includegraphics[width=\linewidth, trim= 0 0 0 0, clip]{elements/figure/periodic_exps/consensus/periodic_static_p.pdf}    
    \caption{\footnotesize
    The two client groups are available alternately; however, the clients' $p_i$'s remain static within each available window.
    }
    \label{fig: static consensus display 3}
    \end{subfigure}
    \hfill
    \begin{subfigure}[b]{.45\textwidth}
    \centering
    \includegraphics[width=\linewidth]{elements/figure/periodic_exps/consensus/periodic_stair_p.pdf}    
    \caption{\footnotesize
    The two client groups are available alternately; however, the clients' $p_i^t$'s show a staircase pattern within each available window.
    }
    \label{fig: staircase consensus display 3}
    \end{subfigure}
    \caption{\footnotesize
    Test accuracy results with $m = 100$ clients, who are divided into two non-overlapping but evenly sized groups. 
    Only a single client group is available in each training round.
    Inside an available window of $P$ rounds, a client $i$ in each group is available with probability $p_i^t$.
    We train CNN networks on the CIFAR-10 data set for $T=20000$ rounds.
    When a client becomes available, it performs $s=10$ steps of local computation.
    The results are obtained under 3 random seeds and sampled every $P$ round. 
    The curves plot the averaged results, while the shaded areas plot the standard deviation.
    }
    \label{fig: cifar10 test periodic}
\end{figure}
\begin{figure}[!t]
    \centering
    \begin{subfigure}[b]{.45\textwidth}
    \centering
    \includegraphics[width=.9\linewidth, trim= 0 0 0 0, clip]{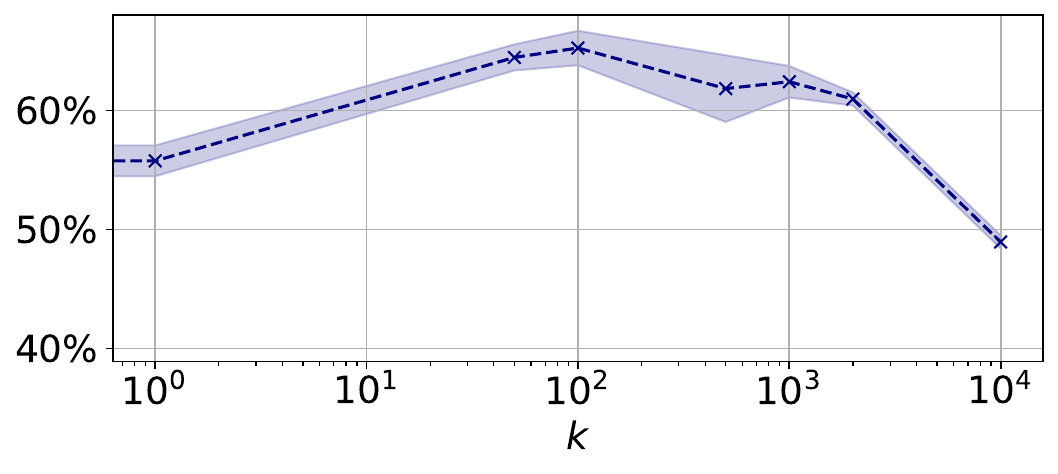}
    \caption{\footnotesize
    Heterogeneous but stationary intra-period probability $p_i$ (\prettyref{fig: static consensus display 3}).
    }
    \label{fig: cifar10 stair test periodic acc}
    \end{subfigure}
    \hfill
    \begin{subfigure}[b]{.45\textwidth}
    \centering
    \includegraphics[width=.9\linewidth, trim= 0 0 0 0, clip]{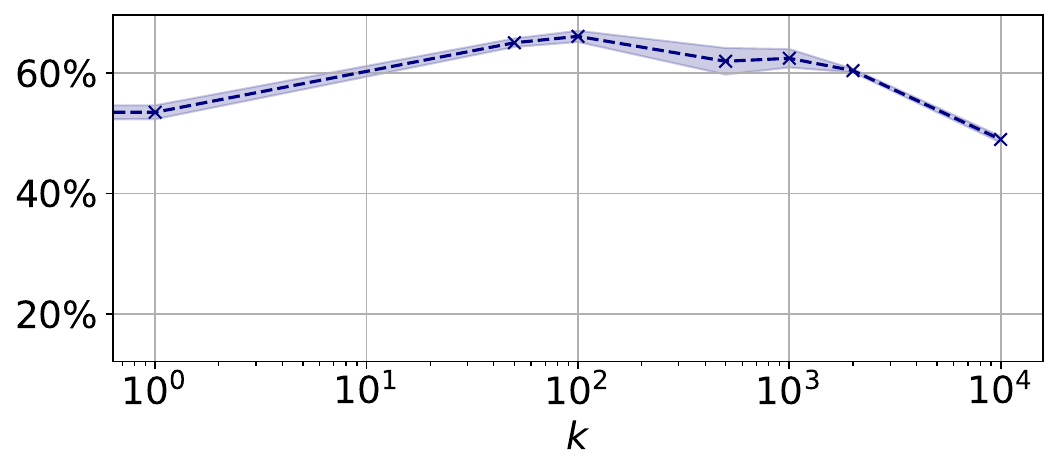}
    \caption{\footnotesize
    Heterogeneous and non-stationary intra-period probability $p_i^t$ (\prettyref{fig: staircase consensus display 3}).
    }
    \label{fig: cifar10 static test periodic acc}
    \end{subfigure} 
    \caption{\footnotesize
    Test accuracy of~\FedAPM~with different $k$'s on CIFAR-10 data set under different unavailability dynamics.
    The reported results are averaged over the last 500 rounds.
    Consistent with~\prettyref{rmk: convergence results}, we observe an initial increase in test accuracy, followed by a decline, peaking around $k = m = 100$.
    }
    \label{fig: cifar10 scale test acc}
\end{figure}
\begin{figure}[!t]
    \centering
    \begin{subfigure}[b]{.45\textwidth}
    \centering
    \includegraphics[width=.9\linewidth, trim= 0 0 0 0, clip]{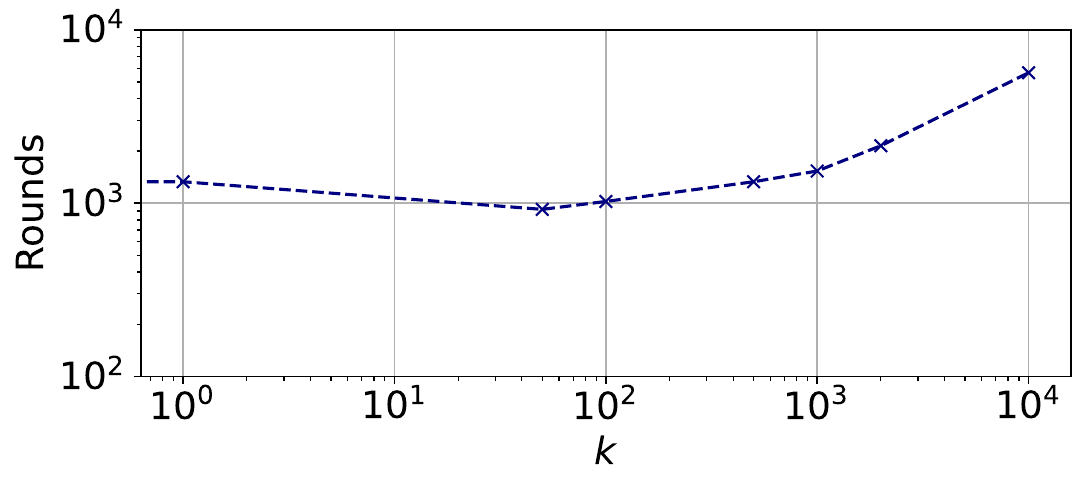}
    \caption{\footnotesize
    Heterogeneous yet stationary intra-period probability $p_i$ (\prettyref{fig: static consensus display 3}). 
    }
    \label{fig: cifar10 stair test periodic convergence}
    \end{subfigure}
    \hfill
    \begin{subfigure}[b]{.45\textwidth}
    \centering
    \includegraphics[width=.9\linewidth, trim= 0 0 0 0, clip]{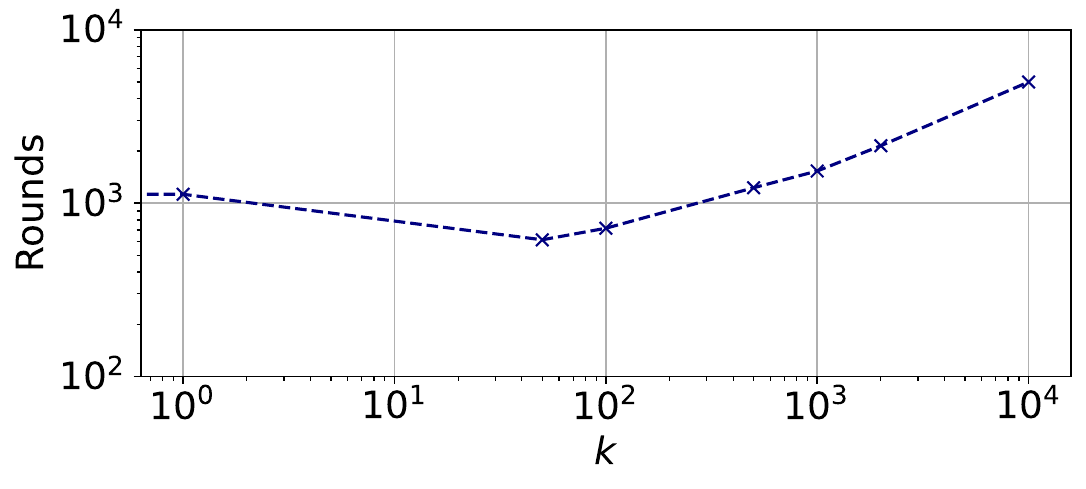}
    \caption{\footnotesize
    Heterogeneous and non-stationary intra-period probability $p_i^t$ (\prettyref{fig: staircase consensus display 3}).
    }
    \label{fig: cifar10 static test periodic convergence}
    \end{subfigure} 
    \caption{\footnotesize
    The number of rounds needed to achieve 40\% test accuracy of~\FedAPM~with various $k$ values on CIFAR-10 data set under different unavailability dynamics. 
    We can observe a similar trend as in~\prettyref{fig: cifar10 scale test acc} that a slight speedup in the beginning but a significant slowdown after $k = m$.
    }
    \label{fig: cifar10 scale convergence}
\end{figure}
We start from a general case where the length of the sliding window $P_{\delta} > 1$ in~\prettyref{ass: prob lower bound}.

\noindent
{\bf Data sets and data heterogeneity.}
We perform the experiments on SVHN \citep{netzer2011readingdigits} and CIFAR-10 \citep{krizhevsky2009learning} data sets.
Similar to~\prettyref{prop: counter single}, we divide clients into two groups $\calM_1$ and $\calM_2$, each with $50$ clients. %
Each group of clients collectively hold 5 classes of images from the original data set, non-overlapping with the other group. 
To emulate highly heterogeneous local data distributions within each group, the images are assigned to individual clients $i$ according to $\nu_i \sim \mathsf{Dirichlet}(\alpha = 0.1)$ \citep{hsu2019measuring,wang2022,wang2023lightweight}.

\noindent
{\bf Non-stationary client unavailability with $\tP > 1$.}
We evaluate two non-stationary unavailable dynamics---static and staircase probabilistic trajectories---both with $P_{\delta}=100$. 
Illustrative plots can be found in Figs.~\ref{fig: static consensus display} and \ref{fig: staircase consensus display}.
The non-stationary dynamics are motivated by real-world federated learning participation statistics and by generalizing the existing participation patterns, such as cyclic participation~\citep{pmlr-v202-cho23b,wang2023lightweight}.
Formally, let $f_i(t)$ be a time-dependent function under the specific non-stationary dynamics, and $p_i = \iprod{\nu_i}{\phi}$, where $\nu_i \sim \mathsf{Dirichlet}(\alpha = 0.1)$, and $\phi$ characterizes the unbalanced contribution of different image classes to the generated probabilities.
The unavailability dynamics of client group $\calM_1$ and $\calM_2$ are illustrated in~\eqref{eq: m1 dynamic} and in~\eqref{eq: m2 dynamic}, respectively.

\vspace{.5\baselineskip}
\noindent
\begin{minipage}{0.45\textwidth}
\fbox{\parbox{\linewidth}{
\begin{small}
For $i \in \calM_1$, we have
\begin{align}
    \label{eq: m1 dynamic}
    p_i^t &= 
    \begin{cases}
        p_i \cdot f_i(t),~ &\text{if}~ (t~\text{mod}~\tP) \le \frac{\tP}{2}; \\
        0, ~&\text{otherwise}.
    \end{cases}
\end{align}
\end{small}
}}
\end{minipage}
\hfill
\begin{minipage}{0.45\textwidth}
\fbox{\parbox{\linewidth}{
\begin{small}
For $j \in \calM_2$, we have
\begin{align}
    \label{eq: m2 dynamic}
    p_j^t &= 
    \begin{cases}
        p_j \cdot f_j(t),~ &\text{if}~ (t~\text{mod}~\tP) > \frac{\tP}{2}; \\
        0, ~&\text{otherwise}.
    \end{cases}
\end{align}
\end{small}
}}
\end{minipage}
\vspace{.5\baselineskip}

Each element of $[\phi]_c$ is drawn from $\mathsf{Uniform}(0,\bm{\Phi}_c)$, where a smaller $\bm{\Phi}_c$ leads to a less significant contribution of that image class. It is immediately clear that the coupling of local data distribution $\nu_i \sim \mathsf{Dirichlet(\alpha=0.1)}$ and class contribution $\phi$ leads to {\em non-}independent $p_i$'s. 
Although the non-independence setup violates our theoretical analysis, we observe that~\FedAPM~retains its outperformance.
Correlating the local data distribution and the probability of client availability is a common practice in the prior literature.
For example, \cite{gu2021fast} experiment with a formula for $p_i$ so that clients that hold images of smaller digits participate less frequently. \cite{wang2023lightweight} construct $p_i$ as an inner product of the clients' local data distribution $\nu_i$ and an external distribution $\bm{\Phi}^\prime$.

We highlight that the periodic unavailability dynamics evaluated in our work are more challenging, \eg, than~\citep{wang2022}, where they select a fixed number of $S$ clients out of the available client group to participate in each round $t$ uniformly at random.
In our work, we may have fewer than $S$ available clients in any round $t$ due to the heterogeneity in the base probability $p_i$ and randomness in the Bernoulli sampling process; therefore, a fixed size of sampling clients in each round is not guaranteed.

\noindent
{\bf Benchmark algorithms.}
We compare~\FedAPM~with six baseline algorithms,
including~\FedAvg~over active clients \citep{mcmahan2017communication},~{\tt gFedAvg}~\citep{wang2022}, {\tt FedAvg} with known probability ({\tt FedKnown}) \citep{perazzone2022communication},~\FedAU~\citep{wang2023lightweight},~\MIFA~\citep{gu2021fast}~and~\FedVARP~\citep{jhunjhunwala2022fedvarp}. 
The details of the algorithms are deferred to~\prettyref{app: numerical}.

\noindent
{\bf Necessity of interpolation ($k>0$).}
Recall that we show in~\prettyref{prop: counter single} that interpolation is necessary for information diffusion over rounds under periodic unavailability.
To validate such a claim, we show in~\prettyref{fig: consensus periodic} that clients fail to reach a consensus when $k=0$. Specifically, instead of decaying, we observe that the consensus errors blow up in the plots.
In contrast, the interpolation carries global updates from round to round and eventually allows clients to correct bias.
Furthermore, as $k$ increases, we observe a smaller consensus error, which implies better client connectivity.
Yet, as we will show next, excessively increasing $k$ will inevitably lead to an unnecessary slowdown in convergence.

\noindent
{\bf Performance discussions.}
We can observe from Figs.~\ref{fig: svhn stair test periodic v0},~\ref{fig: svhn static test periodic v0},~\ref{fig: cifar10 static test periodic v0} and~\ref{fig: cifar10 stair test periodic v0} that most of the algorithms, including our conference version:~\FedAPM~without interpolation ($k=0$), suffer from the challenging periodic non-stationary dynamics and experience high fluctuations.
As we have illustrated in~\prettyref{prop: counter single},~\FedAPM~without interpolation fails to diffuse information across rounds.
For baseline algorithms on the SVHN data set, {\tt FedKnown} attains the best peak accuracy, while~\FedVARP~and {\tt gFedAvg} yield the smoothest curves yet with less accurate predictions.
In sharp contrast,~\FedAPM~with $k=100$ generates relatively smooth trajectories and outperforms all the baseline algorithms on both SVHN (Figs.~\ref{fig: svhn static test periodic v100} and~\ref{fig: svhn stair test periodic v100}) and CIFAR-10 data sets (Figs.~\ref{fig: cifar10 static test periodic v100} and~\ref{fig: cifar10 stair test periodic v100}).
For the SVHN data set,~\FedAPM~with $k=100$ attains comparable accuracy as the peak accuracy of {\tt FedKnown} yet with more consistent performance. For the CIFAR-10 data set,~\FedAPM~with $k=100$ obtains even better accuracy than {\tt FedKnown}. It is worth noting that, despite the impressive empirical performance of {\tt FedKnown}, its analysis requires strictly positive probabilities~\citep{perazzone2022communication}, and its implementation adopts $p_i^t$'s as a known priori.~\FedAU~is provably robust to stationary dynamics; yet, we consider non-stationary availability here.
It is a bit surprising that~\FedAPM~surpasses~\MIFA~and~\FedVARP, which both require heavy memory of size $O(md)$; however, the interpolation in~\FedAPM~only requires light memory of size $O(d)$.
Under the periodic unavailability dynamics, the clients in the inactive client group can be unavailable for quite a long period of global rounds; therefore, their gradients from the most recent availability may not be a good approximation of their latest fresh gradients if they are available.

\noindent
{\bf Effects of interpolation coefficient $k$.}
Figs.~\ref{fig: svhn scale test acc} and~\ref{fig: cifar10 scale test acc} compare the test accuracies of~\FedAPM~under different choices of $k$.
The results show that increasing $k$ is not always beneficial. On the one hand,~\FedAPM~converges faster as $k$ approaches $m$, but it undergoes a substantial slowdown when $k$ continues to rise, as shown by the number of rounds required to reach 40\% test accuracy in Figs.~\ref{fig: svhn scale convergence} and~\ref{fig: cifar10 scale convergence}. 
On the other hand, the final test accuracy increases and then drops in Figs.~\ref{fig: svhn scale test acc} and~\ref{fig: cifar10 scale test acc}. 
This aligns with our discussions in~\prettyref{rmk: convergence results}, where we show that the convergence upper bound minimizes at a unique point.
However, the exact analytical form depends on parameters that are difficult, if not impossible, to obtain in practice.
Based on the results, we empirically recommend $k = \Theta(m)$ to strike a balance between performance and convergence speed.

\subsection{Non-stationary and Heterogeneous Unavailability with $P_{\delta} = 1$}
\label{sec: numerical P=1}
In this section, we investigate the special case of~\prettyref{ass: prob lower bound},
where $P_{\delta}=1$ as in~\citep{xiang2024efficient}, with interpolation providing an added benefit.
Our numerical results are presented in~\prettyref{tab: exp main text}, where we study additional availability dynamics and data sets. 
The results are partitioned into two parts, where the latter part includes algorithms aided by heavy memory or known statistics, such as~\MIFA,~\FedVARP~and {\tt FedKnown}.
\begin{table}[!t]
\centering
\caption{\footnotesize 
Results and comparisons on real-world data sets in the form of mean accuracy $\pm$ standard deviation and are obtained over 3 repetitions in different random seeds.
Results are averaged over the last $50$ rounds.
The total number of global rounds is 2000 
for SVHN, CIFAR-10 and CINIC-10.
Algorithms are categorized into two groups:
(1) ones {\bf not} aided by memory or known statistics;
(2) ones assisted by memory.
For a fair competition,
we {\bf boldface} the best accuracy in the first group,
while the second best is \underline{underlined}.
}
\label{tab: exp main text}
\resizebox{\linewidth}{!}{
\begin{footnotesize}
\begin{tabular}{c|c|p{1.8cm} p{1.8cm}|p{1.8cm} p{1.8cm}|p{1.8cm} p{1.8cm}}
    \toprule
    \multirow{2}{*}{\begin{tabular}{@{}c@{}}{\bf Unavailable} \\ {\bf Dynamics} \end{tabular}} &
    {\bf Data sets} &
    \multicolumn{2}{c|}{\bf SVHN} & 
    \multicolumn{2}{c|}{\bf CIFAR-10} & 
    \multicolumn{2}{c}{\bf CINIC-10} \\
    \cline{2-8}
    & 
    {\bf Algorithms}&
    \multicolumn{1}{c}{\bf Train} &
    \multicolumn{1}{c|}{\bf Test} &
    \multicolumn{1}{c}{\bf Train} &
    \multicolumn{1}{c|}{\bf Test} &
    \multicolumn{1}{c}{\bf Train} &
    \multicolumn{1}{c}{\bf Test} \\
    \midrule
     \multirow{6}{*}{
\begin{tabular}{@{}c@{}} 
{Stationary} \\[2ex]
\adjustbox{width=0.12\linewidth}{ \begin{overpic}[scale=1]{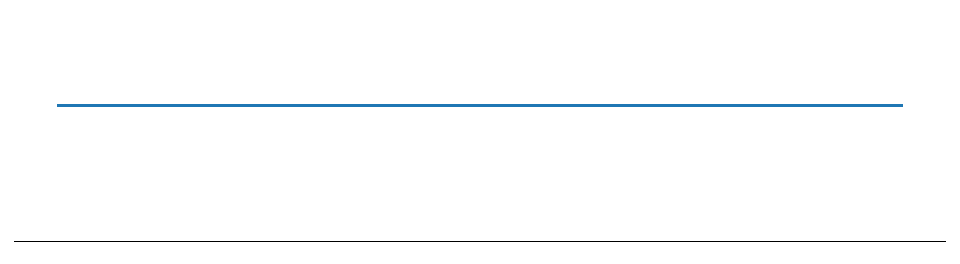} %
    \put(0,19){\fontsize{80}{70}\selectfont $p_i$}
    \put(95,5){\fontsize{60}{50}\selectfont $0$}
  \end{overpic}}
\end{tabular}} & 
\FedAPM~({\bf ours, $k=0$}) & 
{\bf 86.5} $\pm$ 0.7 \%&
{\bf 86.1} $\pm$ 0.7 \%&
\underline{68.1} $\pm$ 1.4 \%&
\underline{66.3} $\pm$ 1.1 \%&
\underline{47.9} $\pm$ 2.1 \%&
\underline{47.3} $\pm$ 2.0 \%
\\
&
\FedAPM~({\bf ours, $k=100$}) & 
\underline{86.3} $\pm$ 1.1\%&
\underline{85.4} $\pm$ 1.0\%&
{\bf 68.7} $\pm$ 1.0 \%&
{\bf 66.9} $\pm$ 0.9 \%&
{\bf 48.3} $\pm$ 1.7 \%&
{\bf 47.9} $\pm$ 1.6 \%
\\
& 
\FedAvg~over {\em active} & 
82.6  $\pm$ 1.0 \%&
82.4  $\pm$ 1.1 \%&
64.1  $\pm$ 1.9 \%&
62.9  $\pm$ 1.4 \%&
43.6  $\pm$ 2.4 \%&
43.1  $\pm$ 2.4 \%
\\
& 
\FedAvg~over {\em all} & 
76.1 $\pm$ 2.1 \%&
76.1 $\pm$ 2.4 \%&
55.8 $\pm$ 2.1 \%&
55.4 $\pm$ 1.8 \%&
38.4 $\pm$ 2.1 \%&
38.0 $\pm$ 2.1 \%
\\
& 
\gFedAvg& 
83.5 $\pm$ 1.5 \%&
83.0 $\pm$ 1.2\%&
64.5 $\pm$ 1.7\%&
63.8 $\pm$ 1.7\%&
45.3 $\pm$ 1.4\%&
44.9 $\pm$ 1.2\%
\\
& 
\FedAU & 
83.4 $\pm$ 1.0 \%&
83.2 $\pm$ 1.0 \%&
65.4 $\pm$ 1.4 \%&
64.1 $\pm$ 1.0 \%&
45.6  $\pm$ 1.5 \%&
45.2  $\pm$ 1.5 \%
\\
& 
\FAST & 
83.2  $\pm$ 0.7 \%&
83.2  $\pm$ 0.7 \%&
64.4 $\pm$ 1.1 \%&
63.5  $\pm$ 0.9 \%&
45.3  $\pm$ 1.2 \%& 
44.8  $\pm$ 1.2 \%
\\
\noalign{\vspace{.5mm}}
\cline{2-8}
\noalign{\vspace{.5mm}}
& 
\FedAvg~with {\em known} $p_i$'s  & 
86.1  $\pm$ 0.5 \%&
85.6  $\pm$ 0.5 \%&
65.4  $\pm$ 1.0 \%&
63.1  $\pm$ 0.9 \%&
45.0  $\pm$ 1.2 \%&
44.6  $\pm$ 1.1 \%
\\
& 
\MIFA~({\em memory aided}) & 
84.2  $\pm$ 0.5 \%&
84.1  $\pm$ 0.6 \%&
66.6  $\pm$ 0.8 \%&
65.3  $\pm$ 0.6 \% &
47.5  $\pm$ 0.5 \%&
46.9  $\pm$ 0.5 \%
\\
& 
\FedVARP~({\em memory aided}) & 
84.6  $\pm$ 0.2 \%&
84.3  $\pm$ 0.1 \%&
67.5  $\pm$ 0.2 \%&
66.3  $\pm$ 0.3 \% &
47.8  $\pm$ 0.2 \%&
47.2  $\pm$ 0.2 \% \\
    
    \toprule
    \multirow{6}{*}{
\begin{tabular}{@{}c@{}} 
{\bf Non}-stationary \\ 
({\bf Staircase}) \\[2ex]
\adjustbox{width=0.12\linewidth}{  \begin{overpic}[scale=1]{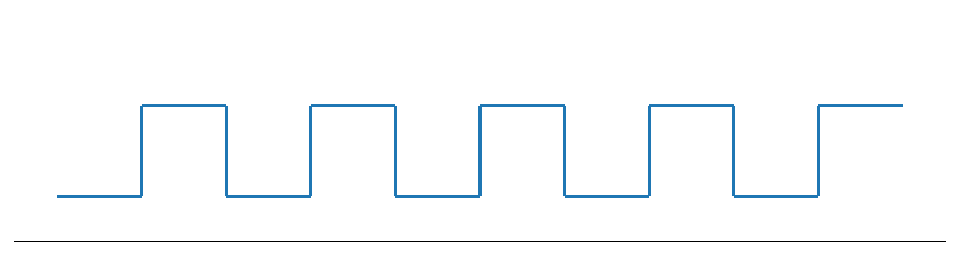} %
    \put(0,14){\fontsize{80}{70}\selectfont $p_i^t$}
    \put(95,5){\fontsize{60}{50}\selectfont $0$}
  \end{overpic}} 
\end{tabular}}& 
\FedAPM~({\bf ours}, $k=0$)& 
\underline{85.9} $\pm$ 0.8 \%&
\underline{85.6} $\pm$ 1.0 \%&
\underline{67.7} $\pm$ 1.3 \%&
\underline{66.0} $\pm$ 1.2 \%&
\underline{47.5} $\pm$ 2.0 \%&
\underline{46.9} $\pm$ 2.0 \% \\
&
\FedAPM~({\bf ours, $k=100$}) & 
{\bf 86.0} $\pm$ 1.2 \%&
{\bf 85.9} $\pm$ 0.7 \%&
{\bf 67.8} $\pm$ 1.0 \%&
{\bf 66.1} $\pm$ 1.2 \%&
{\bf 47.7} $\pm$ 1.5 \%&
{\bf 47.1} $\pm$ 1.4 \%
\\
& 
\FedAvg~over {\em active} & 
82.5  $\pm$ 1.0 \%&
82.4  $\pm$ 0.9 \%&
64.2  $\pm$ 1.8 \%&
63.0  $\pm$ 1.4 \%&
43.7  $\pm$ 2.0 \%&
42.3  $\pm$ 2.2 \%
\\
& 
\FedAvg~over {\em all} & 
75.9 $\pm$ 2.1 \%&
75.9 $\pm$ 2.3 \%&
55.7 $\pm$ 2.1 \%&
55.4 $\pm$ 1.8 \%&
38.4 $\pm$ 2.0 \%&
37.9 $\pm$ 2.0 \%
\\
&
\gFedAvg& 
83.1 $\pm$ 1.3 \%&
83.0 $\pm$ 1.1 \%&
65.0 $\pm$ 1.6 \%&
64.9 $\pm$ 1.5 \%&
45.3 $\pm$ 1.4 \%&
44.9 $\pm$ 1.3 \%
\\
& 
\FedAU & 
83.6 $\pm$ 0.8 \%&
83.4 $\pm$ 0.8 \%&
65.2 $\pm$ 1.7 \%&
63.9 $\pm$ 1.5 \%&
45.7  $\pm$ 1.5 \%&
45.1  $\pm$ 1.5 \%
\\
& 
\FAST &
83.1  $\pm$ 0.6 \%&
83.1  $\pm$ 0.6 \%&
64.3 $\pm$  1.1\%&
63.3  $\pm$ 0.9 \%&

45.2  $\pm$ 1.2 \%& 
44.8  $\pm$ 1.2 \%
\\
\noalign{\vspace{.5mm}}
\cline{2-8}
\noalign{\vspace{.5mm}}
& 
\FedAvg~with {\em known} $p_i^t$'s  & 
85.8  $\pm$ 0.8 \%&
85.2  $\pm$ 0.9 \%&
68.0  $\pm$ 1.6 \%&
66.1  $\pm$ 1.8 \%&
45.0  $\pm$ 1.1 \%&
44.7  $\pm$ 1.0 \%
\\
& 
\MIFA~({\em memory aided}) & 
84.2  $\pm$ 0.5 \%&
84.0  $\pm$ 0.5 \%&
66.7  $\pm$ 0.7 \%&
65.3  $\pm$ 0.5 \% &
47.5  $\pm$ 0.5 \%&
46.9  $\pm$ 0.5 \%
\\
& 
\FedVARP~({\em memory aided}) & 
84.6  $\pm$ 0.2 \%&
84.3  $\pm$ 0.3 \%&
67.3  $\pm$ 0.3 \%&
66.1  $\pm$ 0.3 \% &
47.7  $\pm$ 0.2 \%&
47.2  $\pm$ 0.1 \% \\
    
    \toprule
    
    \multirow{6}{*}{
\begin{tabular}{@{}c@{}} 
\addlinespace[1ex]
{\bf Non}-stationary  \\ 
({\bf Sine}) \\[2ex]
\adjustbox{width=0.12\linewidth}{  \begin{overpic}[scale=1]{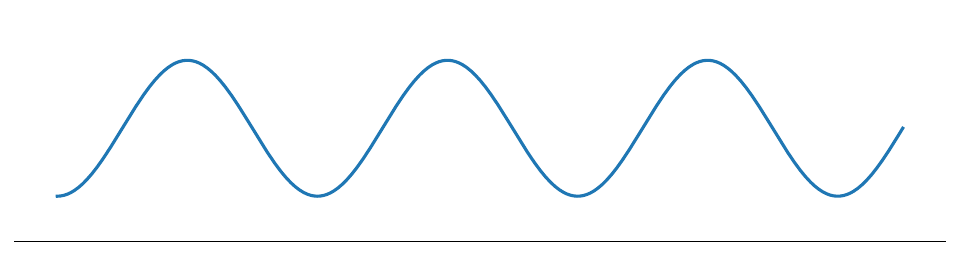} %
    \put(0,14){\fontsize{80}{70}\selectfont $p_i^t$}
    \put(95,5){\fontsize{60}{50}\selectfont $0$}
  \end{overpic}}
\end{tabular}} & 
\FedAPM~({\bf ours}, $k=0$)& 
\underline{85.7} $\pm$ 0.9 \%&
\underline{85.6} $\pm$ 0.9 \%&
\underline{64.9} $\pm$ 1.9 \%&
\underline{63.5} $\pm$ 2.0 \%&
\underline{46.4} $\pm$ 2.4 \%&
\underline{45.8} $\pm$ 2.4 \% \\
&
\FedAPM~({\bf ours, $k=100$}) & 
{\bf 85.9} $\pm$ 1.2 \%&
{\bf 85.8} $\pm$ 0.8 \%&
{\bf 65.8} $\pm$ 1.8 \%&
{\bf 64.2} $\pm$ 1.9 \%&
{\bf 47.2} $\pm$ 2.0 \%&
{\bf 46.7} $\pm$ 1.8 \%
\\
& 
\FedAvg~over {\em active} & 
82.1  $\pm$ 1.1 \%&
82.0  $\pm$ 1.3 \%&
63.3  $\pm$ 1.9 \%&
62.1  $\pm$ 1.8 \%&
43.1  $\pm$ 2.5 \%&
42.6  $\pm$ 2.5 \%
\\
& 
\FedAvg~over {\em all} & 
71.3 $\pm$ 2.5 \%&
71.3 $\pm$ 2.8 \%&
52.2 $\pm$ 2.4 \%&
52.1 $\pm$ 2.2 \%&
36.4 $\pm$ 2.0 \%&
36.0 $\pm$ 1.9 \%
\\
&
\gFedAvg&
83.6 $\pm$ 1.3\%&
83.5 $\pm$ 1.1 \%&
62.4 $\pm$ 1.2 \%&
62.2 $\pm$ 1.2 \%&
43.3 $\pm$ 1.1 \%&
42.9 $\pm$ 1.0 \%
\\
& 
\FedAU & 
82.5 $\pm$ 1.4 \%&
82.5 $\pm$ 1.3 \%&
64.2 $\pm$ 2.3 \%&
63.0 $\pm$ 1.9 \%&
44.4  $\pm$ 2.1 \%&
43.9  $\pm$ 2.1 \%
\\
& 
\FAST & 
82.3  $\pm$ 1.0 \%&
82.3  $\pm$ 1.0 \%&
63.1 $\pm$ 1.7 \%&
62.3  $\pm$ 1.5 \%&
44.1  $\pm$ 1.6 \%& 
43.7  $\pm$ 1.6 \%
\\
\noalign{\vspace{.5mm}}
\cline{2-8}
\noalign{\vspace{.5mm}}
& 
\FedAvg~with {\em known} $p_i^t$'s  & 
86.3  $\pm$ 1.0 \%&
86.0  $\pm$ 1.0 \%&
69.1  $\pm$ 1.2 \%&
67.3  $\pm$ 1.3 \%&
47.9  $\pm$ 1.5 \%&
47.4  $\pm$ 1.1 \%
\\
& 
\MIFA~({\em memory aided}) & 
84.2  $\pm$ 0.4 \%&
84.1  $\pm$ 0.4 \%&
66.6  $\pm$ 0.8 \%&
65.5  $\pm$ 0.6 \% &
47.4  $\pm$ 0.5 \%&
46.9  $\pm$ 0.4 \%
\\
& 
\FedVARP~({\em memory aided}) & 
84.5  $\pm$ 0.2 \%&
84.3  $\pm$ 0.1 \%&
67.4  $\pm$ 0.2 \%&
66.0  $\pm$ 0.3 \% &
47.7  $\pm$ 0.1 \%&
47.1  $\pm$ 0.2 \% \\
    
    \bottomrule
\end{tabular}
\end{footnotesize}
}
\vspace*{-\baselineskip}
\end{table}

\noindent
\textbf{Non-stationary client unavailability with $\tP = 1$.}
We study a total of three client unavailability dynamics in~\prettyref{tab: exp main text}, including stationary, staircase, and sine probabilistic trajectories.
Their visualizations are also available in the same table.
Our choices of non-stationary dynamics are motivated by real-world federated learning participation statistics~\citep{bonawitz2019towards,ribero2022federated}.
The learning tasks become more challenging as the list progresses due to the growing complexity of the non-stationary dynamics.

Mathematically, 
similar to the construction---\eqref{eq: m1 dynamic} and \eqref{eq: m2 dynamic}---in~\prettyref{sec: numerical P>1} but without alternate client group availability,
client $i$'s dynamics is defined as $p_i^t = p_i \cdot f_i(t)$.
The definitions of $p_i^t$ and $f_i(t)$ can be found therein.
Note that the $p_i$'s remain to be {\em non}-independent across different clients,
but we observe that~\FedAPM~retains its outperformance.

\noindent
{\bf Performance discussions.}
In addition to the baselines in~\prettyref{sec: numerical P>1}, we include the evaluation results on 
{\tt FedAvg} over all clients,
{\tt F3AST} algorithm~\citep{ribero2022federated} and on CINIC-10 data set~\citep{darlow2018cinic}.
To understand the nuances in the performance of our conference version~\citep{xiang2024efficient} and this journal extension, we also compare the performance of~\FedAPM~without ($k=0$) and with interpolation ($k=100$).

It is observed that~\FedAPM~consistently outperforms the algorithms not aided by heavy memory or known statistics.
In particular,~\FedAPM~with interpolation attains better accuracies than its non-interpolation variant~\citep{xiang2024efficient} on almost all tasks, providing added benefits.
In the only exception (stationary on SVHN data set), their performances are close, with accuracy differences of less than 1\%.
We also surprisingly observe that~\FedAPM~occasionally beats~\MIFA, which is memory heavy.
We attribute it to its reuse of stored gradients from the unavailable clients. 
Although~\FedAPM~brings in staleness due to implicit gossiping,
our results in~\prettyref{sec: numerical P>1} for $k \ge 0$ and \prettyref{tab: slowdown supp} in~\prettyref{app: numerical} for $k=0$ indicate that there is no significant slowdown for~\FedAPM~when compared to the baseline algorithms.
Furthermore,~\FedAPM~attains competitive or even better performance than~\FedAvg~with known probability, yet completely unknown to the underlying dynamics in client unavailability.

\section{Conclusion}
In this paper, we have shown that the significant impacts of heterogeneous and non-stationary client unavailability on learning performance through~\FedAvg.
To address this, we have proposed an algorithm~\FedAPM,
which provably converges to a stationary point of the global objective  
by adaptively echoing clients' local improvements,
by interpolating updates across rounds via a global moving average,
and by evenly diffusing local updates through implicit gossiping.
Notably, it achieves the desired linear speedup property in certain special cases. 
Experiments have validated the superiority of~\FedAPM~over state-of-the-art algorithms under diversified non-stationary dynamics.
Future work will investigate how to relax the assumption of independence in client availability.

\newpage
\acks{We gratefully acknowledge the support from the National Science Foundation under grants 2106891, 2107062, and the National Science Foundation CAREER award under grant 2340482. The views and conclusions contained in this document are those of the authors and should not be interpreted as representing the official policies, either expressed or implied, of the National Science Foundation or the U.S. Government. The U.S. Government is authorized to reproduce and distribute reprints for Government purposes notwithstanding any copyright notation herein.}
\bibliography{one}

\newpage
\appendix

\section*{Appendices}
\label{toc}

Here,
we provide an overview of the Appendices.
In particular,
the proofs of the main results
are presented and backed by supporting lemmas and propositions.

\startcontents[sections]
\printcontents[sections]{l}{1}{\setcounter{tocdepth}{2}}

\newpage
\newpage
\section{Nomenclature}
\label{app: nomenclature}
In this section,
we provide the notations and nomenclatures used throughout
our proofs for a comprehensive presentation.

\begin{table}[!ht]
    \centering
    \caption{Nomenclature table}
    \label{tab: alg nomenclature table}
    \begin{tabular}{p{2cm} p{11cm}}
    \toprule
    Notation(s) &
    Definition
    \\
    \midrule
       $\calA^t$  
       &
       The set of active clients in round $t$.\\
       \midrule
       $W^{t}$  
       &
       A doubly stochastic matrix to capture the information mixing error.
       Its definition can be found in~\eqref{eq: W matrix elements}.
       \\
       \midrule
       $p_i^t$  
       &
       The probability that a client $i$ becomes available in round $t$. 
       \\
       \midrule
       $\tau_i(t)$  
       &
       $\tau_i(t) \triangleq \sup \{t^\prime \mid t^\prime < t, i \in \calA^{t^\prime}\}$ 
       defines client $i$'s most recent active round. In particular, $\tau_i(0)=-1$ for all $i \in [m]$. \\
       \midrule
       $\x_i^t$ 
       &
       The real model at client $i$ at the {\bf beginning} of round $t$ in~\prettyref{alg: fedpbc+}. \\    
       \midrule
       $\bz_i^t$ 
       &
       The auxiliary model at client $i$ at the {\bf beginning} of round $t$. 
       Refer to~\prettyref{def: auxiliary sequence} for more details. 
       The sequence is for analysis only and is not computed by any clients.
       \\
       \midrule
       $\x^t$ 
       &
       The aggregated real model at the {\bf end} of round $t-1$ in~\prettyref{alg: fedpbc+}. \\    
       \midrule
       $\bz^t$ 
       &
       The auxiliary model at the {\bf end} of round $t-1$. 
       \\
       \midrule
       $\x_i^{t \dagger}$,
       $\bz_i^{t \dagger}$ 
       &
       The real model of an active client $i$,
       and auxiliary model of an active client $i$ 
       after $s$-step local computation in round $t$, respectively. 
       Refer to~\prettyref{alg: fedpbc+} for more details. \\
       \midrule
       $\x_i^{(t,r)}$ 
       &
       The real model at client $i$ after $r$-step local computation. \\
       \midrule
       $\bar{\x}^t$, 
       $\bar{\bz}^t$ 
       &
       The real and auxiliary model mean over all clients in a distributed system and in round $t$, 
       respectively. \\
       \midrule
       $F_i(\x)$ 
       &
       The local objective function at client $i$, 
       which is assumed to be non-convex. \\
       \midrule
       $F(\x)$ 
       &
        The global objective function defined in~\eqref{eq: global obj}:
       $F(\x) \triangleq \sum_{i=1}^m F_i(\x)/m$. \\
       \midrule
       $\nabla \ell_i(\x)$ 
       &
       The local stochastic gradient function at client $i$ taken with respect to $\x$. \\
       \midrule
       $\nabla F_i(\x)$ 
       &
       The local true gradient function at client $i$ taken with respect to $\x$. \\
       \midrule
       $\calD_i$ 
       &
       Client $i$'s local data distribution. \\
       \midrule
       $\xi_i$ 
       &
       An {\bf independent} stochastic sample drawn from client $i$'s local distribution $\calD_i$. \\     
    \bottomrule
    \end{tabular}
\end{table}

\begin{table}[!ht]
    \centering
    \caption{Variable table 
    }
    \label{tab: constant table}
    \begin{tabular}{c  p{12cm}}
    \toprule
       $\delta$ &
       An absolute constant that is the lower bound on the client unavailability. \\
       \midrule 
       $P_{\delta}$, $P$ &
       The period in~\prettyref{ass: prob lower bound}.
       \\
       \midrule
       $L$ &
       Lipschitz constant in~\prettyref{ass: 2 smmothness}. \\
       \midrule
       $\sigma^2$ &
       The upper bound of the stochastic gradient variance. \\
       \midrule
       $(\beta,~\zeta)$ &
       Parameters that capture the averaged gradient dissimilarity 
       between global and local objectives. \\
       \midrule
       $\rho$ &
       The spectral norm of a stochastic matrix in expectation.  \\
       \midrule
       $s$ &
       The number of local computation steps.  \\
       \midrule
       $k$ &
       The interpolation coefficient in the global moving average procedure.  \\
       \midrule
       $m$ &
       The number of clients in the federated learning system.  \\
       \midrule
       $M$ &
       $M = m+k$.
       \\
    \bottomrule
    \end{tabular}
\end{table}

\section{Useful Inequalities}
\label{app: preliminaries}
For completeness and for ease of exposition, we present some common inequalities that will be frequently used in our proofs.

\vskip \baselineskip 

The followings hold for any $\bm{a}_i \in \reals^d$ and any $i\in[m]$.
\begin{enumerate}
\item {Jensen's inequality.}
\begin{align}
\label{eq: jensen's inequality}
\norm{\frac{1}{m} \sum_{i=1}^m \bm{a}_i}^2
\le
\frac{1}{m} \sum_{i=1}^m \norm{\bm{a}_i}^2
~~~\text{and}~~~
\norm{\sum_{i=1}^m \bm{a}_i}^2
\le
m \sum_{i=1}^m \norm{\bm{a}_i}^2.
\end{align}
\item {Young's inequality (a.k.a. Peter-Paul inequality).}
\begin{align}
\label{eq: Young's inequality}
\iprod{\bm{a}_1}{\bm{a}_2}
\le
\frac{\norm{\bm{a}_1}^2}{2 \epsilon} + \frac{\epsilon \norm{\bm{a}_2}^2}{2}, ~~~ \text{for any }\epsilon>0. 
\end{align}
Equivalently,
we have
\begin{align}
\nonumber
\norm{\bm{a}_1 + \bm{a}_2}^2
&=
\norm{\bm{a}_1}^2
+
\norm{\bm{a}_2}^2
+
2 \iprod{\bm{a}_1}{\bm{a}_2} \\
\label{eq: Young's inequality 2}
&\le
\pth{1 + \frac{1}{\epsilon}}
\norm{\bm{a}_1}^2
+
\pth{1 + \epsilon}
\norm{\bm{a}_2}^2
, ~~~ \text{for any }\epsilon>0. 
\end{align}

\item {Smoothness corollary.} 
{\it Given Assumption \ref{ass: 2 smmothness}, it holds that}
\begin{align}
\label{eq: smoothness corollary}
\nonumber
F(\bm{a}_1) - F(\bm{a}_2) 
&=
\iprod{\bm{a}_1 - \bm{a}_2}{\int_0^1 \nabla F(\bm{a}_2 + \tau (\bm{a}_1 - \bm{a}_2)) \diff \tau}
\\
\nonumber
&=
\iprod{\nabla F(\bm{a}_2)}{\bm{a}_1 - \bm{a}_2} +
\int_0^1 \iprod{\bm{a}_1 - \bm{a}_2}{\nabla F(\bm{a}_2 + \tau (\bm{a}_1 - \bm{a}_2)) - \nabla F(\bm{a}_2) }\diff \tau \\
\nonumber
&\overset{(a)}{\le}
\iprod{\nabla F(\bm{a}_2)}{\bm{a}_1 - \bm{a}_2} +
L \int_0^1 \tau \norm{\bm{a}_1 - \bm{a}_2}\norm{(\bm{a}_1 - \bm{a}_2) }\diff \tau
\\
&\le
\iprod{\nabla F(\bm{a}_2)}{\bm{a}_1 - \bm{a}_2}
+
\frac{L}{2} \norm{\bm{a}_1 - \bm{a}_2}^2
,
\end{align}
where $(a)$ follows from Cauchy-Schwartz inequality and Assumption \ref{ass: 2 smmothness}.
\end{enumerate}

\section{Descent Lemma~(\prettyref{lmm: descent lemma})}
\label{app: descent lemma aux sequence}
In this section,
we first present a bound on multi-step local computation.
Then, we apply the bound to the analysis of descent lemma.

\subsection{Multi-step perturbation}
\begin{lemma}
    \label{lmm: multi-local steps}
    Suppose that~\prettyref{ass: 2 smmothness} and~\prettyref{ass: bounded variance client-wise} hold. For each regular client $i\in \calR$, we have
    \begin{align*}
        \expect{\norm{\sum_{r=0}^{s-1} \nabla F_i(\x_i^{(t,r)}) - \nabla F_i(\x_i^t)}^2 ~\Big|~\calF^t}
        &\le
        5 \eta_l^2 s^3 L^2 \sigma^2
        +
        20 \eta_l^2 s^4 L^2 
        \norm{\nabla F_i(\x_i^t)}^2
    \end{align*}
\end{lemma}

\begin{proof}{\bf Proof of~\prettyref{lmm: multi-local steps}}
    The proof shares a similar road map to \cite[Lemma 2]{yang2021achieving},
    but the objective is instead to show an upper bound with respect to  $\norm{\nabla F_i (\x_i^t)}^2$.

    It holds that
    \begin{align}
        \nonumber
        \expect{\norm{\sum_{r=0}^{s-1} \nabla F_i(\x_i^{(t,r)}) - \nabla F_i(\x_i^t)}^2}
        &\overset{(a)}{\le}
        s \sum_{r=0}^{s-1}
        \expect{\norm{\nabla F_i(\x_i^{(t,r)}) - \nabla F_i(\x_i^t)}^2 \Big| ~\calF^t} \\
        \label{eq: local after smoothness}
        &\overset{(b)}{\le}
        s L^2 \sum_{r=0}^{s-1}
        \expect{\norm{\x_i^{(t,r)} - \x_i^t}^2 \Big | ~\calF^t},
    \end{align}
    where inequality $(a)$ holds because of Jensen's inequality,
    inequality $(b)$ holds because of~\prettyref{ass: 2 smmothness}.
    It remains to bound $\mathbb{E}[\|\x_i^{(t,r)} - \x_i^t\|^2 \mid ~\calF^t]$.
    In what follows, we use $\nabla \ell_i^{(t,k)}$ to denote $\nabla \ell_i (\x_i^{(t,k)})$ and $\nabla F_i^{(t,k)}$ as $\nabla F_i (\x_i^{(t,k)})$, respectively, for ease of presentation.
    \begin{align*}
        &\expect{\norm{\x_i^{(t,r)} - \x_i^t}^2\Big| ~\calF^t}
        =
        \expect{\norm{\x_i^{(t,r-1)} - \x_i^t - \eta_l \nabla \ell_i^{(t,r-1)}}^2\Big| ~\calF^t} \\
        &=
        \expect{\norm{
        -\eta_l \pth{
        \nabla \ell_i^{(t,r-1)} - \nabla F_i^{(t,r-1)}}
        +\x_i^{(t,r-1)} - \x_i^t 
        -\eta_l \pth{
        \nabla F_i^{(t,r-1)} - \nabla F_i^{t}
        + \nabla F_i^{t}}
        }^2
        \Big| ~\calF^t}  \\
        &\overset{(c)}{=}
        \eta_l^2
        \expect{\norm{
        \nabla \ell_i^{(t,r-1)} - \nabla F_i^{(t,r-1)}}^2\Big| ~\calF^t}
        +
        \expect{
        \norm{
        \x_i^{(t,r-1)} - \x_i^t 
        -\eta_l \pth{
        \nabla F_i^{(t,r-1)} - \nabla F_i^{t}
        + \nabla F_i^{t}}
        }^2\Big| ~\calF^t} \\
        &\overset{(d)}{\le}
        \eta_l^2
        \expect{\norm{
        \nabla \ell_i^{(t,r-1)} - \nabla F_i^{(t,r-1)}}^2\Big| ~\calF^t} \\
        &~~~+
        \pth{1 + \frac{1}{2 s - 1}}
        \expect{
        \norm{
        \x_i^{(t,r-1)} - \x_i^t}^2\Big| ~\calF^t}
        +
        2 s \eta_l^2
        \expect{
        \norm{
        \nabla F_i^{(t,r-1)} - \nabla F_i^{t}
        + \nabla F_i^{t}}^2
        \Big| ~\calF^t} \\
        &\le
        \eta_l^2
        \expect{\norm{
        \nabla \ell_i^{(t,r-1)} - \nabla F_i^{(t,r-1)}}^2\Big| ~\calF^t} \\
        &~~~+
        \pth{1 + \frac{1}{2 s - 1}}
        \expect{
        \norm{
        \x_i^{(t,r-1)} - \x_i^t}^2\Big| ~\calF^t}
        +
        4 s \eta_l^2
        \expect{
        \norm{
        \nabla F_i^{(t,r-1)} - \nabla F_i^{t}
        }^2\Big| ~\calF^t}
        + 
        4 s \eta_l^2
        \norm{\nabla F_i^{t}}^2 \\
        &\overset{(e)}{\le}
        \eta_l^2 \sigma^2 
        + 
        4 s \eta_l^2
        \norm{\nabla F_i^{t}}^2 \\
        &~~~+
        \pth{1 + \frac{1}{2 s - 1}}
        \expect{
        \norm{
        \x_i^{(t,r-1)} - \x_i^t}^2\Big| ~\calF^t}
        +
        4 s L^2 \eta_l^2
        \expect{
        \norm{
        \x_i^{(t,r-1)} - \x_i^t
        }^2\Big| ~\calF^t}\\
        &=
        \eta_l^2 \sigma^2 
        + 
        4 s \eta_l^2
        \norm{\nabla F_i^{t}}^2 
        +
        \pth{1 + \frac{1}{2 s - 1} + 4 s L^2 \eta_l^2}
        \expect{
        \norm{
        \x_i^{(t,r-1)} - \x_i^t}^2\Big| ~\calF^t}
        ,
    \end{align*}
    where equality $(c)$ holds because $\nabla \ell_i^{(t,k)}$ is an unbiased estimator of $\nabla F_i^{(t,r)}$,
    inequality $(d)$ holds because of Young's inequality,
    inequality $(e)$ holds because of~\prettyref{ass: 2 smmothness}.

    By $\eta_l \le \frac{1}{4 s L}$, it holds that
    \[
        \frac{1}{2 s - 1} + 4 s L^2 \eta_l^2 
        \le
        \frac{1}{2 s - 1} + \frac{1}{4 s}
        \le
        \frac{2}{2 s - 1}.
    \]
    Unroll the recursion, 
    we have
    \begin{align*}
        \expect{\norm{\x_i^{(t,r)} - \x_i^t}^2\Big| ~\calF^t} &\le
        \sum_{k=0}^{r-1}
        \pth{1 + \frac{2}{2 s - 1}}^k
        \pth{
        \eta_l^2 \sigma^2
        +
        4 s \eta_l^2 \norm{\nabla F_i^t}^2
        } \\
        &\le
        \sum_{k=0}^{s-1}
        \pth{1 + \frac{2}{2 s - 1}}^k
        \pth{
        \eta_l^2 \sigma^2
        +
        4 s \eta_l^2 \norm{\nabla F_i^t}^2
        } \\
        &=
        \frac{2 s - 1}{2}
        \qth{ 
        \pth{1 + \frac{2}{2 s - 1}}^{s - \frac{1}{2} }
        \pth{1 + \frac{2}{2 s - 1}}^{\frac{1}{2} }
        - 1}
        \pth{
        \eta_l^2 \sigma^2
        +
        4 s \eta_l^2 \norm{\nabla F_i^t}^2
        } \\
        &\overset{(f)}{\le}
        \pth{s - \frac{1}{2}}
        \qth{ 
        \sqrt{3} e
        - 1}
        \pth{
        \eta_l^2 \sigma^2
        +
        4 s \eta_l^2 \norm{\nabla F_i^t}^2
        } \\
        &\overset{(g)}{\le}
        5 s \eta_l^2 \sigma^2
        +
        20 s^2 \eta_l^2 \norm{\nabla F_i^t}^2,
    \end{align*}
    where inequality $(f)$ holds because of $(1 + 1/x)^x < \exp(1)$,
    inequality $(g)$ holds because of $\sqrt{3} \exp(1) - 1 < 5$.
    Plug it back into~\eqref{eq: local after smoothness},
    we have the desired result
    \begin{align*}
        \expect{\norm{\sum_{r=0}^{s-1} \nabla F_i(\x_i^{(t,r)}) - \nabla F_i(\x_i^t)}^2\Big| ~\calF^t}
        &\le
        5 \eta_l^2 s^3 L^2 \sigma^2
        +
        20 \eta_l^2 s^4 L^2 \norm{\nabla F_i(\x_i^t)}^2
        .
    \end{align*}
\end{proof}

\subsection{Descent lemma}
Recall that we have defined an auxiliary global objective $\tF(\x) = \frac{1}{M} \sum_{i \in \calR \cup \calV} F_i(\x)$.
\begin{proof}{\bf Proof of~\prettyref{lmm: descent lemma}}
By Assumption \ref{ass: 2 smmothness} and inequality \eqref{eq: smoothness corollary}, we have 
\begin{align*}
\tF(\bar{\bz}^{t+1}) - \tF(\bar{\bz}^t)
&\le
\underbrace{\iprod{\nabla \tF(\bar{\bz}^{t})}{\bar{\bz}^{t+1} - \bar{\bz}^t}}_{(\rmA)}
+
\underbrace{\frac{L}{2} \norm{\bar{\bz}^{t+1} - \bar{\bz}^t}^2}_{(\rmB)}. 
\end{align*}
Recall that $M \triangleq m + k = \abth{\calR \cup \calV}$.
The one-round innovation of $\bar{\bm{z}}$ can be rewritten as  
\begin{align}
\nonumber
& \bar{\bz}^{t+1} - \bar{\bz}^t
= 
\frac{1}{M}\sum_{i\in \calA^t} \pth{\bz_i^{t \dagger} - \bz_i^{t}} + 
\frac{1}{M}\sum_{i \in \calR \setminus \calA^t} \pth{\bz_i^{t+1} - \bz_i^{t}} +
\frac{1}{M}\sum_{i \in \calV} \pth{\bz_i^{t+1} - \bz_i^{t}} 
\\
\notag
&= 
\frac{1}{M}\sum_{i\in \calA^t} \pth{\bz_i^{t \dagger} - \bz_i^{t}} + 
\frac{1}{M}\sum_{i \in \calR \setminus \calA^t} \pth{\bz_i^{t+1} - \bz_i^{t}} 
\\
\nonumber
&=
\frac{1}{M}\sum_{i=1}^m \indc{i \in \calA^t}
\pth{\eta_l \eta_g s \sum_{k=\tau_i(t)+1}^{t-1} \nabla F_i(\x_i^k) - \eta_l \eta_g (t - \tau_i(t)) \sum_{r=0}^{s-1} \nabla \ell_i(\x_i^{(t,r)};\xi_i^{(t,r)}) }\\
&~~~
\nonumber
-\frac{\eta_l \eta_g s}{M} \sum_{i=1}^m \indc{i \in \calR \setminus \calA^t}
\nabla F_i(\x_i^t)\\
\nonumber
&\overset{(a)}{=}  
\frac{1}{M}\sum_{i=1}^m \indc{i \in \calA^t}\eta_l \eta_g s (t-1 - \tau_i(t)) \nabla F_i(\x_i^t) 
- \frac{1}{M}\sum_{i=1}^m \indc{i \in \calA^t} \eta_l \eta_g (t - \tau_i(t)) \sum_{r=0}^{s-1} \nabla \ell_i(\x_i^{(t,r)};\xi_i^{(t,r)}) \\
\nonumber
&~~~-\frac{\eta_l \eta_g s}{M} \sum_{i=1}^m \indc{i \in \calR \setminus \calA^t}
\nabla F_i(\x_i^t) \\
\nonumber
&\overset{(b)}{=}
\frac{\eta_l \eta_g}{M} \sum_{i=1}^m \indc{i \in \calA^t}
(t - \tau_i(t))
\sum_{r=0}^{s-1} \pth{\nabla F_i(\x_i^{(t,r)}) - \nabla \ell_i(\x_i^{(t,r)};\xi_i^{(t,r)})} \\
\nonumber
&~~~+
\frac{\eta_l \eta_g}{M} \sum_{i=1}^m \indc{i \in \calA^t}
(t - \tau_i(t)) 
\sum_{r=0}^{s-1}
\pth{ \nabla F_i(\x_i^t) - 
\nabla F_i(\x_i^{(t,r)})} \\
\nonumber
&~~~-\frac{\eta_l \eta_g s}{M} \sum_{i=1}^m 
\nabla F_i(\x_i^t), 
\end{align}
where equality $(a)$ using the fact that $\x_i^k = \x_i^t$ for all $k$ such that $\tau_i(t)+1 \le k \le t$, and equality $(b)$ is obtained by adding and subtracting $\nabla \ell_i(\x_i^{t};\xi_i^{(t,r)})$ 
and by the fact that $\indc{i \in \calA^t} + \indc{i \in \calR \setminus \calA^t} = 1$ since a client $i \in \calV$ does not update gradients. 

\paragraph{Bounding $(\rmA)$.}
\begin{align*}
(\rmA)
&=\iprod{\nabla \tF(\bar{\bz}^{t})}{\bar{\bz}^{t+1} - \bar{\bz}^t} \\
& = \underbrace{
\eta_l \eta_g\iprod{\nabla \tF(\bar{\bz}^t)}{
\frac{1}{M} \sum_{i=1}^m \indc{i \in \calA^t} \sum_{p=-1}^{t-1} \indc{\tau_i(t) = p}
(t - p)
\sum_{r=0}^{s-1} \pth{\nabla F_i(\x_i^{(t,r)}) - \nabla \ell_i(\x_i^{(t,r)};\xi_i^{(t,r)})}}}_{(\rmA.\rmI)} \\
&~~~+
\underbrace{\frac{\eta_l \eta_g}{M} \sum_{i=1}^m \indc{i \in \calA^t}
\sum_{p=-1}^{t-1} \indc{\tau_i(t) = p}
\iprod{\nabla \tF (\bar{\bz}^t)}{(t - p) 
\sum_{r=0}^{s-1}
\pth{ \nabla F_i(\x_i^t) - 
\nabla F_i(\x_i^{(t,r)})} }}_{(\rmA.\rmI\rmI)} \\
&~~~+
\underbrace{
\frac{\eta_l \eta_g s}{M} \sum_{i=1}^m \iprod{\nabla \tF (\bar{\bz}^t)}{
\nabla F_i(\bz_i^t) - \nabla F_i(\x_i^t)}
}_{(\rmA.\rmI\rmI\rmI)} 
-
\underbrace{\eta_l \eta_g s \iprod{\nabla \tF(\bar{\bz}^t)}{\frac{1}{M}\sum_{i=1}^m \nabla F_i(\bz_i^t)}}_{(\rmA.\rmI\rmV)}.
\end{align*}
\paragraph{Bounding $(\rmA.\rmI)$}
\begin{small}
\begin{align*}
    &\expect{(\rmA.\rmI)\Big|\calF^t} \\
    &\overset{(a)}{=}
    \eta_l \eta_g 
    \expect{
    \expect{
    \iprod{\nabla \tF(\bar{\bz}^t)}{\frac{1}{M}\sum_{i=1}^m \indc{i\in\calA^t} \sum_{p=-1}^{t-1} \indc{\tau_i(t) = p}(t-p) 
    \sum_{r=0}^{s-1} \pth{\nabla F_i(\x_i^{(t,r)}) - \nabla \ell_i(\x_i^{(t,r)};\xi_i^{(t,r)})}}
    \Big| \x_i^{(t,r)}, \calF^t}
    \Big|\calF^t} \\
    &\overset{(b)}{=}
    \eta_l \eta_g 
    \left \langle
    {\nabla \tF(\bar{\bz}^t)}, \right.\\
    &\qquad \qquad \left.
    {\frac{1}{M}\sum_{i=1}^m 
    \expect{\indc{i\in\calA^t}\Big|\calF^t}
    \sum_{p=-1}^{t-1} \indc{\tau_i(t) = p}(t-p) 
    \sum_{r=0}^{s-1} 
    \expect{
    \expect{
    \pth{\nabla F_i(\x_i^{(t,r)}) - \nabla \ell_i(\x_i^{(t,r)};\xi_i^{(t,r)})}
    \Big|  \x_i^{(t,r)}, \calF^t}
    \Big|\calF^t}}
    \right \rangle\\
    &=0,
\end{align*}
\end{small}
where equality $(a)$ holds because of the law of total expectation,
equality $(b)$ holds because $\indc{i \in \calA^t}$ is by definition independent of others and~\prettyref{ass: bounded variance client-wise}.

\paragraph{Bounding $(\rmA.\rmI\rmI)$}
\begin{align*}
(\rmA.\rmI\rmI)
&\overset{(c)}{\le}
\frac{\eta_l \eta_g}{M} \sum_{i=1}^m \indc{i \in \calA^t} \sum_{p=-1}^{t-1} \indc{\tau_i(t) = p}
\pth{\frac{s}{12}\norm{\nabla \tF(\bar{\bz}^t)}^2
+
\frac{3 (t - p)^2}{s} \norm{\sum_{r=0}^{s-1} \nabla F_i(\x_i^{t}) - \nabla F_i(\x_i^{(t,r)})}^2} \\
&=
\frac{\eta_l \eta_g s}{12 M} \sum_{i=1}^m \indc{i \in \calA^t} 
\norm{\nabla \tF(\bar{\bz}^t)}^2 
\\&~~~
+
\frac{\eta_l \eta_g}{M} \sum_{i=1}^m \indc{i \in \calA^t} \sum_{p=-1}^{t-1} \indc{\tau_i(t) = p}
\frac{3 (t - p)^2}{s} \norm{\sum_{r=0}^{s-1} \nabla F_i(\x_i^{t}) - \nabla F_i(\x_i^{(t,r)})}^2 
,
\end{align*}
where inequality $(c)$ holds because of Young's inequality.
We further have:
\begin{align*}
\expect{(\rmA.\rmI\rmI)\Big|\calF^t}
&\overset{(d)}{\le}
\frac{\eta_l \eta_g s}{12} 
\norm{\nabla \tF(\bar{\bz}^t)}^2 
+
\frac{15 \eta_g \eta_l^3 s^2 L^2 \sigma^2}{M} \sum_{i=1}^m 
\sum_{p=-1}^{t-1} \indc{\tau_i(t) = p}
(t - p)^2
\\
&~~~
+
\frac{60 \eta_g \eta_l^3 s^3 L^2 }{M} \sum_{i=1}^m 
\sum_{p=-1}^{t-1} \indc{\tau_i(t) = p}
(t - p)^2
\norm{\nabla F_i (\x_i^t)}^2 \\
&=
\frac{\eta_l \eta_g s}{12} 
\norm{\nabla \tF(\bar{\bz}^t)}^2
+
\frac{15 \eta_g \eta_l^3 s^2 L^2 \sigma^2}{M} \sum_{i=1}^m 
\sum_{p=-1}^{t-1} \indc{\tau_i(t) = p}
(t - p)^2 \\
&~~~+
\frac{60 \eta_g \eta_l^3 s^3 L^2 }{M} \sum_{i=1}^m \sum_{p=-1}^{t-1} \indc{\tau_i(t) = p}
(t - p)^2 
\norm{\nabla F_i(\x_i^{p+1})}^2,
\end{align*}
where inequality $(d)$ holds because of~\prettyref{lmm: multi-local steps},
the last equality using the fact that $\x_i^k = \x_i^t$ for all $k$ such that $ \tau_i(t) + 1 \le k \le t$.

\paragraph{Bounding $(\rmA.\rmI\rmI\rmI)$.}
\begin{align*}
(\rmA.\rmI\rmI\rmI)
&=
\frac{\eta_l \eta_g s }{M} 
\sum_{i=1}^m 
\iprod{\nabla \tF(\bar{\bz}^t)}{\nabla F_i(\bz_i^t) - \nabla F_i(\x_i^t)} 
\overset{(e)}{\le}
\frac{\eta_l \eta_g s}{12} \norm{\nabla \tF(\bar{\bz}^t)}^2
+
\frac{3 \eta_l \eta_g s L^2}{M} \sum_{i=1}^m \norm{\bz_i^t - \x_i^t}^2,
\end{align*}
where inequality $(e)$ follows from Young's inequality
and Assumption \ref{ass: 2 smmothness}.
It holds that,
\begin{align*}
\expect{ (\rmA.\rmI\rmI\rmI)\Big|\calF^t}
&\le
\frac{\eta_l \eta_g s}{12} \norm{\nabla \tF(\bar{\bz}^t)}^2
+
\frac{3 \eta_l \eta_g s L^2}{M} \sum_{i=1}^m \norm{\bz_i^t - \x_i^t}^2.
\end{align*}
\paragraph{Bounding $(\rmA.\rmI\rmV)$}
\begin{align*}
(\rmA.\rmI\rmV) 
&=
\frac{\eta_l \eta_g s}{2} \pth{
\norm{\nabla \tF (\bar{\bz}^t)}^2
+
\norm{\frac{1}{M}\sum_{i=1}^m \nabla F_i(\bz_i^t)}^2
-
\norm{\nabla \tF (\bar{\bz}^t) - \frac{1}{M}\sum_{i=1}^m \nabla F_i(\bz_i^t)}^2
}, 
\end{align*}
where the equality follows from the identity in Appendix \ref{app: preliminaries} (3).
It holds that
\begin{align*}
\expect{(\rmA.\rmI\rmV)\Big|\calF^t}
&=
\frac{\eta_l \eta_g s}{2} \pth{
\norm{\nabla \tF (\bar{\bz}^t)}^2
+
\norm{\frac{1}{M}\sum_{i=1}^m \nabla F_i(\bz_i^t)}^2
-
\norm{\frac{1}{M} \sum_{i=1}^m \nabla F_i (\bar{\bz}^t) - \frac{1}{M}\sum_{i=1}^m \nabla F_i(\bz_i^t)}^2
} \\
&\ge
\frac{\eta_l \eta_g s}{2} \pth{
\norm{\nabla \tF (\bar{\bz}^t)}^2
+
\norm{\frac{1}{M}\sum_{i=1}^m \nabla F_i(\bz_i^t)}^2
-
\frac{L^2}{M}
\sum_{i=1}^m
\norm{\bar{\bz}^t - \bz_i^t}^2
},
\end{align*}
where the first equality holds because a client $i \in \calV$ does not update gradients, and it holds that
\[
\nabla \tF(\bar{\bz}^t) 
\triangleq
\frac{1}{M}
\sum_{i \in \calR}
\nabla F_i(\bar{\bz}^t)
+
\frac{1}{M}
\sum_{i \in \calV}
\nabla F_i(\bar{\bz}^t)
=
\frac{1}{M}
\sum_{i=1}^m
\nabla F_i(\bar{\bz}^t),
\]
where we use the convention that $\nabla F_i(\bar{\bz}^t) = \bm{0}$ for $i \in \calV$.
Putting $(\rmA)$ together,
\begin{align*}
&\expect{
(\rmA)
\Big|\calF^t}
\le
- \frac{\eta_l \eta_g s}{3} \norm{\nabla \tF(\bar{\bz}^t)}^2 
+
\frac{15 \eta_g \eta_l^3 s^2 L^2 \sigma^2}{M} \sum_{i=1}^m 
\sum_{p=-1}^{t-1} \indc{\tau_i(t) = p}
(t - p)^2 \\
&~~~+
\frac{3 \eta_l \eta_g s L^2}{M}\sum_{i=1}^m \norm{\x_i^t - \bz_i^t}^2
+
\frac{\eta_l \eta_g s L^2}{2 M} 
\sum_{i=1}^m
\norm{\bar{\bz}^t - \bz_i^t}^2  \\
&~~~-
\frac{\eta_l \eta_g s}{2} \norm{\frac{1}{M} \sum_{i=1}^m \nabla F_i(\bz_i^t)}^2
+
\frac{60 \eta_g \eta_l^3 s^3 L^2 }{M} \sum_{i=1}^m \sum_{p=-1}^{t-1} \indc{\tau_i(t) = p}
(t - p)^2 
\norm{\nabla F_i(\x_i^{p+1})}^2.
\end{align*}
\paragraph{Bounding $(\rmB)$.}
\begin{align*}
(\rmB)
&\le
\underbrace{
2 L
\frac{\eta_l^2 \eta_g^2}{M^2} 
\norm{
\sum_{i=1}^m 
\indc{i \in \calA^t}
(t - \tau_i(t))
\sum_{r=0}^{s-1} \pth{\nabla F_i(\x_i^{(t,r)}) - \nabla \ell_i(\x_i^{(t,r)};\xi_i^{(t,r)})}}^2}_{(\rmB.\rmI)} \\
&~~~+
\underbrace{
2 L
\frac{\eta_l^2 \eta_g^2}{M^2} 
m \sum_{i=1}^m \indc{i \in \calA^t}
(t - \tau_i(t))^2
\norm{
\sum_{r=0}^{s-1} \pth{\nabla F_i(\x_i^t) - \nabla F_i(\x_i^{(t,r)})}
}^2}_{(\rmB.\rmI\rmI)} \\
&~~~+
\underbrace{2 L
\frac{\eta_l^2 \eta_g^2 s^2}{M^2} 
m \sum_{i=1}^m
\norm{
\nabla F_i(\x_i^t) - \nabla F_i(\bz_i^t)
}^2}_{(\rmB.\rmI\rmI\rmI)}
+
\underbrace{
2 L \eta_l^2 \eta_g^2 s^2
\norm{
\frac{1}{M}
\sum_{i=1}^m 
\nabla F_i(\bz_i^t)
}^2}_{(\rmB.\rmI\rmV)}
\end{align*}

\paragraph{Bounding $(\rmB.\rmI)$}
Recall that $\delta_{\max} \triangleq \sup_{i\in [m],t \in [T]} p_i^t.$
It holds that,
\begin{align*}
    &\expect{(\rmB.\rmI)\Big|\calF^t} \\
    &\overset{(f)}{=}
    2 L
    \frac{\eta_l^2 \eta_g^2}{M^2} 
    \sum_{i=1}^m 
    \expect{\indc{i \in \calA^t}\Big|\calF^t}
    (t - \tau_i(t))^2
    \sum_{r=0}^{s-1} 
    \expect{
    \expect{
    \norm{\nabla F_i(\x_i^{(t,r)}) - \nabla \ell_i(\x_i^{(t,r)};\xi_i^{(t,r)})}^2 \Big|\x_i^{(t,r)}, \calF^t
    }
    \Big|\calF^t} \\
    &\overset{(g)}{\le}
    \frac{2 \eta_l^2 \eta_g^2 s L \delta_{\max}\sigma^2}{M^2} 
    \sum_{i=1}^m
    \sum_{p=-1}^{t-1}
    \indc{\tau_i(t) = p}
    (t - p)^2,
\end{align*}
where equality $(f)$ holds by the law of total expectation and by the independence of event $\{ i \in \calA^t \}$,
inequality $(g)$ holds because of~\prettyref{ass: bounded variance client-wise} and by definition $p_i^t \le \delta_{\max}$.

\paragraph{Bounding $(\rmB.\rmI\rmI)$}
We have,
\begin{align*}
    \expect{(\rmB.\rmI\rmI)\Big|\calF^t} 
    &\le
    2 L
    \frac{\eta_l^2 \eta_g^2}{M} 
    \sum_{i=1}^m 
    \sum_{p=-1}^{t-1}
    \indc{\tau_i(t) = p}
    (t - p)^2
    5 \eta_l^2 s^3 L^2 
    \sigma^2 
    \\&~~~
    +
    2 L
    \frac{\eta_l^2 \eta_g^2}{M} 
    \sum_{i=1}^m 
    \indc{\tau_i(t) = p}
    \sum_{p=-1}^{t-1}
    (t - p)^2
    20 \eta_l^2 s^4 L^2 
    \norm{\nabla F_i(\x_i^{t})}^2 \\
    &=
    \frac{10 \eta_g^2 \eta_l^4 s^3 L^3 \sigma^2 }{M} 
    \sum_{i=1}^m 
    \sum_{p=-1}^{t-1}
    \indc{\tau_i(t) = p}
    (t - p)^2
    \\
    &~~~+
    \frac{
    40 \eta_g^2 \eta_l^4
    s^4 L^3
    }{M} 
    \sum_{i=1}^m 
    \sum_{p=-1}^{t-1}
    \indc{\tau_i(t) = p}
    (t - p)^2
    \norm{\nabla F_i(\x_i^{p+1})}^2,
\end{align*}
where the last equality using the fact that $\x_i^k = \x_i^t$ for all $k$ such that $ \tau_i(t) + 1 \le k \le t$.

\paragraph{Bounding $(\rmB.\rmI\rmI\rmI)$.}
\[
\expect{(\rmB.\rmI\rmI\rmI)\Big|\calF^t} 
\le \frac{2 \eta_l^2 \eta_g^2 s^2 L^3}{M} 
\sum_{i=1}^m \norm{\x_i^t - \bz_i^t}^2.
\]

Putting $(\rmB)$ together, we get
\begin{align*}
\expect{(\rmB)\Big|\calF^t}
&\le
\frac{2 \eta_l^2 \eta_g^2 s L \sigma^2}{M^2} 
\sum_{p=-1}^{t-1}
\indc{\tau_i(t) = p}
(t - p)^2
+
\frac{10 \eta_g^2 \eta_l^4 s^3 L^3 \sigma^2 }{M} 
\sum_{i=1}^m 
\sum_{p=-1}^{t-1}
\indc{\tau_i(t) = p}
(t - p)^2 \\
&~~~+
\frac{
40 \eta_g^2 \eta_l^4
s^4 L^3
}{M} 
\sum_{i=1}^m 
\sum_{p=-1}^{t-1}
\indc{\tau_i(t) = p}
(t - p)^2
\norm{\nabla F_i(\x_i^{p+1})}^2 \\
&~~~+
\frac{2 \eta_l^2 \eta_g^2 s^2 L^3}{M}
\sum_{i=1}^m
\norm{\x_i^t - \bz_i^t}^2
+
2 L \eta_l^2 \eta_g^2 s^2
\norm{
\frac{1}{M}
\sum_{i=1}^m 
\nabla F_i(\bz_i^t)
}^2
.
\end{align*}
Now, everything:
\begin{align*}
\expect{\tF(\bar{\bz}^{t+1}) - \tF(\bar{\bz}^t)\Big|\calF^t}
&\le
- \frac{\eta_l \eta_g s}{3} \norm{\nabla \tF(\bar{\bz}^t)}^2 \\
&~~~
- \frac{\eta_l \eta_g s}{2} \pth{1 - 4 L \eta_l \eta_g s}
\norm{\frac{1}{M} \sum_{i=1}^m \nabla F_i(\bz_i^t)}^2 \\
&~~~
+\frac{2 \eta_l^2 \eta_g^2 s L \delta_{\max} \sigma^2}{M^2} 
\sum_{i=1}^m
\sum_{p=-1}^{t-1}
\indc{\tau_i(t) = p}
(t - p)^2
\\
&~~~+
\frac{5 \eta_g \eta_l^3 s^2 L^2 
\pth{3 + 2 \eta_g \eta_l s L}
\sigma^2
}{M} \sum_{i=1}^m 
\sum_{p=-1}^{t-1} \indc{\tau_i(t) = p}
(t - p)^2  \\
&~~~
+ \eta_l \eta_g s L^2 \pth{3 + 2 \eta_l \eta_g s L}
\frac{1}{M}
\sum_{i=1}^m
\norm{\x_i^t - \bz_i^t}^2 
+
\frac{\eta_l \eta_g s L^2}{2 M}
\sum_{i=1}^m
\norm{\bz_i^t - \bar{\bz}^t}^2 \\
&~~~
+
20 \eta_g \eta_l^3 s^3 L^2
\pth{
3+
2 \eta_g \eta_l s L
}
\frac{1}{M} 
\sum_{i=1}^m 
\sum_{p=-1}^{t-1}
\indc{\tau_i(t) = p}
(t - p)^2
\norm{\nabla F_i(\x_i^{p+1})}^2 \\
&\le
- \frac{\eta_l \eta_g s}{3} \norm{\nabla \tF(\bar{\bz}^t)}^2 
+\frac{2 \eta_l^2 \eta_g^2 s L \delta_{\max} \sigma^2}{M^2} 
\sum_{i=1}^m
\sum_{p=-1}^{t-1}
\indc{\tau_i(t) = p}
(t - p)^2
\\
&~~~+
\frac{17
\eta_g \eta_l^3 s^2 L^2 
\sigma^2
}{M} \sum_{i=1}^m 
\sum_{p=-1}^{t-1} \indc{\tau_i(t) = p}
(t - p)^2 \\
&~~~
+ 
4 \eta_l \eta_g s L^2 
\frac{1}{M}
\sum_{i=1}^M
\norm{\x_i^t - \bz_i^t}^2
+
\frac{\eta_l \eta_g s L^2}{2M}
\sum_{i=1}^M
\norm{\bz_i^t - \bar{\bz}^t}^2 \\
&~~~
+
65 \eta_g \eta_l^3 s^3 L^2
\frac{1}{M} 
\sum_{i=1}^m 
\sum_{p=-1}^{t-1}
\indc{\tau_i(t) = p}
(t - p)^2
\norm{\nabla F_i(\x_i^{p+1})}^2 
,
\end{align*}
where the last inequality holds because 
$\eta_l \eta_g \le \frac{1}{8 s L}$ 
and that $\norm{\frac{1}{M} \sum_{i=1}^m \nabla F_i(\bz_i^t)}^2 \ge 0$.
\end{proof}

\section{Intermediate Results}
\label{app: intermediate without pseudo}
In this section,
we present the intermediate results that serve as handy tools in building up our proofs afterwards.

\subsection{Bounding local and global dissimilarity}

\begin{proposition}
\label{prop: average gradient to global gradient}
For any $t$, it holds that 
\begin{align*}
\frac{1}{M}\sum_{i=1}^m\norm{\nabla F_i(\bz_i^t)}^2 
\le \frac{3L^2}{M} \sum_{i=1}^M \norm{\bz_i^t - \bar{\bz}^t}^2 
+ \frac{3 M}{m}\pth{\beta^2 + 1}\norm{\nabla \tF(\bar{\bz}^t)}^2 
+ \frac{3 m \zeta^2}{M}. 
\end{align*}
\end{proposition}
\begin{proof}{\bf Proof of Proposition \ref{prop: average gradient to global gradient}}
\begin{align*}
\frac{1}{m}\sum_{i=1}^m\norm{\nabla F_i(\bz_i^t)}^2  
&=  
\frac{1}{m}
\sum_{i=1}^m
\norm{
\nabla F_i(\bz_i^t) - \nabla F_i(\bar{\bz}^t) 
+ \nabla F_i(\bar{\bz}^t) - \nabla F(\bar{\bz}^t) 
+ \nabla F(\bar{\bz}^t)
}^2\\
& \le  
\frac{3}{m}\sum_{i=1}^m \norm{\nabla F_i(\bz_i^t) - \nabla F_i(\bar{\bz}^t)}^2 
+\frac{3}{m}\sum_{i=1}^m \norm{\nabla F_i(\bar{\bz}^t) - \nabla F(\bar{\bz}^t)}^2
+
3 \norm{\nabla F(\bar{\bz}^t) }^2
\\
& \overset{(a)}{\le} 
\frac{3 L^2}{m} \sum_{i=1}^m \norm{\bz_i^t - \bar{\bz}^t}^2  
+ 3 \pth{\beta^2 + 1} \norm{\nabla F(\bar{\bz}^t)}^2 
+ 3 \zeta^2,
\end{align*} 
where inequality (a) follows from Assumptions \ref{ass: 2 smmothness} and \ref{ass: bounded similarity}.
It follows that
\begin{align*}
\frac{1}{M}
\sum_{i=1}^m
\norm{\nabla F_i(\bz_i^t)}^2  
&\le
\frac{3 L^2}{M} 
\sum_{i=1}^M \norm{\bz_i^t - \bar{\bz}^t}^2  
+ \frac{3 M}{m} \pth{\beta^2 + 1} \norm{\nabla \tF(\bar{\bz}^t)}^2 
+ \frac{3 m}{M} \zeta^2
,
\end{align*} 
where the equality holds because $\nabla F(\x) = \frac{M}{m} \nabla \tF(\x)$.
\end{proof}

\subsection{Weight re-equalization (\prettyref{prop: similar speed})}
\begin{proof}{\bf Proof of~\prettyref{prop: similar speed}}
We show~\prettyref{prop: similar speed} by induction.

When $T=1$ and $i \in \calA^0$, we have
$
\sum_{t=0}^{0} \indc{i \in \calA^t} \pth{t - \tau_i(t)} = \indc{i \in \calA^0} \pth{0 - \tau_i(0)} = 1.
$
Therefore, the base case holds.

The induction hypothesis is that $\sum_{t=0}^{K-1} \indc{i \in \calA^t} \pth{t - \tau_i(t)} = K$ holds for $i \in \calA^{K-1}$. Next, we focus on $K+1$:
\begin{align}
\sum_{t=0}^{K} \indc{i \in \calA^t} \pth{t - \tau_i(t)} &= 
\sum_{t=0}^{K-1} \indc{i \in \calA^t} \pth{t - \tau_i(t)} + \indc{i \in \calA^K} \pth{K - \tau_i(K)}.
\label{eq: induction without pseudo}
\end{align}
Now, we have two cases:
\begin{itemize}
\item Suppose $i \in \calA^{K-1}$, then we simply have $\tau_i(K) = K-1$. It follows that
\eqref{eq: induction without pseudo}~$\overset{(a)}{=} K + 1$,
where $(a)$ follows from induction hypothesis.
\item Suppose $i \notin \calA^{K-1}$, 
\begin{align*}
\sum_{t=0}^{K} \indc{i \in \calA^t} \pth{t - \tau_i(t)} &\overset{(b)}{=} 
\sum_{t=0}^{\tau_i(K)} \indc{i \in \calA^t} \pth{t - \tau_i(t)} + \indc{i \in \calA^K} \pth{K - \tau_i(K)} \\
& = \tau_i(K)+1 +  \pth{K - \tau_i(K)} = K + 1,
\end{align*}
where $(b)$ follows because $\indc{i \in  \calA^t} = 0$ for $ \tau_i(K) \le t \le K-1$ and induction hypothesis that $\sum_{t=0}^{\tau_i(K)} \indc{i \in \calA^t} \pth{t - \tau_i(t)} = \tau_i(K)+1$ for $i \in \calA^{\tau_i(K)}$.
\end{itemize}
\end{proof}

\subsection{Unavailable statistics (\prettyref{lmm: geo second moment main text})}
\label{app: unavail stats}
\begin{proof}{\bf Proof of~\prettyref{lmm: geo second moment main text}}
\begin{align}
\label{eq: first moment to be divided}
\expect{t - \tau_i(t)}
&=
\sum_{r = 0}^{t}
\prob{t - \tau_i(t) > r} \\\notag
&=
\sum_{r = 0}^{t}
\prod_{r_1 = t-r}^{t-1} \pth{1 - p_i^{r_1}} 
=
1+
\sum_{r = 1}^{t}
\prod_{r_1 = t-r}^{t-1} \pth{1 - p_i^{r_1}}
\le
1+
\sum_{r = 1}^{\infty}
\prod_{r_1 = t-r}^{t-1} \pth{1 - p_i^{r_1}} \\\notag
&=
1+
\lim_{\ell \rightarrow \infty}
\sum_{r = 1}^{\ell P}
\prod_{r_1 = t-r}^{t-1} \pth{1 - p_i^{r_1}}
\end{align}
Let $t_0 \in \{0, \tP, 2\tP, \ldots\}$, it holds that
\begin{align*}
\prod_{t=t_0}^{t_0+\tP-1}
(1-p_i^t)
&\overset{(a)}{\le}
\pth{
\frac{\tP - \sum_{t=t_0}^{t_0+\tP-1} p_i^t }{\tP}
}^{\tP}
\le
\pth{1 - \delta}^{\tP}
,
\end{align*}
where inequality $(a)$ holds because of AM-GM inequality.
Also, it trivially holds that
\begin{align*}
\prod_{t=l}^{l^\prime}
(1-p_i^t)
&\le
1,
\end{align*}
where
$t_0 \le l \le l^\prime \le t_0 + \tP - 1$.
In general, we have
\begin{itemize}
\item 
When $r \ge \tP$, it holds that
\begin{align*}
\prod_{r_1=t_0-r}^{t_0-1}
(1-p_i^{r_1})
&\le
\prod_{r_0 = \pth{t_0-r}/{\tP} }^{\lfloor t_0/ \tP \rfloor - 1}
\prod_{r_1=r_0 \tP}^{(r_0+1)\tP-1}
(1-p_i^{r_1})
\le
(1 - \delta)^{\tP \pth{t_0 /\tP  - \lceil \pth{t_0-r} /\tP \rceil}}
\le
(1 - \delta)^{r - \tP}.
\end{align*}
\item 
When $r < \tP$, it holds that
\begin{align*}
\prod_{r_1=t_0-r}^{t_0-1}
(1-p_i^{r_1})
&\le
1.
\end{align*}
\end{itemize}
Hence,
\begin{align}
\label{eq: prod delta}
\prod_{r_1=t_0-r}^{t_0-1}
(1-p_i^{r_1})
&\le
(1 - \delta)^{\pth{r - \tP} \indc{r \ge \tP}}.
\end{align}
It holds for~\eqref{eq: first moment to be divided} that
\begin{align*}
\expect{t - \tau_i(t)}
&\le
1+
\sum_{r = 1}^{\infty}
\pth{1 - \delta}^{(r - \tP) \indc{r \ge \tP}} 
= 
\tP  + \frac{1}{\delta}.
\end{align*}
From~\cite[Theorem 12.3 (1)]{gut2006probability},
we know that 
\[
    \expect{g (\bm{X})} =
    g(0) +
    \int_{0}^\infty g^\prime (x) \prob{X > x} \rmd x,
\]
where $X$ is a non-negative random variable, 
and $g$ a non-negative, strictly increasing, differentiable function.
Therefore,
\begin{align}
    \notag
    \expect{ \pth{t - \tau_i(t)}^2 } 
    &\overset{(\rma)}{\le} 2 \sum_{r=1}^{\infty} r \pth{ 1-\delta }^{\pth{r-\tP} \indc{r \ge \tP}} \\
    \notag
    &\le
    2 \pth{
    \frac{\tP(\tP - 1)}{2}
    +
    \frac{\tP}{\delta}
    +
    \frac{1 - \delta}{\delta^2}
    } \\
    \label{eq: exp var res}
    &=
    \frac{1}{\delta^2}
    \pth{(\tP-1) \delta + 1}^2
    +
    \frac{1}{\delta^2}
    \pth{(\tP - 1) \delta^2 + 1},
\end{align}
where inequality $(\rma)$ holds because of~\prettyref{eq: prod delta},
For a neat presentation, we use $(\variance)$ as a shorthand notation in the following proofs for the constant in~\eqref{eq: exp var res}.
\end{proof}

\subsection{Auxiliary sequence construction and properties (\prettyref{prop: client dis})}
\begin{proposition}
\label{prop: x z: inactive and active}     
For any $t\ge 0$, when $i\notin\calA^t$, it holds that 
$\x_i^{t+1} - \bz_i^{t+1} = 
\eta_l \eta_g s (t-\tau_i(t+1) )
\nabla F_i(\x_i^{\tau_{i}(t+1)+1})$;
when $i\in\calA^t$, it holds that 
$\bz_i^{t\dagger} = \x_i^{t\dagger},~ \bz^{t+1} = \x^{t+1},$ and $\bz_i^{t+1} = \x_i^{t+1}$. 
\end{proposition}
\begin{proof}{\bf Proof of Proposition \ref{prop: x z: inactive and active}}
The proof is divided into two parts: 
$i \notin \calA^t$ and $i \in \calA^t$,

\paragraph{When $ i \notin \calA^t$.}

It holds that
\begin{align*}
\x_i^{t+1} - \bz_i^{t+1} 
&=
\x_i^{\tau_i(t+1)+1} - \qth{ \bz_i^{\tau_i(t+1)+1} - \eta_l \eta_g s \sum_{k=\tau_i(t+1)+1}^{t} \nabla F_i(\x_i^{k})} \\
&\overset{(a)}{=}
\x_i^{\tau_i(t+1)+1} - \qth{ \x_i^{\tau_i(t+1)+1} - \eta_l \eta_g s \sum_{k=\tau_i(t+1)+1}^{t} \nabla F_i(\x_i^{\tau_i(t+1)+1})} \\ 
&=
\eta_l \eta_g s (t - \tau_i(t+1)) \nabla F_i(\x_i^{\tau_i({t+1)+1}}), 
\end{align*}
where equality (a) follows from~\prettyref{def: auxiliary sequence} for inactive clients.

\paragraph{When $i \in \calA^t$.}

Note that if $\bz_i^{t++} = \x_i^{t++}$ for each $i\in \calA^t$, then by the aggregation rules, we know $\x^{t+1} = \pth{1/{\abth{\calA^t}}} \sum_{i\in\calA^t} \x_i^{t++} = \pth{1/{\abth{\calA^t}}} \sum_{i\in\calA^t} \bz_i^{t++} = \bz^{t+1}.$ 
Then, we know that $\x_i^{t+1} = \bz_i^{t+1}, ~ \forall ~ i\in \calA^t.$
Hence, to show the Proposition, it is sufficient to show $\bz_i^{t++} = \x_i^{t++}$ holds for $i\in \calA^t$, which can be shown by induction. 
When $t=0$, 
\begin{align*}
\bz_i^{0++} = \bz_i^0 + 0 - \pth{\x_i^{(0,0)} - \x_i^{(0,s)}}   
= \x_i^0  - \pth{\x_i^{(0,0)} - \x_i^{(0,s)}} = \x_i^{0++}. 
\end{align*}
Thus, the base case holds. The induction hypothesis is that $\bz_i^{t++} = \x_i^{t++}, ~ \forall~ i\in \calA^t$ is true for all $t\ge 0$.  
Now, we focus on $t+1$. 
\begin{align*}
\bz_i^{(t+1)++} & = \bz_i^{t+1} + \eta_l \eta_g s \sum_{k=\tau_i(t+1) +1}^{t} \nabla F_i(\x_i^k) -(t+1- \tau_i(t+1)) \pth{\x_i^{(t+1,0)} - \x_i^{(t+1,s)}}     \\
& = \bz_i^{t+1} + \eta_l \eta_g s (t-\tau_i(t+1)) \nabla F_i(\x_i^{\tau_i(t+1)+1}) -(t+1- \tau_i(t+1)) \pth{\x_i^{(t+1,0)} - \x_i^{(t+1,s)}}  \\
& \overset{(a)}{=} \bz_i^{\tau_i(t+1)+1} 
- \eta_l \eta_g s (t -\tau_i(t+1)-1+1)  \nabla F_i(\x_i^{\tau_i(t+1)+1}) \\
& \qquad
+ \eta_l \eta_g s (t-\tau_i(t+1)) \nabla F_i(\x_i^{\tau_i(t+1)+1})
-(t+1- \tau_i(t+1)) \pth{\x_i^{(t+1,0)} - \x_i^{(t+1,s)}}  \\
& = \bz_i^{\tau_i(t+1)+1}  -(t+1- \tau_i(t+1)) \pth{\x_i^{(t+1,0)} - \x_i^{(t+1,s)}}\\
& \overset{(b)}{=} \x_i^{\tau_i(t+1)+1}  -(t+1- \tau_i(t+1)) \pth{\x_i^{(t+1,0)} - \x_i^{(t+1,s)}}\\
& = \x_i^{(t+1)++},  
\end{align*}
where equality (a) follows from the auxiliary updates $\bz_i$, 
and equality (b) holds because of the induction hypothesis and the fact that $\tau_i(t+1)<t+1$ and $i\in \calA^{\tau_i(t+1)}$.  
\end{proof}

\begin{proof}{\bf Proof of Proposition \ref{prop: client dis}}
From Propositions \ref{prop: x z: inactive and active}, we have 
\begin{align*}
\norm{\x_i^t - \bz_i^t}^2
&\le 
\norm{\eta_l \eta_g s \pth{t - \tau_i(t) - 1} \nabla F_i(\x_i^t)}^2 \\
&=
\eta_l^2
\eta_g^2
s^2 \sum_{p=-1}^{t-1} 
\indc{\tau_i(t)=p}
\pth{t - p - 1}^2 \norm{\nabla F_i(\x_i^{p+1})}^2 
.
\end{align*}
Take expectation over all the randomness

\begin{align*}
\expect{\norm{\x_i^t - \bz_i^t}^2}
&\overset{(a)}{\le}
\eta_l^2 \eta_g^2 s^2 \sum_{p=-1}^{t-1} \expect{\indc{\tau_i(t)=p}} \pth{t - p - 1}^2 \expect{\norm{\nabla F_i(\x_i^{p+1})}^2} \\
&\overset{(b)}{\le}
\eta_l^2 \eta_g^2 s^2 
\sum_{p=-1}^{t-1} 
\pth{t - p - 1}^2 
\prob{\tau_i(t)=p}
\cdot
\expect{\norm{\nabla F_i(\bz_i^{p+1})}^2} 
,
\end{align*}
where inequality $(a)$ follows because by definition $\indc{\tau_i(t)=p}$ is independent of $\norm{\nabla F_i(\x_i^{p+1})}^2$,
inequality $(b)$ follows because $\x_i^{p+1} = \bz_i^{p+1}$ from Proposition \ref{prop: x z: inactive and active}.  

\begin{align*}
&\frac{1}{T}\sum_{t=0}^{T-1}
\frac{1}{M}\sum_{i=1}^M
\expect{\norm{\x_i^t - \bz_i^t}^2}
=
\frac{1}{T}\sum_{t=0}^{T-1}
\frac{1}{M}\sum_{i=1}^m
\expect{\norm{\x_i^t - \bz_i^t}^2}
\\
&= 
\eta_l^2 \eta_g^2 s^2 
\frac{1}{T}\sum_{t=0}^{T-1}\frac{1}{M}\sum_{i=1}^m \sum_{p=-1}^{t-1} 
\prob{\tau_i(t) = p}
\pth{t - p - 1}^2 \expect{\norm{\nabla F_i(\bz_i^{p+1})}^2} \\
&\overset{(c)}{\le}
\eta_l^2 \eta_g^2 s^2 \frac{1}{M} \sum_{i=1}^m \frac{1}{T}\sum_{t=0}^{T-1} \expect{\norm{\nabla F_i(\bz_i^{t})}^2} 
\pth{\expect{\pth{t - \tau_i(t)}^2}} \\
&\overset{(d)}{\le}
\eta_l^2 \eta_g^2 s^2
\pth{\variance}
\frac{1}{M} \sum_{i=1}^m \frac{1}{T}\sum_{t=0}^{T-1} 
\expect{\norm{\nabla F_i(\bz_i^{t})}^2} \\
&\le 
3 \eta_l^2 \eta_g^2 s^2 
\frac{\qth{(P_{\delta}-1) \delta + 1}^2 
+ \qth{(P_{\delta} - 1) \delta^2 + 1}}{\delta^2}
\frac{M (\beta^2 + 1)}{m}
\frac{1}{T}\sum_{t=0}^{T-1} 
\expect{\norm{\nabla \tF(\bar{\bz}^{t})}^2} \\
&~~~
+ 3 \eta_l^2 \eta_g^2 s^2 
\frac{\qth{(P_{\delta}-1) \delta + 1}^2 
+ \qth{(P_{\delta} - 1) \delta^2 + 1}}{\delta^2}
\pth{\frac{m\zeta^2}{M}} \\
&~~~
+ 3 \eta_l^2 \eta_g^2 s^2 
\pth{
\frac{L^2}{M}
\sum_{i=1}^M \frac{1}{T}\sum_{t=0}^{T-1} 
\expect{\norm{\bz_i^{t} - \bar{\bz}^t}^2} }
,
\end{align*}
where inequality $(c)$ follows from re-indexing, 
inequality $(d)$ from \prettyref{lmm: geo second moment main text}.
\end{proof}

\subsection{Consensus error of the auxiliary sequence}
\label{app: consensus}
\begin{lemma}[Consensus error of $\bz_i^t$]
    \label{lmm: consensus z}
    Assuming that $\eta_l \le \delta /(20 s L )$, and $\eta_l \eta_g \le \delta (1 - \sqrt{\rho}) / ( 10 s L (\sqrt{\rho}+1))$, under Assumptions~\ref{ass: 2 smmothness}, \ref{ass: bounded variance client-wise} and \ref{ass: bounded similarity}, it holds that
    \begin{align*}
    \frac{1}{M T}
    \sum_{t=0}^{T-1}
    &\sum_{i=1}^M
    \expect{\norm{\bz_i^t - \bar{\bz}^t}^2}
    \le 
    \frac{32 \eta_l^2 \eta_g^2 s \tP 
    }{(1 - \rp)^2}
    \frac{\pth{(\tP-1) \delta + 1}^2
    +
    \pth{(\tP - 1) \delta^2 + 1}}{\delta^2}
    \pth{\frac{m \sigma^2}{M}}\\
    &~~~+
    \frac{24 \eta_l^2 \eta_g^2 s^2
    \tP (\beta^2 +1) }{(1 - \rp)^2}
    \frac{\pth{(\tP-1) \delta + 1}^2
    +
    \pth{(\tP - 1) \delta^2 + 1}}{\delta^2}
    \pth{\frac{M}{m}}
    \frac{1}{T}
    \sum_{t=0}^{T-1}
    \expect{\norm{\nabla \tF(\bar{\bz}^t)}^2}\\
    &~~~+
    \frac{24 \eta_l^2 \eta_g^2 s^2
    \tP }{(1 - \rp)^2}
    \frac{
    \qth{(\tP-1) \delta + 1}^2
    +
    \qth{(\tP - 1) \delta^2 + 1}}{\delta^2}
    \pth{\frac{m \zeta^2}{M}}.
    \end{align*}
\end{lemma}
\begin{proof}{\bf Proof of \prettyref{lmm: consensus z}}
When $t=0$,  $ \bm{Z}^0 = [\bz^0, \cdots, \bz^0]$, which immediately leads to 
\begin{align*}
 \bm{Z}^0 \pth{\identity - \allones}  = [\bz^0, \cdots, \bz^0] - [\bz^0, \cdots, \bz^0] = \bm{0}.   
\end{align*}

For $t\ge 1$, 
recall that $W^{(t)}$ is a doubly stochastic matrix to characterize the information mixture,
and that $\tilde{\bm{G}}^t$ in~\eqref{eq: auxiliary update detail} captures the local parameter changes in each round.
Specifically, for a client $i \in \calR$, it holds that
\begin{align}
    \nonumber
    \tilde{\bm{G}}^t_{i}
    &\triangleq
    \indc{i \in \calA^t }
    \qth{
    \pth{t - \tau_i(t)}
    \sum_{r=0}^{s-1}
    \nabla \ell_i(\x_i^{(t,r)})
    - 
    s\pth{t- 1 - \tau_i(t)}
    \nabla F_i(\x_i^{\tau_i(t)+1})}\\
    \nonumber
    &~~~+
    \indc{i \notin \calA^t}
    s \nabla F_i(\x_i^{\tau_i(t)+1})\\
    &=
    \indc{i \in \calA^t}
    \pth{t - \tau_i(t)}
    \sum_{r=0}^{s-1}
    \pth{
    \nabla \ell_i(\x_i^{(t,r)})
    -
    \nabla F_i(\x_i^{t})}
    +
    s \nabla F_i(\x_i^{t}),
    \label{eq: auxiliary update detail}
\end{align}
where the last equality holds because $\x_i^t = \x_i^{\tau_i(t)+1}$ and re-grouping.
It can be seen that  
\begin{align*}
\bm{Z}^{(t)} = 
\pth{\bm{Z}^{(t-1)} - \eta_l \eta_g \tilde{\bm{G}}^{t-1}}
W^{(t-1)}.
\end{align*}
Define $n (t) = \lfloor \frac{t}{\tP} \rfloor - 1$.
For ease of presentation, we drop the variable $t$.
We have $t - 2 \tP < n \tP \le t - \tP$;
therefore, $(n + 1)\tP \le t < (n + 2) \tP$.
Further,
$ \tP \le t - n \tP < 2 \tP$.
Expanding $\bm{Z}$, we get 
\begin{align*}
\bm{Z}^{(t)} \pth{\identity - \allones} 
&= (\bm{Z}^{(t-1)} - \eta_l \eta_g \tilde{\bm{G}}^{t-1}) W^{(t-1)} \pth{\identity - \allones}\\
& = 
\bm{Z}^{n \tP} \prod_{\ell=n \tP}^{t-1} W^{\ell} \pth{\identity - \allones} 
- 
\eta_l \eta_g \sum_{q=n \tP}^{t-1} \tilde{\bm{G}}^{q} 
\prod_{\ell=q}^{t-1} W^{(\ell)} \pth{\identity- \allones}. 
\end{align*}

Let matrix notations 
${\tilde \Delta}^t$,
$\Delta^t$ and 
$\nabla \bm{F}_{\x}^t$ 
define as follows:
\begin{align*}
\bm{G}^{q}_i
&=
\underbrace{
\indc{i \in \calA^t}
(t - \tau_i(t))
\sum_{r=0}^{s-1} \pth{\nabla \ell_i(\x_i^{(t,r)};\xi_i^{(t,r)}) - \nabla F_i(\x_i^{(t,r)})}}_{
[\TDt{t}]_i
}
\\ &~~~+
\underbrace{
\indc{i \in \calA^t}
(t - \tau_i(t)) 
\sum_{r=0}^{s-1}
\pth{
\nabla F_i(\x_i^{(t,r)})
-
\nabla F_i(\x_i^t) 
}}_{
[\Dt{t}]_i
}  
+
s
\underbrace{\nabla F_i(\x_i^t)}_{
[\DFt{t}]_i
}.
\end{align*}
Assuming an absolute constant $\rrp \in (\rp, 1)$, it holds that
\begin{align*}
&\expects{\fnorm{\bm{Z}^{(t)} \pth{\identity - \allones}}^2}{W}
 =
\expects{
\fnorm{\bm{Z}^{n\tP} \prod_{\ell= n \tP}^{t-1} W^{\ell} \pth{\identity - \allones} 
- 
\eta_l \eta_g \sum_{q= n \tP}^{t-1} \tilde{\bm{G}}^{q} 
\prod_{\ell=q}^{t-1} W^{(\ell)} \pth{\identity- \allones}}^2}{W}
\\
& \overset{(a)}{\le}
\pth{1 + \frac{1}{\gamma}}
\rp
\expects{
\fnorm{
\bm{Z}^{n \tP} - \bar{\bm{Z}}^{n \tP}
}^2}{W}
+
(1 + \gamma)
\eta_l^2 \eta_g^2
\expects{
\fnorm{
\sum_{q= n \tP}^{t-1} \tilde{\bm{G}}^{q} 
\prod_{\ell=q}^{t-1} W^{(\ell)} \pth{\identity- \allones}
}^2}{W}
\\
&\overset{(b)}{\le}
\rrp
\fnorm{
\bm{Z}^{n \tP} - \bar{\bm{Z}}^{n \tP}
}^2
+
\frac{\eta_l^2 \eta_g^2}{\rrp - \rp}
\expects{
\fnorm{
\sum_{q= n \tP}^{t-1} \tilde{\bm{G}}^{q} 
\prod_{\ell=q}^{t-1} W^{(\ell)} \pth{\identity- \allones}
}^2}{W}.
\end{align*}
where inequality $(a)$ holds because Young's inequality and~\prettyref{lmm: spectral norm},
inequality $(b)$ holds by plugging in $\gamma=\frac{\rp}{\rrp - \rp}$.
Unrolling the recursion, it holds that
\begin{align*}
&\fnorm{\bm{Z}^{(t)} \pth{\identity - \allones}}^2   
\le
\rrp
\fnorm{
\bm{Z}^{n \tP} - \bar{\bm{Z}}^{n \tP}
}^2
+
\frac{\eta_l^2 \eta_g^2}{\rrp - \rp}
\fnorm{
\sum_{q= n \tP}^{t-1} \tilde{\bm{G}}^{q} 
\prod_{\ell=q}^{t-1} W^{(\ell)} \pth{\identity- \allones}
}^2 \\
&\le
\rrp^2
\fnorm{
\bm{Z}^{(n-1) \tP} - \bar{\bm{Z}}^{(n-1) \tP}
}^2
+
\frac{\eta_l^2 \eta_g^2}{\rrp - \rp}
\rrp
\fnorm{
\sum_{q= (n-1) \tP}^{n \tP -1} \tilde{\bm{G}}^{q} 
\prod_{\ell=q}^{n \tP-1} W^{(\ell)} \pth{\identity- \allones}
}^2 \\
&\qquad \qquad \qquad \qquad \qquad \qquad \quad +
\frac{\eta_l^2 \eta_g^2}{\rrp - \rp}
\fnorm{
\sum_{q= n \tP}^{t-1} \tilde{\bm{G}}^{q} 
\prod_{\ell=q}^{t-1} W^{(\ell)} \pth{\identity- \allones}
}^2 \\
&\le
\frac{\eta_l^2 \eta_g^2}{\rrp - \rp}
\pth{
\fnorm{
\sum_{q= n \tP}^{t-1} \tilde{\bm{G}}^{q} 
\prod_{\ell=q}^{t-1} W^{(\ell)} \pth{\identity- \allones}
}^2
+
\sum_{j=0}^{n-1}
\rrp^{n - j}
\fnorm{
\sum_{q= j \tP}^{ (j+1) \tP -1} \tilde{\bm{G}}^{q} 
\prod_{\ell=q}^{(j+1) \tP-1} W^{(\ell)} \pth{\identity- \allones}
}^2}
.
\end{align*}
Define 
$(A) =
    \fnorm{
    \sum_{q=a}^{b}
    \tilde{\bm{G}}^{q}
    \prod_{\ell=q}^{b} W^{(\ell)} \pth{\identity- \allones}
    }^2.
$
It remains to bound $(A)$:
\begin{align*}
    (A) &=
    \fnorm{
    \sum_{q=a}^{b}
    \pth{
    \TDt{q} + \Dt{q} + \DFt{q}
    }
    \prod_{\ell=q}^{b} W^{(\ell)} \pth{\identity- \allones}
    }^2 \\
    &=
    \fnorm{
    \sum_{q=a}^{b}
    \TDt{q} 
    \prod_{\ell=q}^{b} W^{(\ell)} \pth{\identity- \allones}
    }^2
    +
    \fnorm{
    \sum_{q=a}^{b}
    \pth{
    \Dt{q} + \DFt{q}
    }
    \prod_{\ell=q}^{b} W^{(\ell)} \pth{\identity- \allones}
    }^2 \\
    &~~~
    +2
    \iprod{\sum_{q=a}^{b}
    \TDt{q} 
    \prod_{\ell=q}^{b} W^{(\ell)} \pth{\identity- \allones}}{
    \sum_{q=a}^{b}
    \pth{
    \Dt{q} + \DFt{q}
    }
    \prod_{\ell=q}^{b} W^{(\ell)} \pth{\identity- \allones}
    }_{\rmF}.
\end{align*}
Take expectation with respect to randomness in stochastic gradients, 
denote by $\expects{\cdot}{\xi}$:
\begin{align*}
&\expects{(A)}{\xi}
=
\expects{\fnorm{
\sum_{q=a}^{b} 
\TDt{q}
\pth{\prod_{\ell=q}^{b} W^{\pth{\ell}} - \allones}}^2 }{\xi}
+
\expects{
\fnorm{
\sum_{q=a}^{b} 
\pth{\Dt{q} + \DFt{q}}
\pth{\prod_{\ell=q}^{b} W^{\pth{\ell}} - \allones}
}^2}{\xi} \\
&\qquad \qquad +
2 \expects{
\left \langle
\sum_{q=a}^{b} 
\TDt{q}
\pth{\prod_{\ell=q}^{b} W^{\pth{\ell}} - \allones},
\sum_{q=a}^{b} 
\pth{\Dt{q} + \DFt{q}}
\pth{\prod_{\ell=q}^{b} W^{\pth{\ell}} - \allones}
\right \rangle_{\rm F}}{\xi} \\
&=
\expects{\fnorm{
\sum_{q=a}^{b} 
\TDt{q}
\pth{\prod_{\ell=q}^{b} W^{\pth{\ell}} - \allones}}^2 }{\xi}
+
\expects{
\fnorm{
\sum_{q=a}^{b} 
\pth{\Dt{q} + \DFt{q}}
\pth{\prod_{\ell=q}^{t-1} W^{\pth{\ell}} - \allones}
}^2}{\xi} \\
&~~~
\qquad \qquad
+2
\left \langle
\sum_{q=a}^{b} 
 \expects{\TDt{q}}{\xi}
\pth{\prod_{\ell=q}^{b} W^{\pth{\ell}} - \allones},
\sum_{q=a}^{b} 
\pth{\Dt{q} + \DFt{q}}
\pth{\prod_{\ell=q}^{b} W^{\pth{\ell}} - \allones}
\right \rangle_{\rm F} \\
&\le
\expects{\fnorm{
\sum_{q=a}^{b} 
\TDt{q}
\pth{\prod_{\ell=q}^{b} W^{\pth{\ell}} - \allones}}^2 }{\xi}
+
\expects{
\fnorm{
\sum_{q=a}^{b} 
\pth{\Dt{q} + \DFt{q}}
\pth{\prod_{\ell=q}^{b} W^{\pth{\ell}} - \allones}
}^2}{\xi},
\end{align*}
where the last inequality holds because $\expects{\TDt{q}}{\xi} = 0$.
Next, we take expectation over the remaining randomness.
\begin{footnotesize}
\begin{align}
\notag
\expect{(A)}
&\le
\expect{\fnorm{
\sum_{q=a}^{b} 
\TDt{q}
\pth{\prod_{\ell=q}^{b} W^{\pth{\ell}} - \allones}}^2}
+
\expect{
\fnorm{
\sum_{q=a}^{b} 
\pth{\Dt{q} + \DFt{q}}
\pth{\prod_{\ell=q}^{t-1} W^{\pth{\ell}} - \allones}
}^2} \\\notag
&
\le 
\expect{\underbrace{\fnorm{
\sum_{q=a}^{b} 
    \TDt{q}
\pth{\prod_{\ell=q}^{b} W^{\pth{\ell}} - \allones}
}^2}_{(\rmI)}
} 
+ 2 
\expect{\underbrace{\fnorm{
    \sum_{q=a}^{b}
    \Dt{q}
\pth{\prod_{\ell=q}^{b} W^{\pth{\ell}} - \allones}}^2}_{(\rmI\rmI)}}
+ 2 s^2 
\expect{
\underbrace{\fnorm{\sum_{q=a}^{b}
    \DFt{q}
\pth{\prod_{\ell=q}^{t-1} W^{\pth{\ell}} - \allones}}^2}_{(\rmI\rmI\rmI)}}. 
\end{align}
\end{footnotesize}
\paragraph{Bounding $\expect{(\rmI)}$}
\begin{align}
\label{eq: conseneus iterative error 1}
    \expect{(\rmI) } &= \sum_{q=a}^{b} \expect{\fnorm{ 
        \TDt{q}
    \pth{\prod_{\ell=q}^{b} W^{\pth{\ell}} - \allones}
    }^2 } \\
     \notag
     &\qquad \qquad + \sum_{q=a}^{b} \sum_{p\neq q} \expect{\iprod{
        \TDt{p}
     \pth{\prod_{\ell=p}^{t-1}W^{\pth{\ell}} - \allones}}{
        \TDt{q}
    \pth{\prod_{\ell=q}^{t-1} W^{\pth{\ell}} - \allones}
    } }\\
    &\overset{(c)}{\le} \sum_{q=a}^{b} 
    \expect{\fnorm{\TDt{q}}^2 }
    , 
\end{align}
where inequality $(c)$ holds because of 
independent and unbiased stochastic gradients.
It remains to bound 
$\expect{\fnorm{
    \TDt{q}
}^2}$.
\begin{align*}
\fnorm{
    \TDt{q}
}^2
&=
\sum_{i=1}^m
\indc{i \in \calA^q}
\norm{
\sum_{p=-1}^{q-1}
\indc{\tau_i(t) = p}
(q - p)
\sum_{r=0}^{s-1} \pth{
 \nabla \ell_i(\x_i^{(q,r)};\xi_i^{(q,r)}) - \nabla F_i(\x_i^{(q,r)})}
 }^2.
\end{align*}
Further take expectation w.r.t. the randomness in stochastic gradients.
\begin{align*}
    \expects{\fnorm{\TDt{q}}^2}{\xi}
    &=
    \sum_{i=1}^m
    \indc{i \in \calA^q}
    \sum_{p=-1}^{q-1}
    \indc{\tau_i(t) = p}
    (q - p)^2
    \sum_{r=0}^{s-1} 
    \expects{\norm{
    \nabla \ell_i(\x_i^{(q,r)};\xi_i^{(p,r)}) - \nabla F_i(\x_i^{(q,r)})
    }^2}{\xi} \\
    &\le    
    s \sigma^2
    \sum_{i=1}^m
    \indc{i \in \calA^q}
    \sum_{p=-1}^{q-1}
    \indc{\tau_i(t) = p}
    (q - p)^2.
\end{align*}
Take expectation over the remaining randomness:
\begin{align*}
    \expect{\fnorm{\TDt{q}}^2}
    &=
    \expect{\expects{\fnorm{\TDt{q}}^2}{\xi}}
    \le    
    s \sigma^2
    \sum_{i=1}^m
    \expect{\indc{i \in \calA^q}}
    \sum_{p=-1}^{q-1}
    \expect{\indc{\tau_i(t) = p}}
    (q - p)^2 
    \le
    m s \sigma^2
    \pth{\variance}.
\end{align*}
Recall that $(\variance)$ refers to~\eqref{eq: exp var res}.
Therefore, we have
\begin{align*}
    \expect{(\rmI)}
    \le
    2 \tP
    m s \sigma^2
    \pth{\variance}.
\end{align*}

\paragraph{Bounding $\expect{(\rmI\rmI)}$}
\begin{align*}
    \expect{(\rmI\rmI)}
    &=
    \expect{\fnorm{\sum_{q=a}^{b}
    \Dt{q}
    \pth{\prod_{\ell=q}^{b} W^{\pth{\ell}} - \allones}}^2} 
    \le
    2 \tP
    \sum_{q=t-\tP}^{b}
    \expect{\fnorm{
    \Dt{q}
    \pth{\prod_{\ell=q}^{b} W^{\pth{\ell}} - \allones}}^2} \\
    &\le
    2 \tP
    \sum_{q=a}^{b}
    \expect{\fnorm{
    \Dt{q}}^2} .
\end{align*}
It remains to bound $\expect{\fnorm{\Dt{q}}^2}$.
Take expectation with respect to randomness in stochastic gradients:
\begin{align*}
\expects{\fnorm{\Dt{q}}^2}{\xi}
&\le
4 \eta_l^2 s^3 L^2 
\sum_{i=1}^m
\sum_{p=-1}^{q-1}
\indc{\tau_i(q) = p}
(q - p)^2
\sigma^2
+
16 \eta_l^2 s^4 L^2
\sum_{i=1}^m
\sum_{p=-1}^{q-1}
\indc{\tau_i(q) = p}
(q - p)^2
\norm{\nabla F_i(\x_i^q)}^2,
\end{align*}
Next, we take expectation over the remaining randomness and plug back in:
\begin{align*}
    \expect{(\rmI\rmI)}
    &\le
    16 \eta_l^2 s^3 L^2 
    \tP^2
    \sum_{i=1}^m
    \sum_{p=-1}^{q-1}
    \expect{\indc{\tau_i(q) = p}}
    (q - p)^2
    \sigma^2 \\
    &~~~+
    32 \eta_l^2 s^4 L^2
    \tP
    \sum_{q=a}^{b}
    \sum_{i=1}^m
    \expect{\norm{\nabla F_i(\x_i^q)}^2}
    \sum_{p=-1}^{q-1}
    \expect{\indc{\tau_i(q) = p}}
    (q - p)^2 \\
    &\le
    16 \eta_l^2 s^3 L^2 
    \tP^2
    m
    \sigma^2 
    \pth{\variance}
    +
    32 \eta_l^2 s^4 L^2
    \tP
    \sum_{q=a}^{b}
    \sum_{i=1}^m
    \expect{\norm{\nabla F_i(\x_i^q)}^2}
    \pth{\variance}
    ,
\end{align*}

\paragraph{Bounding $\expect{(\rmI\rmI\rmI)}$}
Use a similar trick as in bounding $\expect{(\rmI \rmI)},$ and we get
\begin{align*}
\expect{(\rmI\rmI\rmI)} &= \expect{\fnorm{\sum_{q=a}^{b}
    \DFt{q}
\pth{\prod_{\ell=q}^{b} W^{\pth{\ell}} - \allones}}^2} 
\le 
2 \tP
\sum_{q=a}^{b} 
\sum_{i=1}^m
\expect{\norm{\nabla F_i(\x_i^q)}^2}.
\end{align*}

Hence, we have
\begin{align*}
    &\expect{
    (A)
    }
    \le
    2 \tP \pth{\variance}
    \pth{
    m s \sigma^2 
    \pth{
    1 + 16 \eta_l^2 s^2 L^2 \tP \pth{\variance}}
    +
    2 s^2
    \pth{
    1+
    16 \eta_l^2 s^2 L^2}
    \sum_{q=a}^{b}
    \sum_{i=1}^m
    \expect{\norm{\nabla F_i(\x_i^q)}^2}
    }
\end{align*}

It follows that
\begin{small}
\begin{align*}
&\expect{\fnorm{\bm{Z}^{(t)} \pth{\identity - \allones}}^2 }
\le
\frac{\eta_l^2 \eta_g^2}{\rrp - \rp}
\pth{
\fnorm{
\sum_{q= n \tP}^{t-1} \tilde{\bm{G}}^{q} 
\prod_{\ell=q}^{t-1} W^{(\ell)} \pth{\identity- \allones}
}^2
+
\sum_{j=0}^{n-1}
\rrp^{n - j}
\fnorm{
\sum_{q= j \tP}^{ (j+1) \tP -1} \tilde{\bm{G}}^{q} 
\prod_{\ell=q}^{(j+1) \tP-1} W^{(\ell)} \pth{\identity- \allones}
}^2} \\
&\le
2 \tP \pth{\variance}
\frac{\eta_l^2 \eta_g^2}{\rrp - \rp}
\pth{
m s \sigma^2 
\pth{
1 + 16 \eta_l^2 s^2 L^2 \tP \pth{\variance}}
+
2 s^2
\pth{
1+
16 \eta_l^2 s^2 L^2}
\sum_{q=n \tP}^{t-1}
\sum_{i=1}^m
\expect{\norm{\nabla F_i(\x_i^q)}^2}
} \\
&~~~+
2 \tP \pth{\variance}
\frac{\eta_l^2 \eta_g^2}{\rrp - \rp}
\sum_{j=0}^{n-1}
\rrp^{n - j}
\pth{
m s \sigma^2 
\pth{
1 + 16 \eta_l^2 s^2 L^2 \tP \pth{\variance}}
+
2 s^2
\pth{
1+
16 \eta_l^2 s^2 L^2}
\sum_{q=j \tP}^{(j+1)\tP-1}
\sum_{i=1}^m
\expect{\norm{\nabla F_i(\x_i^q)}^2}
} 
.
\end{align*}
\end{small}

Rearranging the terms, we have
\begin{small}
\begin{align*}
&\expect{\fnorm{\bm{Z}^{(t)} \pth{\identity - \allones}}^2 }
\le
2 \tP \pth{\variance}
\frac{\eta_l^2 \eta_g^2}{\rrp - \rp}
\sum_{j=0}^{n}
\rrp^{n - j}
\pth{
m s \sigma^2 
\pth{
1 + 16 \eta_l^2 s^2 L^2 \tP \pth{\variance}}}
 \\
&~~~+
4 s^2
\pth{
1+
16 \eta_l^2 s^2 L^2}
\tP \pth{\variance}
\frac{\eta_l^2 \eta_g^2}{\rrp - \rp}
\pth{
\sum_{q=n \tP}^{t-1}
\sum_{i=1}^m
\expect{\norm{\nabla F_i(\x_i^q)}^2}
+
\sum_{j=0}^{n}
\rrp^{n - j}
\sum_{q=j \tP}^{(j+1)\tP-1}
\sum_{i=1}^m
\expect{\norm{\nabla F_i(\x_i^q)}^2}
} \\
&\le
2 \tP \pth{\variance}
\frac{\eta_l^2 \eta_g^2}{(\rrp - \rp) (1 - \rrp)}
\pth{
m s \sigma^2 
\pth{
1 + 16 \eta_l^2 s^2 L^2 \tP \pth{\variance}}}
 \\
&~~~+
4 s^2
\pth{
1+
16 \eta_l^2 s^2 L^2}
\tP \pth{\variance}
\frac{\eta_l^2 \eta_g^2}{\rrp - \rp}
\pth{
\sum_{q=n \tP}^{t-1}
\sum_{i=1}^m
\expect{\norm{\nabla F_i(\x_i^q)}^2}
+
\sum_{j=0}^{n}
\rrp^{n - j}
\sum_{q=j \tP}^{(j+1)\tP-1}
\sum_{i=1}^m
\expect{\norm{\nabla F_i(\x_i^q)}^2}
} 
.
\end{align*}
\end{small}
It follows that
\begin{small}
\begin{align*}
&
\frac{1}{M T}
\sum_{t=0}^{T-1}
\expect{\fnorm{\bm{Z}^{(t)} \pth{\identity - \allones}}^2 }
\le
2 \tP \pth{\variance}
\frac{\eta_l^2 \eta_g^2}{(\rrp - \rp) (1 - \rrp)}
\pth{
\frac{m s \sigma^2}{M} 
\pth{
1 + 16 \eta_l^2 s^2 L^2 \tP \pth{\variance}}}
 \\
&+
4 s^2
\pth{
1+
16 \eta_l^2 s^2 L^2}
\tP \pth{\variance}
\frac{\eta_l^2 \eta_g^2}{\rrp - \rp}
\frac{1}{M T}
\sum_{t=0}^{T-1}
\pth{
\sum_{q=n \tP}^{t-1}
\sum_{i=1}^m
\expect{\norm{\nabla F_i(\x_i^q)}^2}
+
\sum_{j=0}^{n}
\rrp^{n - j}
\sum_{q=j \tP}^{(j+1)\tP-1}
\sum_{i=1}^m
\expect{\norm{\nabla F_i(\x_i^q)}^2}
} \\
&\le 
4 \tP \pth{\variance}
\frac{\eta_l^2 \eta_g^2}{(\rrp - \rp) (1 - \rrp)}
\pth{\frac{m s \sigma^2 }{M}}
+
16 s^2
\tP^2 \pth{\variance}
\frac{\eta_l^2 \eta_g^2}{\rrp - \rp}
\frac{1}{M T}
\sum_{t=0}^{T-1}
\sum_{i=1}^m
\expect{\norm{\nabla F_i(\x_i^t)}^2}
\pth{
\sum_{q=0}^{\infty}
\rrp^{q}
} \\
&\le 
4 \tP \pth{\variance}
\frac{\eta_l^2 \eta_g^2}{(\rrp - \rp) (1 - \rrp)}
\frac{m s \sigma^2 }{M}
+
16 s^2
\tP^2 \pth{\variance}
\frac{\eta_l^2 \eta_g^2}{(\rrp - \rp) (1 - \rrp)}
\frac{1}{M T}
\sum_{t=0}^{T-1}
\sum_{i=1}^m
\expect{\norm{\nabla F_i(\x_i^t)}^2}
.
\end{align*}
\end{small}
Further, it holds that
\begin{small}
\begin{align*}
    \frac{1}{M T}
    \sum_{t=0}^{T-1}
    \sum_{i=1}^m
    \expect{\norm{\nabla F_i(\x_i^t)}^2}
    &\le 
    \eta_l^2 \eta_g^2 s^2
    \pth{\variance}
    \frac{1}{M T}
    \sum_{t=0}^{T-1}
    \sum_{i=1}^m
    \expect{
    \norm{\nabla F_i(\bz_i^t)}^2
    } \\
    &\le
    \frac{3 M (\beta^2 + 1)}{m} 
    \eta_l^2 \eta_g^2 s^2
    \pth{\variance}
    \frac{1}{T}
    \sum_{t=0}^{T-1}
    \expect{
    \norm{\nabla \tF(\bar{\bz}^t)}^2
    } \\
    &~~~+
    \frac{3 m
    \eta_l^2 \eta_g^2 s^2
    \pth{\variance}
    \zeta^2}{M} \\
    &~~~+
    3 
    \eta_l^2 \eta_g^2 s^2 L^2
    \pth{\variance}
    \frac{1}{MT}
    \sum_{t=0}^{T-1}
    \sum_{i=1}^M
    \expect{\norm{\bz_i^t - \bar{\bz}^t}^2}
    .
\end{align*}
\end{small}
Plug it back in, we have
\begin{small}
\begin{align*}
\frac{1}{M T}
\sum_{t=0}^{T-1}
\sum_{i=1}^M
\expect{\norm{\bz_i^t - \bar{\bz}^t}^2}
&\le 
4 \tP \pth{\variance}
\frac{\eta_l^2 \eta_g^2}{(\rrp - \rp) (1 - \rrp)}
\frac{s m \sigma^2}{M}  \\
&~~~+
\frac{48 M}{m} 
\frac{\eta_l^4 \eta_g^4 (\beta^2 +1) s^4 \tP^2 \pth{\variance}^2}{(\rrp - \rp) (1 - \rrp)}
\frac{1}{T}
\sum_{t=0}^{T-1}
\expect{
\norm{\nabla \tF(\bar{\bz}^t)}^2
}
\\
&~~~+
48 s^4
\tP^2 
\pth{\variance}^2
\frac{\eta_l^4 \eta_g^4}{(\rrp - \rp) (1 - \rrp)}
\pth{\frac{m \zeta^2}{M}}
\\
&~~~+
3 \eta_l^2 \eta_g^2 s^2
\tP 
\pth{\variance}
\frac{1}{(\rrp - \rp) (1 - \rrp)}
\frac{1}{M T}
\sum_{t=0}^{T-1}
\sum_{i=1}^M
\pth{\expect{\norm{\bz_i^t - \bar{\bz}^t}^2}}.
\end{align*}
\end{small}
Under our choice of learning rate condition, it holds that
\begin{align*}
\frac{1}{M T}
\sum_{t=0}^{T-1}
\sum_{i=1}^M
\expect{\norm{\bz_i^t - \bar{\bz}^t}^2}
&\le 
\frac{8 \eta_l^2 \eta_g^2 s \tP \pth{\variance}
}{(\rrp - \rp) (1 - \rrp)}
\frac{m \sigma^2}{M}\\
&~~~+
\frac{6 \eta_l^2 \eta_g^2 s^2
\tP \pth{\variance} (\beta^2 +1) }{ (\rrp - \rp) (1 - \rrp)}
\pth{\frac{M}{m}}
\frac{1}{T}
\sum_{t=0}^{T-1}
\expect{
\norm{\nabla \tF(\bar{\bz}^t)}^2
}\\
&~~~+
\frac{6 \eta_l^2 \eta_g^2 s^2
\tP \pth{\variance}}{(\rrp - \rp) (1 - \rrp)}
\pth{\frac{m \zeta^2}{M}}.
\end{align*}
By taking $\rrp = \frac{1 + \rp}{2}$, we get
\begin{align*}
\frac{1}{M T}
\sum_{t=0}^{T-1}
\sum_{i=1}^M
\expect{\norm{\bz_i^t - \bar{\bz}^t}^2}
&\le 
\frac{32 \eta_l^2 \eta_g^2 s \tP \pth{\variance}
}{(1 - \rp)^2}
\pth{\frac{m \sigma^2}{M}} \\
&~~~+
\frac{24 \eta_l^2 \eta_g^2 s^2
\tP \pth{\variance} 
(\beta^2 +1) }{
(1 - \rp)^2}
\pth{\frac{M}{m}}
\frac{1}{T}
\sum_{t=0}^{T-1}
\expect{
\norm{\nabla \tF(\bar{\bz}^t)}^2
}
\\
&~~~+
\frac{24 \eta_l^2 \eta_g^2 s^2
\tP \pth{\variance}}{(1 - \rp)^2}
\pth{\frac{m \zeta^2}{M}}.
\end{align*}
\end{proof}

\subsection{Spectral norm upper bound (\prettyref{lmm: rho upper bound main text})}
\begin{proof}{\bf Proof of~\prettyref{lmm: rho upper bound main text}}
Before we dive into the concrete proof, 
let us first go over some additional notations and agree on a fact about matrix products.

\paragraph{Notations.}
$\prod_{t=1}^k W^{(t)} \triangleq W^{(1)} W^{(2)} \cdots W^{(k)}$ defines an ordered matrix product from $W^{(1)}$ to $W^{(k)}$, where $k\in \integers^+$, and
an undirected graph $\calG_t \triangleq \sth{\calV_{\calG_t}, \calE_t, W^{(t)}}$, 
where $\calV_{\calG_t} \triangleq \calR \cup \calV$ defines the set of vertices for all $t \in [T]$,
$W^{(t)}$ defines the weighted adjacency matrix, 
an edge $(i,j) \in \calE_t$ if $W^{(t)}_{ij}>0$ for $i,j \in [M]$.

\paragraph{A fact about matrix product.}
By definition of matrix products, we have 
\begin{align}
\label{eq: graph directed walk}
\prod_{t=1}^{\tP} W^{(t)}_{ij}
&=
\sum_{k_{1} = 1}^M
W^{(1)}_{i k_{1}}
\cdots
\sum_{k_{\tP-2} = 1}^M
W^{(\tP-2)}_{k_{\tP-3} k_{\tP-2}}
\sum_{k_{\tP-1} = 1}^M
W^{(\tP-1)}_{k_{\tP-2} k_{\tP-1}}
W^{(\tP)}_{k_{\tP-1} j}.
\end{align}
It is observed from~\eqref{eq: graph directed walk} that 
$\prod_{t=1}^{\tP} W^{(t)}_{ij} > 0$ if there exists at least one 
\[
    W^{(1)}_{i k_{1}} 
    W^{(2)}_{k_{1} k_{2}} 
    \cdots 
    W^{(\tP-1)}_{k_{\tP-2} k_{\tP-1}} 
    W^{(\tP)}_{k_{\tP-1} j} >0,
\] 
where $k_1,\ldots, k_{\tP-1} \in \calR \cup \calV$.
In the language of graph theory, 
$\prod_{t=1}^{\tP} W^{(t)}_{ij}>0$ if there exists a directed path starting from node $i$ of $\calG_1$ and ends at node $j$ of $\calG_{\tP}$ through $\calG_2$, $\calG_3$, $\ldots$, $\calG_{\tP-1}$ sequentially.
Therefore, to lower bound elements in $\prod_{t=1}^{\tP} W_t$, it suffices to find the paths of interest for the lower bound.

\paragraph{Element-wise lower bound of $W^{(t)}$ matrix for all $t \in [T]$.}

Recall the definition of a $W$ matrix:
\begin{align*}
\Wt{t}_{ij} = 
\begin{cases}
\frac{\indc{i \in \calA^t} \indc{j \in \calA^t}}{\abth{\calA^t} + k},~& \text{if } i \neq j ~\text{and} ~ i,j \in \calR;\\
\frac{\indc{i \in \calA^t}}{\abth{\calA^t} + k},~& \text{if } 
i \in \calR ~\text{and} ~j \in \calV;\\
\frac{\indc{j \in \calA^t}}{\abth{\calA^t} + k},~& \text{if } i \in \calV ~\text{and} ~j \in \calR;\\
\frac{\indc{i \in \calA^t}}{\abth{\calA^t} + k} + (1 - \indc{i \in \calA^t}),~& \text{if } i = j ~\text{and} ~ i \in \calR;\\
\frac{1}{\abth{\calA^t} + k},~& \text{if } i = j ~\text{and} ~ i \in \calV,
\end{cases}
\end{align*}
Next, we focus on the diagonal element in $\Wt{t}$, we have two cases:
\begin{itemize}
\item When $i \in \calR$, we have
\begin{align*}
    \Wt{t}_{ii}&=
    \frac{\indc{i \in \calA^t}}{\abth{\calA^t} + k} + (1 - \indc{i \in \calA^t}) 
    \ge
    \frac{\indc{i \in \calA^t}}{\abth{\calA^t} + k} + \frac{(1 - \indc{i \in \calA^t})}{\abth{\calA^t} + k}
    =
    \frac{1}{\abth{\calA^t} + k};
\end{align*}
\item When $i \in \calV$, we have
\[
    \Wt{t}_{ii}=
    \frac{1}{\abth{\calA^t} + k}.
\]
\end{itemize}
For the edge weight of $ij$, where $i \neq j$,
it holds that
$$
\Wt{t}_{ij} 
\ge
\frac{\indc{i\in [M]} \indc{j \in [M]}}{\abth{\calA^t} + k}.
$$

\paragraph{Element-wise lower bound of $\pth{\prod_{t = t_0 }^{t_0 + P-1} \Wt{t}}^2$.}
\begin{figure}[!htb]
\centering
\begin{subfigure}[b]{\textwidth}
\centering
\scalebox{.8}{

\tikzset{every picture/.style={line width=0.75pt}} %

\begin{tikzpicture}[x=0.75pt,y=0.75pt,yscale=-1,xscale=1]

\draw   (102.97,147.89) .. controls (102.97,127.79) and (119.26,111.5) .. (139.36,111.5) .. controls (159.46,111.5) and (175.75,127.79) .. (175.75,147.89) .. controls (175.75,167.99) and (159.46,184.28) .. (139.36,184.28) .. controls (119.26,184.28) and (102.97,167.99) .. (102.97,147.89) -- cycle ;
\draw  [draw opacity=0][dash pattern={on 4.5pt off 4.5pt}] (102.53,165.57) .. controls (100.4,166.88) and (98.02,167.84) .. (95.43,168.36) .. controls (83.54,170.77) and (71.94,163.09) .. (69.53,151.2) .. controls (67.12,139.31) and (74.8,127.71) .. (86.69,125.3) .. controls (94.47,123.72) and (102.12,126.46) .. (107.16,131.87) -- (91.06,146.83) -- cycle ; \draw  [dash pattern={on 4.5pt off 4.5pt}] (102.53,165.57) .. controls (100.4,166.88) and (98.02,167.84) .. (95.43,168.36) .. controls (83.54,170.77) and (71.94,163.09) .. (69.53,151.2) .. controls (67.12,139.31) and (74.8,127.71) .. (86.69,125.3) .. controls (94.47,123.72) and (102.12,126.46) .. (107.16,131.87) ;  
\draw  [color={rgb, 255:red, 0; green, 0; blue, 0 }  ,draw opacity=1 ][fill={rgb, 255:red, 0; green, 0; blue, 0 }  ,fill opacity=1 ] (99.68,164.31) -- (105.83,163.53) -- (102.98,169.03) -- (103.58,165.1) -- cycle ;

\draw    (175.75,147.89) -- (289.97,147.89) ;
\draw [shift={(292.97,147.89)}, rotate = 180] [fill={rgb, 255:red, 0; green, 0; blue, 0 }  ][line width=0.08]  [draw opacity=0] (10.72,-5.15) -- (0,0) -- (10.72,5.15) -- (7.12,0) -- cycle    ;
\draw   (292.97,147.89) .. controls (292.97,127.79) and (309.26,111.5) .. (329.36,111.5) .. controls (349.46,111.5) and (365.75,127.79) .. (365.75,147.89) .. controls (365.75,167.99) and (349.46,184.28) .. (329.36,184.28) .. controls (309.26,184.28) and (292.97,167.99) .. (292.97,147.89) -- cycle ;

\draw  [draw opacity=0][dash pattern={on 4.5pt off 4.5pt}] (346.47,110.63) .. controls (347.77,108.5) and (348.73,106.11) .. (349.25,103.52) .. controls (351.65,91.63) and (343.96,80.04) .. (332.07,77.64) .. controls (320.17,75.24) and (308.59,82.93) .. (306.18,94.82) .. controls (304.61,102.6) and (307.36,110.25) .. (312.77,115.28) -- (327.72,99.17) -- cycle ; \draw  [dash pattern={on 4.5pt off 4.5pt}] (346.47,110.63) .. controls (347.77,108.5) and (348.73,106.11) .. (349.25,103.52) .. controls (351.65,91.63) and (343.96,80.04) .. (332.07,77.64) .. controls (320.17,75.24) and (308.59,82.93) .. (306.18,94.82) .. controls (304.61,102.6) and (307.36,110.25) .. (312.77,115.28) ;  
\draw  [color={rgb, 255:red, 0; green, 0; blue, 0 }  ,draw opacity=1 ][fill={rgb, 255:red, 0; green, 0; blue, 0 }  ,fill opacity=1 ] (345.2,107.78) -- (344.43,113.93) -- (349.93,111.07) -- (346,111.68) -- cycle ;
Straight Lines [id:da8923532475310485] 
\draw    (365.75,148.89) -- (479.97,148.89) ;
\draw [shift={(482.97,148.89)}, rotate = 180] [fill={rgb, 255:red, 0; green, 0; blue, 0 }  ][line width=0.08]  [draw opacity=0] (10.72,-5.15) -- (0,0) -- (10.72,5.15) -- (7.12,0) -- cycle    ;
\draw   (557.19,147.89) .. controls (557.19,127.79) and (540.9,111.5) .. (520.81,111.5) .. controls (500.71,111.5) and (484.42,127.79) .. (484.42,147.89) .. controls (484.42,167.99) and (500.71,184.28) .. (520.81,184.28) .. controls (540.9,184.28) and (557.19,167.99) .. (557.19,147.89) -- cycle ;
\draw  [draw opacity=0][dash pattern={on 4.5pt off 4.5pt}] (557.63,165.57) .. controls (559.76,166.88) and (562.15,167.84) .. (564.74,168.36) .. controls (576.63,170.77) and (588.22,163.09) .. (590.64,151.2) .. controls (593.05,139.31) and (585.37,127.71) .. (573.48,125.3) .. controls (565.7,123.72) and (558.05,126.46) .. (553.01,131.87) -- (569.11,146.83) -- cycle ; \draw  [dash pattern={on 4.5pt off 4.5pt}] (557.63,165.57) .. controls (559.76,166.88) and (562.15,167.84) .. (564.74,168.36) .. controls (576.63,170.77) and (588.22,163.09) .. (590.64,151.2) .. controls (593.05,139.31) and (585.37,127.71) .. (573.48,125.3) .. controls (565.7,123.72) and (558.05,126.46) .. (553.01,131.87) ;  
\draw  [color={rgb, 255:red, 0; green, 0; blue, 0 }  ,draw opacity=1 ][fill={rgb, 255:red, 0; green, 0; blue, 0 }  ,fill opacity=1 ] (560.49,164.31) -- (554.34,163.53) -- (557.19,169.03) -- (556.59,165.1) -- cycle ;

\draw (135.14,140.5) node [anchor=north west][inner sep=0.75pt]  [font=\Large]  {$i$};
\draw (317.23,140.5) node [anchor=north west][inner sep=0.75pt]  [font=\Large] [align=left] {$\calV$};
\draw (515.14,140.5) node [anchor=north west][inner sep=0.75pt]  [font=\Large]  {$j$};
\draw (38,130.68) node [anchor=north west][inner sep=0.75pt]    {\Large$\frac{1}{M}$};
\draw (360,80.68) node [anchor=north west][inner sep=0.75pt]    {\Large$\frac{1}{M}$};
\draw (599,130.68) node [anchor=north west][inner sep=0.75pt]    {\Large$\frac{1}{M}$};
\draw (201,152.68) node [anchor=north west][inner sep=0.75pt]    {\Large$\frac{\indc{i\in\calA^t}}{M}$};
\draw (401,152.68) node [anchor=north west][inner sep=0.75pt]    {\Large$\frac{\indc{j\in\calA^t}}{M}$};

\end{tikzpicture} 
}
\caption{A directed path from node $i$ to $j$ through a virtual node.}
\label{fig: ij complete path}
\end{subfigure}
\begin{subfigure}[b]{.49\textwidth}
\centering
\scalebox{.8}{

\tikzset{every picture/.style={line width=0.75pt}} %

\begin{tikzpicture}[x=0.75pt,y=0.75pt,yscale=-1,xscale=1]

\draw   (102.97,147.89) .. controls (102.97,127.79) and (119.26,111.5) .. (139.36,111.5) .. controls (159.46,111.5) and (175.75,127.79) .. (175.75,147.89) .. controls (175.75,167.99) and (159.46,184.28) .. (139.36,184.28) .. controls (119.26,184.28) and (102.97,167.99) .. (102.97,147.89) -- cycle ;
\draw  [draw opacity=0][dash pattern={on 4.5pt off 4.5pt}] (102.53,165.57) .. controls (100.4,166.88) and (98.02,167.84) .. (95.43,168.36) .. controls (83.54,170.77) and (71.94,163.09) .. (69.53,151.2) .. controls (67.12,139.31) and (74.8,127.71) .. (86.69,125.3) .. controls (94.47,123.72) and (102.12,126.46) .. (107.16,131.87) -- (91.06,146.83) -- cycle ; \draw  [dash pattern={on 4.5pt off 4.5pt}] (102.53,165.57) .. controls (100.4,166.88) and (98.02,167.84) .. (95.43,168.36) .. controls (83.54,170.77) and (71.94,163.09) .. (69.53,151.2) .. controls (67.12,139.31) and (74.8,127.71) .. (86.69,125.3) .. controls (94.47,123.72) and (102.12,126.46) .. (107.16,131.87) ;  
\draw  [color={rgb, 255:red, 0; green, 0; blue, 0 }  ,draw opacity=1 ][fill={rgb, 255:red, 0; green, 0; blue, 0 }  ,fill opacity=1 ] (99.68,164.31) -- (105.83,163.53) -- (102.98,169.03) -- (103.58,165.1) -- cycle ;

\draw    (175.75,147.89) -- (289.97,147.89) ;
\draw [shift={(292.97,147.89)}, rotate = 180] [fill={rgb, 255:red, 0; green, 0; blue, 0 }  ][line width=0.08]  [draw opacity=0] (10.72,-5.15) -- (0,0) -- (10.72,5.15) -- (7.12,0) -- cycle    ;
\draw   (292.97,147.89) .. controls (292.97,127.79) and (309.26,111.5) .. (329.36,111.5) .. controls (349.46,111.5) and (365.75,127.79) .. (365.75,147.89) .. controls (365.75,167.99) and (349.46,184.28) .. (329.36,184.28) .. controls (309.26,184.28) and (292.97,167.99) .. (292.97,147.89) -- cycle ;

\draw  [draw opacity=0][dash pattern={on 4.5pt off 4.5pt}] (346.47,110.63) .. controls (347.77,108.5) and (348.73,106.11) .. (349.25,103.52) .. controls (351.65,91.63) and (343.96,80.04) .. (332.07,77.64) .. controls (320.17,75.24) and (308.59,82.93) .. (306.18,94.82) .. controls (304.61,102.6) and (307.36,110.25) .. (312.77,115.28) -- (327.72,99.17) -- cycle ; \draw  [dash pattern={on 4.5pt off 4.5pt}] (346.47,110.63) .. controls (347.77,108.5) and (348.73,106.11) .. (349.25,103.52) .. controls (351.65,91.63) and (343.96,80.04) .. (332.07,77.64) .. controls (320.17,75.24) and (308.59,82.93) .. (306.18,94.82) .. controls (304.61,102.6) and (307.36,110.25) .. (312.77,115.28) ;  
\draw  [color={rgb, 255:red, 0; green, 0; blue, 0 }  ,draw opacity=1 ][fill={rgb, 255:red, 0; green, 0; blue, 0 }  ,fill opacity=1 ] (345.2,107.78) -- (344.43,113.93) -- (349.93,111.07) -- (346,111.68) -- cycle ;

\draw (135.14,140.5) node [anchor=north west][inner sep=0.75pt]  [font=\Large]  {$i$};
\draw (317.23,140.5) node [anchor=north west][inner sep=0.75pt]  [font=\Large] [align=left] {$\calV$};
\draw (38,130.68) node [anchor=north west][inner sep=0.75pt]    {\Large$\frac{1}{M}$};
\draw (360,80.68) node [anchor=north west][inner sep=0.75pt]    {\Large$\frac{1}{M}$};
\draw (201,152.68) node [anchor=north west][inner sep=0.75pt]    {\Large$\frac{\indc{i\in\calA^t}}{M}$};
\end{tikzpicture} 
}
\caption{A directed path from node $i$ to a virtual node.}
\label{fig: path from i to PS}
\end{subfigure}
\begin{subfigure}[b]{.49\textwidth}
\centering
\scalebox{.8}{

\tikzset{every picture/.style={line width=0.75pt}} %

\begin{tikzpicture}[x=0.75pt,y=0.75pt,yscale=-1,xscale=1]

\draw   (292.97,147.89) .. controls (292.97,127.79) and (309.26,111.5) .. (329.36,111.5) .. controls (349.46,111.5) and (365.75,127.79) .. (365.75,147.89) .. controls (365.75,167.99) and (349.46,184.28) .. (329.36,184.28) .. controls (309.26,184.28) and (292.97,167.99) .. (292.97,147.89) -- cycle ;

\draw  [draw opacity=0][dash pattern={on 4.5pt off 4.5pt}] (346.47,110.63) .. controls (347.77,108.5) and (348.73,106.11) .. (349.25,103.52) .. controls (351.65,91.63) and (343.96,80.04) .. (332.07,77.64) .. controls (320.17,75.24) and (308.59,82.93) .. (306.18,94.82) .. controls (304.61,102.6) and (307.36,110.25) .. (312.77,115.28) -- (327.72,99.17) -- cycle ; \draw  [dash pattern={on 4.5pt off 4.5pt}] (346.47,110.63) .. controls (347.77,108.5) and (348.73,106.11) .. (349.25,103.52) .. controls (351.65,91.63) and (343.96,80.04) .. (332.07,77.64) .. controls (320.17,75.24) and (308.59,82.93) .. (306.18,94.82) .. controls (304.61,102.6) and (307.36,110.25) .. (312.77,115.28) ;  
\draw  [color={rgb, 255:red, 0; green, 0; blue, 0 }  ,draw opacity=1 ][fill={rgb, 255:red, 0; green, 0; blue, 0 }  ,fill opacity=1 ] (345.2,107.78) -- (344.43,113.93) -- (349.93,111.07) -- (346,111.68) -- cycle ;
Straight Lines [id:da8923532475310485] 
\draw    (365.75,148.89) -- (479.97,148.89) ;
\draw [shift={(482.97,148.89)}, rotate = 180] [fill={rgb, 255:red, 0; green, 0; blue, 0 }  ][line width=0.08]  [draw opacity=0] (10.72,-5.15) -- (0,0) -- (10.72,5.15) -- (7.12,0) -- cycle    ;
\draw   (557.19,147.89) .. controls (557.19,127.79) and (540.9,111.5) .. (520.81,111.5) .. controls (500.71,111.5) and (484.42,127.79) .. (484.42,147.89) .. controls (484.42,167.99) and (500.71,184.28) .. (520.81,184.28) .. controls (540.9,184.28) and (557.19,167.99) .. (557.19,147.89) -- cycle ;
\draw  [draw opacity=0][dash pattern={on 4.5pt off 4.5pt}] (557.63,165.57) .. controls (559.76,166.88) and (562.15,167.84) .. (564.74,168.36) .. controls (576.63,170.77) and (588.22,163.09) .. (590.64,151.2) .. controls (593.05,139.31) and (585.37,127.71) .. (573.48,125.3) .. controls (565.7,123.72) and (558.05,126.46) .. (553.01,131.87) -- (569.11,146.83) -- cycle ; \draw  [dash pattern={on 4.5pt off 4.5pt}] (557.63,165.57) .. controls (559.76,166.88) and (562.15,167.84) .. (564.74,168.36) .. controls (576.63,170.77) and (588.22,163.09) .. (590.64,151.2) .. controls (593.05,139.31) and (585.37,127.71) .. (573.48,125.3) .. controls (565.7,123.72) and (558.05,126.46) .. (553.01,131.87) ;  
\draw  [color={rgb, 255:red, 0; green, 0; blue, 0 }  ,draw opacity=1 ][fill={rgb, 255:red, 0; green, 0; blue, 0 }  ,fill opacity=1 ] (560.49,164.31) -- (554.34,163.53) -- (557.19,169.03) -- (556.59,165.1) -- cycle ;

\draw (317.23,140.5) node [anchor=north west][inner sep=0.75pt]  [font=\Large] [align=left] {$\calV$};
\draw (515.14,140.5) node [anchor=north west][inner sep=0.75pt]  [font=\Large]  {$j$};
\draw (360,80.68) node [anchor=north west][inner sep=0.75pt]    {\Large$\frac{1}{M}$};
\draw (599,130.68) node [anchor=north west][inner sep=0.75pt]    {\Large$\frac{1}{M}$};
\draw (401,152.68) node [anchor=north west][inner sep=0.75pt]    {\Large$\frac{\indc{j\in\calA^t}}{M}$};
\end{tikzpicture}    
}
\caption{A directed path from a virtual node to node $j$.}
\label{fig: path from PS to j}
\end{subfigure}
\caption{An illustration of a directed path from node $i$ to $j$ through an intermediate virtual node.
The dashed arcs are self-loops.
The weights next to the dashed arcs and solid lines are the lower bounds of each edge weight, respectively.
}
\label{fig: directed path i to j}
\end{figure}
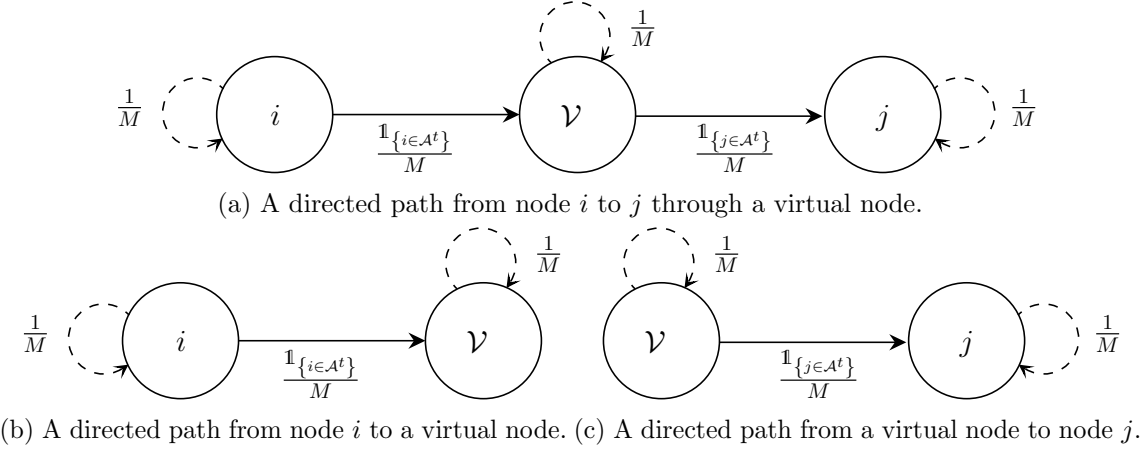

\prettyref{fig: directed path i to j} presents a simple example of a path from node $i$ to $j$, which visualizes the transition between different nodes.
We will use it to assist our proof.
By definition, we have
\begin{align*}
\pth{\prod_{t = t_0 }^{t_0 + \tP-1} \Wt{t}}^2
=
\pth{\prod_{t = t_0 }^{t_0 + \tP-1} \Wt{t}}^\top \pth{\prod_{t = t_0 }^{t_0 + \tP-1} \Wt{t}} 
=
\underbrace{
\pth{\prod_{t = t_0 + \tP-1}^{t_0} \Wt{t}}}_{(\rmI)}
\underbrace{
\pth{\prod_{t = t_0 }^{t_0 + \tP-1} \Wt{t}}}_{(\rmI\rmI)}.
\end{align*}
To lower bound $[\pth{\prod_{t = t_0 }^{t_0 + \tP-1} \Wt{t}}^2]_{ij}$, where $i \neq j$.
It suffices to let a node $i$ take one path to an intermediate node $h$ in $(\rmI)$ and then take another path from the intermediate node $h$ back to node $j$ in $(\rmI\rmI)$.
We will return to the case $i = j$ later.
Let $h \in \calV$.
\begin{itemize}
\item In $(\rmI)$ (\prettyref{fig: path from i to PS}), we start from node $i$ in $\Wt{t_0+P-1}$ to node $h$ in $\Wt{t_0}$ via a directed path.
One possible path is to stay in self-loops as long as one can, and jump only once from node $i$ to node $h$.
It holds that
\begin{align}
\label{eq: rmI product lower bound }
[(\rmI)]_{i h} \ge
\sum_{t=t_0}^{t_0+ \tP-1}
\pth{
\underbrace{
\frac{1}{\prod_{t^\prime \neq t}\abth{\calA^{t^\prime}} + k}}_{\scriptsize\text{self-loop weights}}~\cdot~
\underbrace{\frac{\indc{i\in\calA^t}}{\abth{\calA^t} + k}}_{\text{one jump weight}}
}
=
\frac{\sum_{t=t_0}^{t_0+\tP-1}\indc{i\in\calA^t}}{\prod_{t^\prime=t_0}^{t_0 + \tP - 1} \pth{\abth{\calA^{t^\prime}} + k}}.
\end{align}
\item 
In $(\rmI\rmI)$ (\prettyref{fig: path from PS to j}), we can show the results similarly in the following
\begin{align}
\label{eq: rmII product lower bound }
[(\rmI\rmI)]_{h j} \ge
\sum_{t=t_0}^{t_0+\tP-1}
\pth{
\underbrace{\frac{\indc{i\in\calA^t}}{\abth{\calA^t} + k}}_{\text{one jump weight}}
~\cdot~
\underbrace{\frac{1}{\prod_{t^\prime \neq t}\abth{\calA^{t^\prime}} + k}}_{\scriptsize\text{self-loop weights}}}
=
\frac{\sum_{t=t_0}^{t_0+\tP-1}\indc{j\in\calA^t}}{\prod_{t^\prime=t_0}^{t_0 + \tP - 1} \pth{\abth{\calA^{t^\prime}} + k}}.
\end{align}
\end{itemize}
Hence, when $h \in \calV$, 
we have
\begin{align*}
    \sum_{h \in \calV} 
    [(\rmI)]_{i h} \cdot [(\rmI\rmI)]_{h j}  
    \ge
    \frac{\pth{\sum_{t=t_0}^{t_0+\tP-1}\indc{i\in\calA^t}}\pth{\sum_{t=t_0}^{t_0+\tP-1}\indc{j\in\calA^t}}}{\pth{\prod_{t^\prime=t_0}^{t_0 + \tP - 1} \pth{\abth{\calA^{t^\prime}} + k}}^2}
\end{align*}

\paragraph{Expected element-wise lower bound of $\pth{\prodmatrix{\Wt{t}}{t_0}}^2$.}
\newcommand{\uexpects}[2]{\expects{#1}{\bigcup_{t=#2}^{#2+\tP-1} \Wt{t}}}
\newcommand{\prodmatrix}[2]{\prod_{t = #2}^{#2 + \tP-1} #1}
Recall that $\calF^t$ defines the sigma algebra generated by the randomness up to round $t$, and that we assume independent availability. Note that we are interested only in the off-diagonal elements.
By the law of total expectation, it holds for $i \neq j$ that
\begin{itemize}[leftmargin=*]
    \item $i \in \calR$, $j \in \calR$.
    \begin{align*}
        &\expect{\qth{\pth{\prodmatrix{\Wt{t}}{t_0}}^2}_{ij} \Big| \calF^{t_0} } 
        \ge
        \expect{\frac{\pth{\sum_{t=t_0}^{t_0+\tP-1}\indc{i\in\calA^t}}
        \pth{\sum_{t=t_0}^{t_0+\tP-1}\indc{j\in\calA^t}}}{M^{2 \tP}} \Big| \calF^{t_0} } \\
        &=
        \frac{1}{M^{2 \tP}}
        \expect{
        \expect{\pth{\sum_{t=t_0}^{t_0+\tP-1}\indc{i\in\calA^t}}
        \pth{\sum_{t=t_0}^{t_0+\tP-1}\indc{j\in\calA^t}} \Big| \calF^{t_0 + \tP - 1}, \calF^{t_0}} \Big| \calF^{t_0} } \\
        &\overset{(a)}{=}
        \frac{1}{M^{2 \tP}}
        \expect{
        \expect{\pth{\sum_{t=t_0}^{t_0+ \tP-1}\indc{i\in\calA^t}}\Big| \calF^{t_0 + \tP - 1}}
        \expect{\pth{\sum_{t=t_0}^{t_0+\tP-1}\indc{j\in\calA^t}} \Big| \calF^{t_0 + \tP - 1}} \Big| \calF^{t_0} } \\
        &=
        \frac{1}{M^{2 \tP}}
        \expect{
        \pth{\sum_{t=t_0}^{t_0+ \tP-2}\indc{i\in\calA^t} + \prob{i\in\calA^{t_0 + \tP -1}}}
        \pth{\sum_{t=t_0}^{t_0+ \tP-2}\indc{j\in\calA^t} + \prob{j\in\calA^{t_0 + \tP -1}}} \Big| \calF^{t_0} } \\
        &=\cdots \\
        &=
        \frac{
        \pth{\sum_{t=t_0}^{t_0+ \tP-1}\prob{i\in\calA^t}}
        \pth{\sum_{t=t_0}^{t_0+ \tP-1}\prob{j\in\calA^t}}}{M^{2 \tP}}
        \ge
        \frac{\delta^2 {\tP}^2}{M^{2 \tP}}
        ,
    \end{align*}
    where equality $(a)$ holds because of independence.
    \item $i \in \calV$, $j \in \calR$.
    In this case, the node can move to an arbitrary virtual node ($j$ included) in (I).
    \begin{align*}
        &\expect{\qth{\pth{\prodmatrix{\Wt{t}}{t_0}}^2}_{ij} \Big| \calF^{t_0} } 
        \ge
        \frac{\tP}{M^{2 \tP}}
        \expect{\pth{\sum_{t=t_0}^{t_0+ \tP-1}\indc{i\in\calA^t}} 
        \Big| \calF^{t_0} } \\
        &=
        \frac{\tP}{M^{2 \tP}}
        \expect{
        \expect{\pth{\sum_{t=t_0}^{t_0+\tP-1}\indc{i\in\calA^t}}\Big| \calF^{t_0 + \tP - 1}} \Big| \calF^{t_0} } \\
        &=
        \frac{\tP \pth{\sum_{t=t_0}^{t_0+ \tP-1}\prob{i\in\calA^t}}}{M^{2 \tP}}
        \ge
        \frac{\delta {\tP}^2}{M^{2 \tP}}
        .
    \end{align*}
    \item $i \in \calR$, $j \in \calV$.
    Similar strategy as above.
    \begin{align*}
    \expect{\qth{\pth{\prodmatrix{\Wt{t}}{t_0}}^2}_{ij} \Big| \calF^{t_0}}
    &\ge
    \frac{\delta {\tP}^2}{M^{2 \tP}}
    .
    \end{align*}
    \item $i \in \calV$, $j \in \calV$.
    \begin{align*}
    \expect{\qth{\pth{\prodmatrix{\Wt{t}}{t_0}}^2}_{ij} \Big| \calF^{t_0}}
    &\ge
    \frac{{\tP}^2}{M^{2 \tP}}
    \ge
    \frac{\delta {\tP}^2}{M^{2 \tP}}
    .
    \end{align*}
\end{itemize}

\paragraph{Cheeger's inequality.}
For ease of presentation, define 
$W^{t_0, \tP} \triangleq \expect{\pth{\prod_{i = t_0}^{t_0 + \tP -1}{\Wt{i}}^2}_{ij}}$,
and $\rho = \lambda_2 (\calW^{t_0, \tP}).$
Suppose we have $\abth{\calS \cap \calV} = a$.
Accordingly, we have $\abth{\calS \cap \calR} = \abth{\calS} - a$,
$\abth{\calV \cap \calS^\prime} = k - a$ and
$\abth{\calR \cap \calS^\prime} = \abth{\calS^\prime} - (k - a)$.
Let 
\[
    f(a, \calS) \triangleq 
    \frac{\sum_{i\in\calS, j\notin \calS} \pi_i W^{t_0, \tP}_{ij}}{\sum_{i\in \calS} \pi_i}
    =
    \frac{\sum_{i\in\calS, j\notin \calS} W^{t_0, \tP}_{ij}}{\abth{\calS}}
    ,
\]
where the equality holds because $\pi_i = \frac{1}{m+k}$ for all $i \in \calR \cup \calV$.
It follows that
\begin{align*}
    \sum_{i \in \calS, j \notin \calS} W^{t_0, \tP}_{ij}&
    \ge
    \frac{{\tP}^2 \delta^2}{M^{2 \tP}} \pth{\abth{\calS} - a}\pth{\abth{\calS^\prime} - (k - a)}
    +
    \frac{{\tP}^2 \delta}{M^{2 \tP}} \pth{\abth{\calS}\abth{\calS^\prime} - (\abth{\calS} - a) (\abth{\calS^\prime} - (k-a))}.
\end{align*}
Next, we decompose the two coefficients, respectively.
\begin{align*}
    \abth{\calS}\abth{\calS^\prime} - (\abth{\calS} - a) (\abth{\calS^\prime} - (k-a))
    &=
    \abth{\calS} \pth{m + k - \abth{\calS}}
    -
    \pth{\abth{\calS} - a} \pth{m + k - \abth{\calS} - k + a} \\
    &=
    m \abth{\calS}
    +
    k \abth{\calS}
    - 
    \abth{\calS}^2
    - m \abth{\calS}
    +
    \abth{\calS}^2
    -
    a \abth{\calS}
    +
    a m
    -
    a \abth{\calS}
    +
    a^2 \\
    &=
    a^2
    +
    a \pth{m - 2 \abth{\calS}}
    +
    k \abth{\calS}
    .
\end{align*}
\begin{align*}
    \pth{\abth{\calS} - a}\pth{\abth{\calS^\prime} - (k - a)}
    &=
    \pth{\abth{\calS} - a}\pth{m + k - \abth{\calS} - (k - a)}\\
    &=
    m \abth{\calS}
    -
    am
    -
    \abth{\calS}^2
    - 
    a^2
    +
    2 a \abth{\calS} \\
    &=
    - a^2
    - a (m - 2 \abth{\calS})
    + 
    m \abth{\calS}
    -
    \abth{\calS}^2
    .
\end{align*}
Therefore, $f$ becomes:
\begin{align*}
    &\frac{\delta^2 \pth{\abth{\calS} - a}\pth{\abth{\calS^\prime} - (k - a)}
    +
    \delta \pth{\abth{\calS}\abth{\calS^\prime} - (\abth{\calS} - a) (\abth{\calS^\prime} - (k-a))}}{M^{2 {\tP}} \abth{\calS}} \\
    & =
    \frac{a^2 (\delta - \delta^2) 
    +
    a (m - 2 \abth{\calS}) (\delta - \delta^2)
    +
    k \abth{\calS} \delta {\tP}
    +
    \abth{\calS} 
    (m - \abth{\calS})
    \delta^2}{M^{2 {\tP}} \abth{\calS}}.
\end{align*}
It follows that,
\begin{itemize}[leftmargin=*]
    \item $ m \ge 2 \abth{\calS}$, \ie, $\abth{\calS} \le \frac{m}{2}$.
    Hence, we have 
    \begin{align*}
        \frac{f (a, \calS)}{k} &\ge
        \frac{f (0, \calS)}{k} 
        \ge
        \frac{(\frac{m}{2}\delta + k ) \delta {\tP}}{M^{2 {\tP}}} 
        .
    \end{align*}
    \item $ m < 2 \abth{\calS}$, \ie, $\abth{\calS} > \frac{m}{2}$. In this case, the axis of symmetry is on the right-hand side.
    Hence, it holds that
    \begin{align}
    \notag
        f & \ge
        \frac{a^2}{M^{2{\tP}}} 
        \pth{\frac{\delta - \delta^2}{\abth{\calS}}}
        +
        \frac{a}{M^{2{\tP}}} 
        \pth{\frac{m - 2 \abth{\calS}}{\abth{\calS}}}
        (\delta - \delta^2)
        +
        \frac{k \delta {\tP}
        +
        (m - \abth{\calS})
        \delta^2 {\tP}}{M^{2 {\tP}}} \\
    \label{eq: final exam}
        &\ge 
        \frac{k \delta
        +
        (m - \abth{\calS})
        \delta^2
        -
        \pth{\frac{\delta - \delta^2}{4}}
        \pth{2 \sqrt{\abth{\calS}} - \frac{m}{\sqrt{\abth{\calS}}}}^2}{M^{2 P}}
        .
    \end{align}
    It is easy to see that~\eqref{eq: final exam} is monotonic decreasing w.r.t. $\abth{\calS}$.
    By definition, we know that $\abth{\calS} \le \frac{m+k}{2}$ and plug this in, we have
    \begin{align*}
        f & \ge
        \frac{k \delta
        +
        (\frac{m - k}{2})
        \delta^2
        -
        \pth{\frac{\delta - \delta^2}{4}}
        \pth{
        4 
        \pth{\frac{m + k}{2}}
        +
        \frac{m^2}{\frac{m + k}{2}}
        -
        4 m
        } }{M^{2 {\tP}}}\\
        &=
        \frac{k \delta
        +
        (\frac{m - k}{2})
        \delta^2
        -
        \pth{\frac{\delta - \delta^2}{4}}
        \pth{
        \frac{m^2}{\frac{m + k}{2}}
        -
        2 ( m  - k )
        }}{M^{2 {\tP}}} \\
        &=
        \frac{\delta 
        \pth{
        k - \frac{k^2}{2 (m + k)}
        }
        +
        \delta^2 
        \pth{\frac{m - k}{2} + \frac{k^2}{2 (m + k)}}}{(m + k)^{2 {\tP}}} \\
        &=
        \frac{m^2 \delta^2 + (k^2 + 2 m k) \delta }{2 (m + k)^{2 {\tP} + 1}}
        .
    \end{align*}
\end{itemize}
By comparing the above two lower bounds, we conclude that
\begin{align*}
\Phi(M) 
&= \min_{\sum_{i\in\calS} \pi_i \le \frac{1}{2}} \frac{\sum_{i\in\calS, j\notin \calS} \pi_i W^{{\tP}, t_0}_{ij}}{\sum_{i\in \calS} \pi_i}
\ge 
\frac{m^2 \delta^2 + (k^2 + 2 m k) \delta}{2 (m + k)^{2 {\tP} + 1}}.
\end{align*}
From Cheeger's inequality, we know that
$
\frac{1 - \lambda_2}{2} \le \Phi(M) \le \sqrt{2 \pth{1-\lambda_2}}.
$
Thus,
\[
\rho(t) = \lambda_2 \le 1 - \frac{\Phi^2\pth{M}}{2} \le 1 - 
\frac{\qth{m^2 \delta + (k^2 + 2 m k)}^2}{8 (m + k)^{4 P + 2}}\delta^2.
\]

\paragraph{Special case where $P=1$.}
By using a similar argument as above and adapting the results in \citep{xiang2024efficient}, it holds that
\[
\rho(t) = \lambda_2 \le 1 - \frac{\Phi^2\pth{M}}{2} \le 1 - 
\frac{\qth{m^2 \delta + (k^2 + 2 m k)}^2}{8 (m + k)^2}\delta^2,
\]
which is decreasing in $k$.
The monotonicity can be seen by taking the partial derivative w.r.t. $k$.
Let 
\[
    h (k) \triangleq \frac{m^2 \delta + (2 m k + k^2)}{(m+k)^2}.
\]
It follows that
\begin{align*}
    \frac{\partial h(k)}{\partial k}
    &=
    \frac{
    (2 m + 2 k)(m+k)^2
    -
    (2 m + 2 k)(m^2 \delta + 2 m k + k^2)
    }{(m+k)^4} \\
    &=
    2
    \frac{
    (m+k)^2
    -
    (m^2 \delta + 2 m k + k^2)
    }{(m+k)^2} \\
    &=
    \frac{2 m^2 (1 - \delta)}{(m+k)^2}
    >0.
\end{align*}
\end{proof}

\section{Convergence Error of $\bar{\bz}^t$ (\prettyref{thm: z bar rate})}
In the sequel, 
we recall and assume the following learning rate conditions in~\eqref{eq: lr condition main text}:
\begin{align*}
   \eta_l \eta_g 
    \le
    \frac{\delta (1 - \rho_k) \sqrt{m}}{96 s L
    \sqrt{\tP 
    M
    \qth{(\tP - 1)^2 \delta^2 + 1}
    \pth{\beta^2 + 1}}
    }
    ~ \text{and}~
    \eta_l 
    \le
    \frac{\delta}{216 s L
    \sqrt{
    \qth{(\tP - 1)^2 \delta^2 + 1}
    (\beta^2 + 1)}}.
\end{align*}
Recall that $\delta_{\max} \triangleq \max_{i\in[m],t\in[T]} p_i^t$ and $\tF^\star \triangleq \min_{\x} \tF(\x)$.
\begin{proof}{Proof of~\prettyref{thm: z bar rate}}
Take expectation over all the randomness, 
plug in \prettyref{lmm: consensus z} and Proposition \ref{prop: client dis}.
By telescoping sum, it holds that
\begin{align}
\nonumber
&\frac{\expect{ \tF^{\star} - \tF(\bar{\bz}^0) }}{T}
\le
- \frac{\eta_l \eta_g s}{3} 
\frac{1}{T}
\sum_{t=0}^{T-1}
\expect{\norm{\nabla \tF(\bar{\bz}^t)}^2 } \\
\nonumber
&~~~+\frac{2 \eta_l^2 \eta_g^2 s L {\delta_{\max}} \sigma^2}{M^2 T} 
\sum_{t=0}^{T-1}
\sum_{i=1}^m
\sum_{p=-1}^{t-1}
\expect{\indc{\tau_i(t) = p}}
(t - p)^2
\\
\nonumber
&~~~+
\frac{17
\eta_g \eta_l^3 s^2 L^2 
\sigma^2
}{M T} 
\sum_{t=0}^{T-1}
\sum_{i=1}^m 
\sum_{p=-1}^{t-1} 
\expect{\indc{\tau_i(t) = p}}
(t - p)^2 \\
\label{eq: z telescope third to last without pseudo}
&~~~
+ 
\frac{4 \eta_l \eta_g s L^2 }{M T}
\sum_{t=0}^{T-1}
\sum_{i=1}^M
\expect{\norm{\x_i^t - \bz_i^t}^2} \\
\label{eq: z telescope second to last without pseudo}
&~~~
+
\frac{\eta_l \eta_g s L^2}{2 M T}
\sum_{t=0}^{T-1}
\sum_{i=1}^M
\expect{\norm{\bz_i^t - \bar{\bz}^t}^2 }\\
&~~~
\label{eq: z telescope last without pseudo}
+
\frac{65 \eta_g \eta_l^3 s^3 L^2}{M T}
\sum_{t=0}^{T-1}
\sum_{i=1}^m 
\sum_{p=-1}^{t-1}
\expect{\indc{\tau_i(t) = p}}
(t - p)^2
\expect{\norm{\nabla F_i(\x_i^{p+1})}^2}
.
\end{align}
Next, we bound~\eqref{eq: z telescope third to last without pseudo},~\eqref{eq: z telescope second to last without pseudo} and~\eqref{eq: z telescope last without pseudo}, respectively.
First, we show that
\begin{align}
\notag
    &\frac{1}{MT}
    \sum_{t=0}^{T-1}
    \sum_{i=1}^m
    \expect{\norm{\nabla F_i (\bz_i^t)}^2}
    \le
    \frac{3 m \zeta^2}{M}
    + 
    \frac{3 M \pth{\beta^2 + 1}}{m}
    \frac{1}{T}
    \sum_{t=0}^{T-1}
    \expect{\norm{\nabla \tF (\bar{\bz}^t)}^2} 
    +
    \frac{3L^2}{M T}
    \sum_{t=0}^{T-1}
    \sum_{i=1}^M
    \expect{\norm{\bz_i^t - \bar{\bz}^t}^2} \\\notag
    &\le 
    3
    \qth{1 + 
    \frac{24 \eta_l^2 \eta_g^2 s^2 \tP \pth{\variance} L^2}{(1 - \rp)^2}
    }
    \pth{\frac{m \zeta^2}{M}} \\
    \nonumber
    &\qquad +
    \frac{3 M \pth{\beta^2 + 1}}{m}
    \qth{1 + 
    \frac{24 \eta_l^2 \eta_g^2 s^2 \tP \pth{\variance} L^2}{(1 - \rp)^2}
    }
    \frac{1}{T}
    \sum_{t=0}^{T-1}
    \expect{\norm{\nabla \tF (\bar{\bz}^t)}^2} \\
    \label{eq: hetero consensus}
    &\qquad +
    \frac{96 \eta_l^2 \eta_g^2 s \tP \pth{\variance} L^2}{(1 - \rp)^2}
    \pth{\frac{m \sigma^2}{M}}
    ,
\end{align}
where the last inequality follows from~\prettyref{lmm: consensus z}.
\paragraph{Bounding~\eqref{eq: z telescope third to last without pseudo}.}
\begin{align*}
    & \frac{4 \eta_l \eta_g s L^2}{M T}
    \sum_{t=0}^{T-1}
    \sum_{i=1}^M
    \expect{\norm{\x_i^t - \bz_i^t}^2} 
    \le
    \frac{4 \eta_l^3 \eta_g^3 s^3 L^2 \pth{\variance}}{M T}
    \sum_{t=0}^{T-1}
    \sum_{i=1}^m
    \expect{\norm{\nabla F_i(\bz_i^t)}^2}  \\
    &\le
    12 \eta_l^3 \eta_g^3 s^3 L^2 \pth{\variance}
    \qth{1 + 
    \frac{24 \eta_l^2 \eta_g^2 s^2 \tP \pth{\variance} L^2}{(1 - \rp)^2}
    }
    \pth{\frac{m \zeta^2}{M}} \\
    &~~~+
    12 \eta_l^3 \eta_g^3 s^3 L^2 \pth{\variance}
    \pth{\beta^2+ 1}
    \qth{1 + 
    \frac{24 \eta_l^2 \eta_g^2 s^2 \tP \pth{\variance} L^2}{(1 - \rp)^2}
    }
    \pth{\frac{M}{m}} 
    \frac{1}{T}
    \sum_{t=0}^{T-1}
    \expect{\norm{\nabla \tF (\bar{\bz}^t)}^2} \\
    &~~~+
    \frac{384 \eta_l^5 \eta_g^5 s^4 \tP \pth{\variance}^2 L^4}{(1 - \rp)^2}
    \pth{\frac{m \sigma^2}{M}} \\
    &\le
    16 \eta_l^3 \eta_g^3 s^3 L^2 \pth{\variance}
    \pth{\frac{m \zeta^2}{M}}
    +
    4 \eta_l^3 \eta_g^3 s^3 \tP \pth{\variance} L^2
    \pth{\frac{m \sigma^2}{M}} \\
    &~~~+
    16 \eta_l^3 \eta_g^3 s^3 L^2 \pth{\variance}
    \pth{\beta^2 + 1}
    \pth{\frac{M}{m}}
    \frac{1}{T}
    \sum_{t=0}^{T-1}
    \expect{\norm{\nabla \tF (\bar{\bz}^t)}^2},
\end{align*}
where the last inequality holds due to 
\(
    \eta_l \eta_g \le (1 - \rp)/(10 s L \sqrt{\tP \cdot {\variance}}).
\)
\paragraph{Bounding~\eqref{eq: z telescope second to last without pseudo}.}
\begin{align*}
    &\frac{\eta_l \eta_g s L^2}{2 M T}
    \sum_{t=0}^{T-1}
    \sum_{i=1}^m
    \expect{\norm{\bz_i^t - \bar{\bz}^t}^2} \\
    &\le
    \frac{4 \eta_l^3 \eta_g^3 s^2 L^2 \tP \pth{\variance}
    }{(\rrp - \rp) (1 - \rrp)}
    \pth{\frac{m \sigma^2}{M}}  \\
    &~~~+
    \frac{3 \eta_l^3 \eta_g^3 s^3 L^2
    \tP \pth{\variance} (\beta^2 +1) }{(\rrp - \rp) (1 - \rrp)}
    \pth{\frac{M}{m}}
    \frac{1}{T}
    \sum_{t=0}^{T-1}
    \expect{
    \norm{\nabla \tF(\bar{\bz}^t)}^2
    } 
    +
    \frac{3 \eta_l^3 \eta_g^3 s^3 L^2
    \tP \pth{\variance}}{(\rrp - \rp) (1 - \rrp)}
    \pth{\frac{m \zeta^2}{M}}.
\end{align*}
\paragraph{Bounding~\eqref{eq: z telescope last without pseudo}.}
\begin{align*}
&
65 \eta_g \eta_l^3 s^3 L^2
\frac{1}{M T}
\sum_{t=0}^{T-1}
\sum_{i=1}^m 
\sum_{p=-1}^{t-1}
\expect{\indc{\tau_i(t) = p}}
(t - p)^2
\expect{
\norm{\nabla F_i(\x_i^{p+1})}^2} \\
&\le
\frac{
65 \eta_l^3 \eta_g s^3 L^2 \pth{\variance}
}{M T} 
\sum_{t=0}^{T-1}
\sum_{i=1}^m 
\expect{\norm{\nabla F_i(\x_i^{t})}^2} \\
&\le
260
\eta_l^3 \eta_g s^3 L^2 \pth{\variance}
\pth{\frac{m \zeta^2}{M}} \\
&~~~+
\pth{\frac{M}{m}}
\frac{260 \eta_l^3 \eta_g s^3 L^2 \pth{\variance} \pth{\beta^2 + 1}}{T}
\sum_{t=0}^{T-1}
\expect{\norm{\nabla \tF (\bar{\bz}^t)}^2}
+
65 \eta_g \eta_l^3 s^3 L^2
\pth{\frac{m \sigma^2}{M}}
,
\end{align*}
where the last inequality holds due to 
\(
    \eta_l \eta_g \le (1 - \rp)/(10 s L \sqrt{\tP \cdot {\variance}}).
\)
Putting~\eqref{eq: z telescope third to last without pseudo},~\eqref{eq: z telescope second to last without pseudo} and~\eqref{eq: z telescope last without pseudo} together and plugging them back into the telescoping sum, it holds that
\begin{align*}
&\frac{\expect{ \tF^{\star} - \tF(\bar{\bz}^0) }}{T} \\
&\le
- \pth{
\frac{\eta_l \eta_g s}{3}
-
16 \eta_l^3 \eta_g^3 s^3 L^2 \pth{\variance}
\pth{\frac{M}{m}}
\pth{\beta^2+ 1}
-
\frac{12 \eta_l^3 \eta_g^3 s^3 L^2
\pth{\frac{M}{m}}
\tP \pth{\variance} (\beta^2 +1) }{(1 - \rp)^2}
}
\frac{1}{T} \sum_{t=0}^{T-1}
\expect{\norm{\nabla \tF(\bar{\bz}^t)}^2} \\
&~~~- \pth{
-
260
\eta_l^3 \eta_g s^3 L^2 \pth{\variance}
\pth{\frac{M}{m}}
\pth{\beta^2 + 1}
}
\frac{1}{T} \sum_{t=0}^{T-1}
\expect{\norm{\nabla \tF(\bar{\bz}^t)}^2}
\\
&~~~+\frac{2 \eta_l^2 \eta_g^2 s L {\delta_{\max}} \pth{\variance} m \sigma^2}{M^2} 
+
17
\eta_g \eta_l^3 s^2 L^2
\pth{\variance}
\pth{\frac{m \sigma^2}{M}} \\
&~~~+ 
4 \eta_l^3 \eta_g^3 s^3 \tP \pth{\variance} L^2
\frac{m \sigma^2}{M}
+
\frac{16 \eta_l^3 \eta_g^3 s^2 L^2 \tP \pth{\variance}
}{(1 - \rp)^2}
\frac{m \sigma^2}{M}
+
65 \eta_g \eta_l^3 s^3 L^2 
\frac{m \sigma^2}{M}
\\
&~~~+ 
16 \eta_l^3 \eta_g^3 s^3 L^2 \pth{\variance}
\pth{\frac{m \zeta^2}{M}} 
+
\frac{12 \eta_l^3 \eta_g^3 s^3 L^2
\tP \pth{\variance}}{(1 - \rp)^2}
\pth{\frac{m \zeta^2}{M}}
+
260
\eta_l^3 \eta_g s^3 L^2 \pth{\variance}
\pth{\frac{m \zeta^2}{M}} \\
&\le
- 
\frac{\eta_l \eta_g s}{4}
\frac{1}{T} \sum_{t=0}^{T-1}
\expect{\norm{\nabla \tF(\bar{\bz}^t)}^2} \\
&~~~+
\frac{2 \eta_l^2 \eta_g^2 s L {\delta_{\max}} \pth{\variance} m \sigma^2}{M^2} 
+
17
\eta_g \eta_l^3 s^2 L^2
\pth{\variance}
\pth{\frac{m \sigma^2}{M}} \\
&~~~+ 
4 \eta_l^3 \eta_g^3 s^3 \tP \pth{\variance} L^2
\pth{\frac{m \sigma^2}{M}}
+
\frac{16 \eta_l^3 \eta_g^3 s^2 L^2 \tP \pth{\variance}
}{(1 - \rp)^2}
\pth{\frac{m \sigma^2}{M}}
+
65 \eta_g \eta_l^3 s^3 L^2 
\pth{\frac{m \sigma^2}{M}}
\\
&~~~+ 
16 \eta_l^3 \eta_g^3 s^3 L^2 \pth{\variance}
\pth{\frac{m \zeta^2}{M}}
+
\frac{12 \eta_l^3 \eta_g^3 s^3 L^2
\tP \pth{\variance}}{(1 - \rp)^2}
\pth{\frac{m \zeta^2}{M}}
+
260
\eta_l^3 \eta_g s^3 L^2 \pth{\variance}
\pth{\frac{m \zeta^2}{M}}
,
\end{align*}
where the last inequality holds because
\begin{align*}
    \eta_l \eta_g 
    \le
    \frac{(1 - \rp) \sqrt{m}}{48 s L \sqrt{{\tP M \pth{\variance} \pth{\beta^2 + 1}}}}
    ~ \text{and}~
    \eta_l 
    \le
    \frac{1}{108 s L \sqrt{\pth{\variance} (\beta^2 + 1)}}
    .
\end{align*}
Combining the above and rearranging the terms,
it holds that
\begin{align*}    
&\frac{1}{T} \sum_{t=0}^{T-1}
\expect{\norm{\nabla \tF(\bar{\bz}^t)}^2} 
\le
\frac{4\pth{\tF(\bar{\bz}^0) - \tF^\star}}{\eta_l \eta_g s T} \\
&~~~+\frac{8 \eta_l \eta_g L {\delta_{\max}} \pth{\variance} m \sigma^2}{M^2} 
+
68
\eta_l^2 s L^2
\pth{\variance}
\pth{\frac{m \sigma^2}{M}} \\
&~~~+ 
16 \eta_l^2 \eta_g^2 s^2 \tP \pth{\variance} L^2
\pth{\frac{m \sigma^2}{M}}
+
\frac{64 \eta_l^2 \eta_g^2 s L^2 \tP \pth{\variance}
}{(1 - \rp)^2}
\pth{\frac{m \sigma^2}{M}}
+
260 \eta_l^2 s^2 L^2 \pth{\variance} 
\pth{\frac{m \sigma^2}{M}}
\\
&~~~+ 
64 \eta_l^2 \eta_g^2 s^2 L^2 \pth{\variance}
\pth{\frac{m \zeta^2}{M}}
+
\frac{48 \eta_l^2 \eta_g^2 s^2 L^2
\tP \pth{\variance}}{(1 - \rp)^2}
\pth{\frac{m \zeta^2}{M}}
+
1040
\eta_l^2 s^2 L^2 \pth{\variance}
\pth{\frac{m \zeta^2}{M}}
.
\end{align*}  
Expanding the second moment of the unavailable duration, it holds that
\begin{align*}    
&\frac{1}{T} \sum_{t=0}^{T-1}
\expect{\norm{\nabla \tF(\bar{\bz}^t)}^2} 
\le
\frac{4\pth{\tF(\bar{\bz}^0) - \tF^\star}}{\eta_l \eta_g s T} \\
&~~~+
\frac{\qth{(P_{\delta}-1) \delta + 1}^2 + \qth{(P_{\delta} - 1) \delta^2 + 1}}{\delta^2}
\pth{
\frac{8 \eta_l \eta_g L {\delta_{\max}}}{M} 
+
68
\eta_l^2 s L^2
}
\pth{\frac{m \sigma^2}{M}} \\
&~~~+ 
\frac{\qth{(\tP-1) \delta + 1}^2 + \qth{(\tP - 1) \delta^2 + 1}}{\delta^2}
\pth{
16 \eta_l^2 \eta_g^2 s^2 \tP L^2
+
\frac{64 \eta_l^2 \eta_g^2 s L^2 \tP 
}{(1 - \rp)^2}
+
260 \eta_l^2 s^2 L^2 }
\pth{\frac{m \sigma^2}{M}}
\\
&~~~+ 
\frac{\qth{(\tP-1) \delta + 1}^2 + \qth{(\tP - 1) \delta^2 + 1}}{\delta^2}
\pth{
64 \eta_l^2 \eta_g^2 s^2 L^2 
+
\frac{48 \eta_l^2 \eta_g^2 s^2 L^2
\tP }{(1 - \rp)^2}
+
1040
\eta_l^2 s^2 L^2 
}
\pth{\frac{m \zeta^2}{M}}
.
\end{align*}  
It holds that
\begin{align}
\notag
\frac{\qth{(\tP-1) \delta + 1}^2 + \qth{(\tP - 1) \delta^2 + 1}}{\delta^2}
&\le
\frac{2 \pth{(\tP-1)^2 \delta^2 + 1 } + (\tP - 1)^2 \delta^2 + 1}{\delta^2} \\
\label{eq: variance simplification}
&=
3 \pth{
\pth{\tP - 1}^2
+
\frac{1}{\delta^2}
}.
\end{align}
Grouping the terms of the same order in terms of asymptotic, we have
\begin{align*}
    \frac{1}{T}\sum_{t=0}^{T-1}\expect{\norm{\nabla \tF(\bar{\bz}^t)}^2}
    &\lesssim
    \frac{\pth{\tF(\bar{\bz}^0) - \tF^\star}}{\eta_l \eta_g s T}
    +\frac{\delta_{\max} \eta_l \eta_g L m \sigma^2}{M^2}
    \qth{
    \pth{\tP - 1}^2
    +
    \frac{1}{\delta^2}
    }
    \\
    &~~~\qquad +
    \eta_l^2 \eta_g^2 s^2 L^2 \tP
    \pth{\frac{m}{M}}
    \pth{\sigma^2 +  \zeta^2}
    \qth{
    \pth{\tP - 1}^2
    +
    \frac{1}{\delta^2}
    }
    \pth{
    1
    +
    \frac{1}{(1 - \rp)^2}
    },
\end{align*}
where we use the convention that $\eta_g \ge 1$ for ease of presentation.
Using the fact that $\nabla \tF (\x) = (\frac{m}{M}) \nabla F (\x)$, it holds that
\begin{align*}
    \frac{1}{T}\sum_{t=0}^{T-1}\expect{\norm{\nabla F(\bar{\bz}^t)}^2}
    &\lesssim
    \pth{\frac{M}{m}}
    \frac{\pth{F(\bar{\bz}^0) - F^\star}}{\eta_l \eta_g s T}
    +\frac{\delta_{\max} \eta_l \eta_g L \sigma^2}{M}
    \pth{\frac{M}{m}}
    \qth{
    \pth{\tP - 1}^2
    +
    \frac{1}{\delta^2}
    }
    \\
    &~~~\qquad +
    \eta_l^2 \eta_g^2 s^2 L^2 \tP
    \pth{\frac{M}{m}}
    \pth{\sigma^2 +  \zeta^2}
    \qth{
    \pth{\tP - 1}^2
    +
    \frac{1}{\delta^2}
    }
    \pth{
    1
    +
    \frac{1}{(1 - \rp)^2}
    } \\
    &\lesssim
    \pth{\frac{M}{m}}
    \frac{\pth{F(\bar{\bz}^0) - F^\star}}{\eta_l \eta_g s T}
    +\frac{\delta_{\max} \eta_l \eta_g L \sigma^2}{m}
    \qth{
    \pth{\tP - 1}^2
    +
    \frac{1}{\delta^2}
    }
    \\
    &~~~\qquad +
    \eta_l^2 \eta_g^2 s^2 L^2 \tP
    \pth{\frac{M}{m}}
    \pth{\sigma^2 +  \zeta^2}
    \qth{
    \pth{\tP - 1}^2
    +
    \frac{1}{\delta^2}
    }
    \qth{
    1
    +
    \frac{1}{(1 - \rp)^2}
    }
    ,
\end{align*}
\end{proof}

\section{Convergence Rate of $\bar{\x}^t$ (\prettyref{cor: x bar rate})}

\subsection{Convergence error of~\prettyref{alg: fedpbc+}}
\begin{corollary}[Convergence error of $\x_i^t$]
\label{cor: x bar without pseudo}
Suppose learning rates conditions in~\eqref{eq: lr condition main text} are met for $\eta_l$ and $\eta_g$,
and Assumptions \ref{ass: prob lower bound}, \ref{ass: 2 smmothness}, \ref{ass: bounded variance client-wise} and \ref{ass: bounded similarity} hold
for $T \ge 1$,
it holds that
\begin{align*}
    \frac{1}{T}\sum_{t=0}^{T-1}\expect{\norm{\nabla F(\bar{\x}^t)}^2}
    &\lesssim
    \pth{\frac{M}{m}}
    \frac{\pth{F(\bar{\bz}^0) - F^\star}}{\eta_l \eta_g s T}
    +
    \frac{\delta_{\max} \eta_l \eta_g L \sigma^2}{m}
    \qth{
    \pth{P - 1}^2
    +
    \frac{1}{\delta^2}
    }
    \\
    &~~~\qquad +
    \eta_l^2 \eta_g^2 s^2 L^2 \tP
    \pth{\frac{M}{m}}
    \pth{\sigma^2 +  \zeta^2}
    \qth{
    \pth{P - 1}^2
    +
    \frac{1}{\delta^2}
    }
    \qth{
    1
    +
    \frac{1}{(1 - \rp)^2}
    },
\end{align*}
\end{corollary}
\begin{proof}{\bf Proof of \prettyref{cor: x bar without pseudo}}
\begin{align*}
&\frac{1}{T}\sum_{t=0}^{T-1}\expect{\norm{\nabla F(\bar{\x}^t)}^2}\le
\frac{3}{T}\sum_{t=0}^{T-1}\expect{\norm{\nabla F(\bar{\x}^t) - \nabla F(\bar{\bz}^t)}^2}
+
\frac{3}{2T}\sum_{t=0}^{T-1} \expect{\norm{\nabla F(\bar{\bz}^t)}^2} \\
&\overset{(a)}{\le}
\frac{3 L^2}{T}\sum_{t=0}^{T-1} \expect{\norm{\bar{\x}^t - \bar{\bz}^t}^2}
+
\frac{3}{2T}\sum_{t=0}^{T-1} \expect{\norm{\nabla F(\bar{\bz}^t)}^2} \\
&\overset{(b)}{\le}
\frac{3 L^2}{T}
\sum_{t=0}^{T-1} 
\frac{1}{M} 
\sum_{i=1}^m
\expect{\norm{\x_i^t - \bz_i^t}^2}
+
\frac{3}{2T}\sum_{t=0}^{T-1} \expect{\norm{\nabla F(\bar{\bz}^t)}^2} \\
&\le
3
\pth{\variance}
\frac{\eta_l^2 \eta_g^2 s^2 L^2}{T}\sum_{t=0}^{T-1} \frac{1}{M} \sum_{i=1}^m
\expect{\norm{\nabla F_i(\bz_i^t)}^2}
+
\frac{3}{2T}\sum_{t=0}^{T-1} \expect{\norm{\nabla F(\bar{\bz}^t)}^2} 
,
\end{align*}
where inequality $(a)$ follows from~\prettyref{app: preliminaries} 2, 
inequality $(b)$ follows from Assumption \ref{ass: 2 smmothness}.
Further plug in~\eqref{eq: hetero consensus} and use the learning rate conditions, it holds that
\begin{align*}
    \frac{1}{T}\sum_{t=0}^{T-1}\expect{\norm{\nabla F(\bar{\x}^t)}^2}
    &
    \le
    \frac{2}{T}
    \sum_{t=0}^{T-1}
    \expect{\norm{\nabla F(\bar{\bz}^t)}^2}  
    +
    12
    \pth{\variance}
    \eta_l^2 \eta_g^2 s^2 L^2
    \pth{\frac{M \zeta^2}{m} }
    +
    3
    \pth{\variance}
    \eta_l^2 \eta_g^2 s^2 L^2
    \pth{\frac{M \sigma^2}{m}}
    . 
\end{align*}
Grouping the terms of the same order in terms of asymptotic, we have
\begin{align*}
    \frac{1}{T}\sum_{t=0}^{T-1}\expect{\norm{\nabla F(\bar{\x}^t)}^2}
    &\lesssim
    \pth{\frac{M}{m}}
    \frac{\pth{F(\bar{\bz}^0) - F^\star}}{\eta_l \eta_g s T}
    +\frac{\delta_{\max} \eta_l \eta_g L \sigma^2}{m}
    \qth{
    \pth{P - 1}^2
    +
    \frac{1}{\delta^2}
    }
    \\
    &~~~\qquad +
    \eta_l^2 \eta_g^2 s^2 L^2 \tP
    \pth{\frac{M}{m}}
    \pth{\sigma^2 +  \zeta^2}
    \qth{
    \pth{P - 1}^2
    +
    \frac{1}{\delta^2}
    }
    \qth{
    1
    +
    \frac{1}{(1 - \rp)^2}
    },
\end{align*}
where we use the convention that $\eta_g \ge 1$ for ease of presentation.
\end{proof}

\subsection{Convergence rate of~\prettyref{alg: fedpbc+}}
\begin{proof}{\bf Proof of~\prettyref{cor: x bar rate}}
Choose step-size as
$\eta_l = \frac{1}{\sqrt{T} s L}$,
$\eta_g = \sqrt{s \delta m}$
such that learning rate conditions in~\eqref{eq: lr condition main text} are met,
it holds that
\begin{align*}
    \frac{1}{T}\sum_{t=0}^{T-1}\expect{\norm{\nabla F(\bar{\x}^t)}^2}
    &\lesssim
    \pth{\frac{M}{m}}
    \frac{L \pth{F(\bar{\x}^0) - F^\star}}{\sqrt{s \delta m T}}
    +
    \frac{\delta_{\max}}{\delta^{\frac{3}{2}} \sqrt{s m T}}
    \qth{\pth{\tP - 1}^2 \delta^2 + 1}
    \sigma^2
    \\
    &~~~\qquad +
    \pth{\frac{M}{m}}
    \frac{s m \tP}{T}
    \pth{\sigma^2 +  \zeta^2}
    \qth{\pth{\tP - 1}^2 \delta^2 + 1}
    \qth{1 + \frac{1}{(1 - \rp)^2}}.
\end{align*}

\end{proof}

\section{Additional Results and Interpretations}

\subsection{Consensus error of~\prettyref{alg: fedpbc+}}
\begin{corollary}[Consensus error of $\x_i^t$]
\label{cor: x consensus without pseudo}
Suppose learning rates conditions are met in~\eqref{eq: lr condition main text} for $\eta_l$ and $\eta_g$,
and Assumptions \ref{ass: prob lower bound}, \ref{ass: 2 smmothness}, \ref{ass: bounded variance client-wise} and \ref{ass: bounded similarity} hold
for $T \ge 1$,
it holds that
\begin{align*}
\frac{1}{T} \sum_{t=0}^{T-1}
\frac{1}{M} \sum_{i=1}^M 
\expect{\norm{\x_i^t - \bar{\x}^t}^2}
&\lesssim
\pth{\frac{M}{m}}
\frac{\pth{F(\bar{\bz}^0) - F^\star}}{\eta_l \eta_g s T}
+\frac{\delta_{\max} \eta_l \eta_g L \sigma^2}{m}
\qth{
\pth{\tP - 1}^2
+
\frac{1}{\delta^2}
}
\\
&~~~\qquad +
\pth{\frac{M}{m}}
\eta_l^2 \eta_g^2 s^2 L^2 \tP
\pth{\sigma^2 +  \zeta^2}
\qth{\pth{\tP - 1}^2 + \frac{1}{\delta^2}}
\qth{1 + \frac{1}{(1 - \rp)^2}},
\end{align*}
\end{corollary}
\begin{proof}{\bf Proof of \prettyref{cor: x consensus without pseudo}}
\begin{align*}
&\frac{1}{T} \sum_{t=0}^{T-1}
\frac{1}{M} \sum_{i=1}^M 
\norm{\x_i^t - \bar{\x}^t}^2
=
\frac{1}{T} \sum_{t=0}^{T-1}
\frac{1}{M} \sum_{i=1}^M 
\norm{\x_i^t - \bz_i^t + \bz_i^t - \bar{\bz}^t + \bar{\bz}^t - \bar{\x}^t}^2 \\
&\overset{(a)}{\le}
\frac{1}{T} \sum_{t=0}^{T-1}
\frac{3}{M} \sum_{i=1}^M 
\norm{\x_i^t - \bz_i^t}^2
+
\frac{1}{T} \sum_{t=0}^{T-1}
\frac{3}{M} \sum_{i=1}^M 
\norm{\bz_i^t - \bar{\bz}^t}^2
+
\frac{1}{T} \sum_{t=0}^{T-1}
3 \norm{\bar{\bz}^t - \bar{\x}^t}^2 \\
&\overset{(b)}{\le}
\frac{1}{T} \sum_{t=0}^{T-1}
\frac{3}{M} \sum_{i=1}^M 
\norm{\x_i^t - \bz_i^t}^2
+
\frac{1}{T} \sum_{t=0}^{T-1}
\frac{3}{M} \sum_{i=1}^M 
\norm{\bz_i^t - \bar{\bz}^t}^2
+
\frac{1}{T} \sum_{t=0}^{T-1}
\frac{3}{M} \sum_{i=1}^M 
\norm{\bz_i^t - \x_i^t}^2 \\
&=
\frac{1}{T} \sum_{t=0}^{T-1}
\frac{6}{M} \sum_{i=1}^M 
\norm{\x_i^t - \bz_i^t}^2
+
\frac{1}{T} \sum_{t=0}^{T-1}
\frac{3}{M} \sum_{i=1}^M 
\norm{\bz_i^t - \bar{\bz}^t}^2
,
\end{align*}
where inequalities $(a)$ and $(b)$ follow from Jensen's inequality.
It holds that
\begin{align*}
\frac{1}{M T} \sum_{t=0}^{T-1}
\sum_{i=1}^M 
\expect{\norm{\x_i^t - \bar{\x}^t}^2}
&\le
\frac{4}{T}
\sum_{t=0}^{T-1}
\expect{\norm{\nabla F(\bar{\bz}^t)}^2} 
\\
&~~~+
\frac{96 \eta_l^2 \eta_g^2 s \tP \pth{\variance}
}{(1 - \rp)^2}
\pth{\frac{m \sigma^2}{M}}
+
6
\pth{\variance}
\eta_l^2 \eta_g^2 s^2 L^2
\pth{\frac{m \sigma^2}{M}}
\\
&~~~+
\frac{72 \eta_l^2 \eta_g^2 s^2
\tP \pth{\variance}}{(1 - \rp)^2}
\pth{\frac{m \zeta^2}{M}}
+
72
\pth{\variance}
\eta_l^2 \eta_g^2 s^2 L^2
\pth{\frac{m \zeta^2}{M}}.
\end{align*}
where the last inequality holds because of learning rate condition in~\eqref{eq: lr condition main text}.
Group the terms of the same order in terms of asymptotics,
we have
\begin{align*}
    \frac{1}{T} \sum_{t=0}^{T-1}
    \frac{1}{M} \sum_{i=1}^M 
    \expect{\norm{\x_i^t - \bar{\x}^t}^2}
    &\lesssim
    \pth{\frac{M}{m}}
    \frac{\pth{F(\bar{\bz}^0) - F^\star}}{\eta_l \eta_g s T}
    +\frac{\delta_{\max} \eta_l \eta_g L \sigma^2}{m}
    \qth{
    \pth{\tP - 1}^2
    +
    \frac{1}{\delta^2}
    }
    \\
    &~~~ +
    \eta_l^2 \eta_g^2 s^2 L^2 \tP
    \pth{\frac{M}{m}}
    \pth{\sigma^2 +  \zeta^2}
    \qth{
    \pth{\tP - 1}^2
    +
    \frac{1}{\delta^2}
    }
    \qth{
    1
    +
    \frac{1}{(1 - \rp)^2}
    },
\end{align*}
where we use the convention that $\eta_g \ge 1$ for ease of presentation.
\end{proof}

\subsection{Orders of the asymptotic rates}
From~\prettyref{thm: z bar rate}, 
\prettyref{cor: x bar without pseudo},
\prettyref{cor: x consensus without pseudo},
it is easy to see from the theorem statements that they are all of the same asymptotic order, \ie,
\[
    \frac{1}{T}\sum_{t=0}^{T-1}\mathbb{E}[\| \nabla F (\bar{\x}^t)\|_2^2]
    \asymp
    \frac{1}{T}\sum_{t=0}^{T-1}
    \frac{1}{M} \sum_{i=1}^M
    \Expect[\| \x_i^t - \bar{\x}^t\|_2^2]
    \asymp
    \frac{1}{T}\sum_{t=0}^{T-1}
    \Expect
    [\| \nabla F (\bar{\bz}^t)\|_2^2].
\]
In addition,
we can also see that
\[
    \frac{1}{T}\sum_{t=0}^{T-1}
    \frac{1}{M} \sum_{i=1}^M
    \Expect[\| \x_i^t - \bz_i^t\|_2^2]
    \asymp
    \frac{1}{T}\sum_{t=0}^{T-1}
    \frac{1}{M} \sum_{i=1}^M
    \Expect[\| \bz_i^t - \bar{\bz}^t\|_2^2]
    \asymp
    \frac{1}{T}
    \sum_{t=0}^{T-1}
    \Expect
    [\| \nabla F (\bar{\bz}^t)\|_2^2].
\]
Therefore,
we conclude that~\eqref{eq: approx and consensus for z},~\eqref{eq: x bar consensus} and~\eqref{eq: x bar relation} hold.

\section{Numerical Experiments}
\label{app: numerical}
\subsection{Experimental setups}
\label{app: numerical setup}
\noindent{\bf \color{white} Hardware and Software Setups.}
\begin{wrapfigure}[15]{r}{0.5\textwidth} 
\vspace*{-4\baselineskip}
\centering
\resizebox{.8\linewidth}{!}{
\includegraphics[width=\linewidth, trim=0.1cm 0.1cm 0 0.1cm, clip]{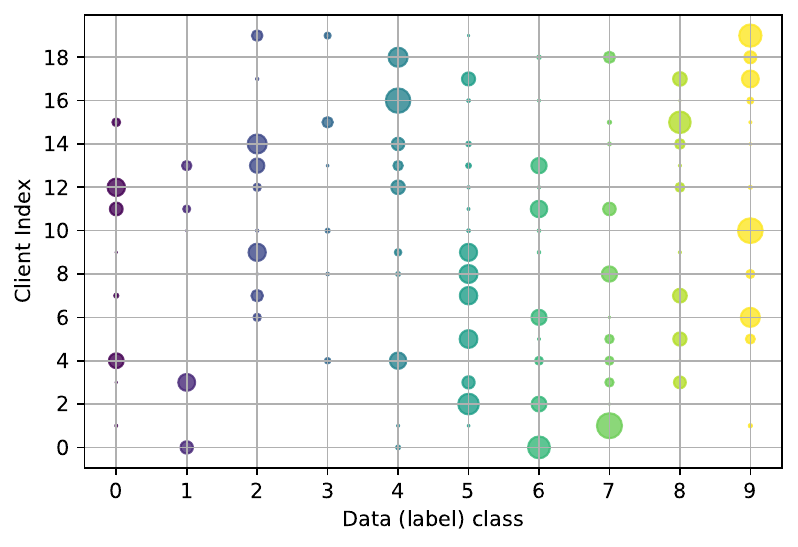}}
\caption{
\footnotesize
An example of data heterogeneity using $\mathsf{Dirichlet}(\alpha=0.1)$ distribution with $20$ clients.
$x$-axis denotes the categories of images, 
while $y$-axis denotes the client index.
The size of a circle refers to the proportion of pictures in a given class.
The color of a circle distinguishes images with different categories.
}
\centering
\label{fig: noniid 20 cifar10}
\end{wrapfigure}
\vspace*{-2\baselineskip}
\begin{itemize}[leftmargin=*]
\item{\bf Hardware.}
The simulations are performed on a private cluster with 64 CPUs, 500 GB RAM and 8 NVIDIA A5000 GPU cards.
\item {\bf Software.}
We code the experiments based on PyTorch 1.13.1~\citep{paszke2019pytorch} and Python 3.7.16.
\end{itemize}

\vspace{\baselineskip}
\noindent{\bf Neural Network and Hyperparameter Specifications.}
\label{app: hyperparameter}
Table~\ref{tbl: cnn structures} specifies details of the structures of the convolutional neural network and training.
We initialize CNNs using the Kaiming initialization.
The initial local learning rate $\eta_0$ and the global learning rate $\eta_g$ are searched, 
based on the best performance after $500$ global rounds, 
over two grids $\sth{0.1, 0.05, 0.01, 0.005, 0.001, 0.0005}$ and $\sth{0.5,1,1.5,5,10,50}$, respectively.
The results are presented in Table~\ref{tbl: learning rate}.

\vspace{\baselineskip}
\noindent{\bf Baseline Algorithm Details.}

The difference between~\FedAvg~over active clients and~\FedAvg~over all clients
is that the latter counts the contributions of unavailable clients as $\bm{0}$'s.
We set $\beta = 0.001$ for~\FAST~\citep{ribero2022federated}, 
which is tuned over a grid of $\sth{0.1, 0.05, 0.01, 0.005, 0.001, 0.0005}$.
The amplification factor of {\tt gFedAvg} in~\citep{wang2022} is set as $10$, and the period is set as $100$. 
In addition, as recommended by~\citep{wang2023lightweight}, we choose $K=50$ in~\FedAU~without further specification.
\prettyref{fig: motivating example non-stationary} adopts the same hyperparameter setups as the ones in~\prettyref{sec: numerical P=1}, yet with only $1000$ training rounds.
\begin{table}[!t]
\caption{
\footnotesize
Neural network architecture, loss function, learning rate scheduling, training steps and batch size specifications}
\label{tbl: cnn structures}
\resizebox{\linewidth}{!}{
\begin{tabular}{cccc}
\toprule
{\bf Data sets}& 
{\bf SVHN} & 
{\bf CIFAR-10} & 
{\bf CINIC-10} \\
\toprule
Neural network &
CNN &
CNN &
CNN \\
Model architecture$^*$ & 
\begin{tabular}{p{.18\textwidth}}
\centering
{\bf C}(3,32)
-- {\bf R}
-- {\bf M}
-- {\bf C}(32,32)
-- {\bf R}
-- {\bf M}
-- {\bf L}(128)
-- {\bf R}
-- {\bf L}(10)
\end{tabular}
&
\begin{tabular}{p{.18\textwidth}}
\centering
{\bf C}(3,32)
-- {\bf R}
-- {\bf M}
-- {\bf C}(32,32)
-- {\bf R}
-- {\bf M}
-- {\bf L}(256)
-- {\bf R}
-- {\bf L}(64)
-- {\bf R}
-- {\bf L}(10)
\end{tabular}
&
\begin{tabular}{p{.18\textwidth}}
\centering
{\bf C}(3,32)
-- {\bf R}
-- {\bf M}
-- {\bf C}(32,32)
-- {\bf R}
-- {\bf M}
-- {\bf D}
-- {\bf L}(512)
-- {\bf R}
-- {\bf D}
-- {\bf L}(256)
-- {\bf R}
-- {\bf D}
-- {\bf L}(10)
\end{tabular} \\
\midrule
Loss function &
\multicolumn{3}{c}{Cross-entropy loss} \\
\addlinespace[1ex]
\begin{tabular}{p{.25\textwidth}}
\centering
Local learning rate $\eta_l$ \\ scheduling 
\end{tabular}&
\multicolumn{3}{c}{ 
\begin{tabular}{l}
$\eta_l = \eta_0$ in~\prettyref{sec: numerical P>1}; 
$\eta_l = \frac{\eta_0}{\sqrt{t/10 + 1}}$ in~\prettyref{sec: numerical P=1}, \\
where $t$ denotes the global round.
\end{tabular}
}
\\
\addlinespace[1ex]
Number of local steps $s$ &
\multicolumn{3}{c}{10} \\
\addlinespace[1ex]
Number of global rounds $T$ in~\prettyref{sec: numerical P>1} &
10000 & 20000 & --
\\
\addlinespace[1ex]
Number of global rounds $T$ in~\prettyref{sec: numerical P=1} &
2000 & 2000 & 2000
\\
\midrule
Batch size &
\multicolumn{3}{c}{128} \\
\bottomrule
\end{tabular}}
\vskip.2\baselineskip
\begin{tabular}{p{.95\textwidth}}
$^*$
\begin{footnotesize}    
{\bf C}(\# in-channel, \# out-channel): a 2D convolution layer (kernel size 3, stride 1, padding 1);
{\bf R}: ReLU activation function;
{\bf M}: a 2D max-pool layer (kernel size 2, stride 2);
{\bf L}: (\# outputs): a fully-connected linear layer;
{\bf D}: a dropout layer (probability 0.2).
\end{footnotesize}
\end{tabular}
\end{table}
\begin{table}[!t]
\caption{Initial learning rate $\eta_0$ and global learning rate $\eta_g$}
\centering
\label{tbl: learning rate}
\resizebox{\linewidth}{!}{
\begin{tabular}{ccccccccccccccc}
\toprule
\multicolumn{1}{c}{\bf Algorithms} &
\multicolumn{2}{c}{\begin{tabular}{c} \FedAvg \\ {\em active} \end{tabular}}&
\multicolumn{2}{c}{\begin{tabular}{c} \FedAvg \\ {\em known} \end{tabular}}&
\multicolumn{2}{c}{\begin{tabular}{c} \FedAvg \\ {\em all} \end{tabular}}&
\multicolumn{2}{c}{\FedAU}&
\multicolumn{2}{c}{\FAST}&
\multicolumn{2}{c}{\FedAPM}&
\multicolumn{2}{c}{\MIFA} \\
\midrule
\multirow{2}{*}{SVHN}&
$\eta_0$&
$\eta_g$&
$\eta_0$&
$\eta_g$&
$\eta_0$&
$\eta_g$&
$\eta_0$&
$\eta_g$&
$\eta_0$&
$\eta_g$&
$\eta_0$&
$\eta_g$&
$\eta_0$&
$\eta_g$ \\[2pt]
&
0.05 &
1.0 &
0.1 &
1.0 &
0.05 &
1.0 &
0.05 &
1.0 &
0.05 &
1.0 &
0.1 &
1.0 &
0.05 &
1.0 \\
\midrule
\multirow{2}{*}{CIFAR-10}&
$\eta_0$&
$\eta_g$&
$\eta_0$&
$\eta_g$&
$\eta_0$&
$\eta_g$&
$\eta_0$&
$\eta_g$&
$\eta_0$&
$\eta_g$&
$\eta_0$&
$\eta_g$&
$\eta_0$&
$\eta_g$ \\[2pt]
&
0.05 &
1.0 &
0.1 &
1.0 &
0.05 &
1.0 &
0.05 &
1.0 &
0.05 &
1.0 &
0.1 &
1.0 &
0.05 &
1.0 \\
\midrule
\multirow{2}{*}{CINIC-10}&
$\eta_0$&
$\eta_g$&
$\eta_0$&
$\eta_g$&
$\eta_0$&
$\eta_g$&
$\eta_0$&
$\eta_g$&
$\eta_0$&
$\eta_g$&
$\eta_0$&
$\eta_g$&
$\eta_0$&
$\eta_g$ \\[2pt]
&
0.05 &
1.0 &
0.1 &
1.0 &
0.05 &
1.0 &
0.05 &
1.0 &
0.05 &
1.0 &
0.1 &
1.0 &
0.05 &
1.0 \\
\bottomrule
\end{tabular}}
\end{table}

\vspace{.5\baselineskip}
\noindent{\bf Data sets and Data Heterogeneity.}

\noindent\textit{Data sets.}
All the data sets we evaluate contain 10 classes of images.
Some data enhancement tricks that are standard in training image classifiers are applied during training.
Specifically,
we apply random cropping and gradient clipping with a max norm of 0.5 to all data set trainings.
Furthermore, random horizontal flipping is applied to CIFAR-10 and CINIC-10.

One full set of experiments in~\prettyref{sec: numerical P=1}
takes about 6 hours on SVHN and CIFAR-10 data sets, 
while about 10 hours on CINIC-10 data set.
The training time for experiments in~\prettyref{sec: numerical P>1} almost doubles.
\begin{itemize}[leftmargin=*]
\item 
{\bf SVHN \citep{netzer2011readingdigits}.}
The data set contains 32$\times$32 colored images of 10 different digits.
In total,
there are 73257 train images and 26032 test images.
\item
{\bf CIFAR-10 \citep{krizhevsky2009learning}.}
The data set contains 32$\times$32 colored images of 10 different objects.
In total,
there are 50000 train images and 10000 test images.
\item
{\bf CINIC-10 \citep{darlow2018cinic}.}
The data set contains 32$\times$32 colored images of 10 different objects.
In total,
there are 90000 train images and 90000 test images.
\end{itemize}

\vspace{.5\baselineskip}
\noindent\textit{Data heterogeneity.}
\prettyref{fig: noniid 20 cifar10} visualizes an example of 20 clients, 
the size of each circle corresponds to the relative proportion of images from a specific class. 
The larger the circle, the greater the share of images associated with that particular class.
Moreover,
$\alpha$ controls the heterogeneity of the data such that a greater $\alpha$ entails a more non-\iid local data distribution and vice versa.

\subsection{Non-stationary client unavailability dynamics}
\label{app: unavail dynamics}
\begin{wrapfigure}[12]{r}{0.4\textwidth} 
\vspace*{-\baselineskip}
\centering
\includegraphics[width=\linewidth, trim = .1cm .2cm 0 .1cm, clip]{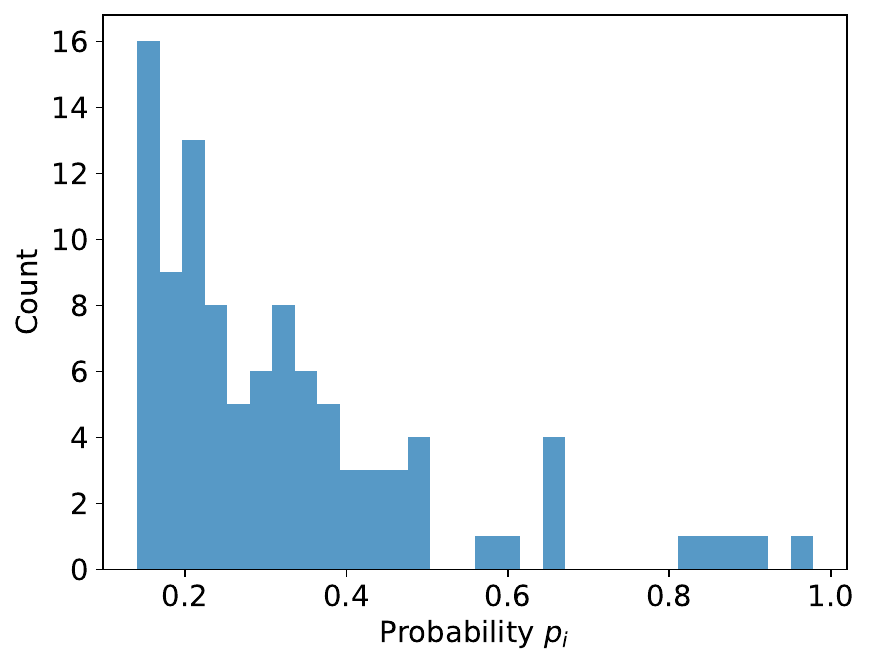}
\caption{
\footnotesize
A histogram of one generated $p_i$'s example
with a total of $m = 100$ clients.
It can be seen that the majority of $p_i$'s are below $0.5$.
}
\label{fig: heterogeneous pit example}%
\end{wrapfigure}
\textbf{Client unavailability dynamics and visualizations.}
As specified in~\prettyref{sec: numerical},
we consider a total of four client unavailable dynamics in the form of $p_i^t = p_i \cdot f_i(t)$,
where $p_i = \iprod{\nu_i}{\phi}$,
$\nu_i \sim \mathsf{Dirichlet} (\alpha)$
and $\phi$ is the distribution to characterize the uneven contributions of each image class.
In detail,
each element $[\phi]_c$ is drawn from a uniform distribution $\mathsf{Uniform} (0, \bm{\Phi}_c)$.
We set $\bm{\Phi}_c = 1$ for the first half image classes and $\bm{\Phi}_{c^\prime} = 0.5$ for the remaining half image classes.
\prettyref{fig: heterogeneous pit example} plots one resulting $p_i$'s example,
wherein $p_i$'s are heterogeneous across clients.

Next,
we formally introduce $f_i(t)$'s under each dynamic in~\prettyref{sec: numerical P=1}.
\begin{itemize}[leftmargin=*]
    \item Stationary: 
    $f_i(t) \triangleq 1$;
    \item Non-stationary with staircase trajectory:
    \[
        f_i(t) \triangleq 
        \indc{t \in [t_0, t_0+P/2)} 
        + 0.4 \cdot \indc{t \in [t_0+P/2, t_0+P)},
    \]
    where $P$ defines a period, $t_0 \in \{0, P, 2P, 3P, \ldots \}$.
    \item Non-stationary with sine trajectory: 
    \[
        f_i(t) \triangleq 
        \gamma \sin (2 \pi/ P \cdot t) + (1 - \gamma),
    \]
    where $\gamma$ signifies the degree of non-stationary.
\end{itemize}
We choose $\gamma = 0.3$ and $P = 20$ for all non-stationary dynamics in~\prettyref{sec: numerical P=1}.
Next,
we visualize the probability trajectories and sampled client availability in~\prettyref{sec: numerical P=1} in~\prettyref{fig: dynamics visualization}.

\begin{figure}
    \centering
    \begin{subfigure}[b]{0.48\textwidth}
    \includegraphics[width=\linewidth]{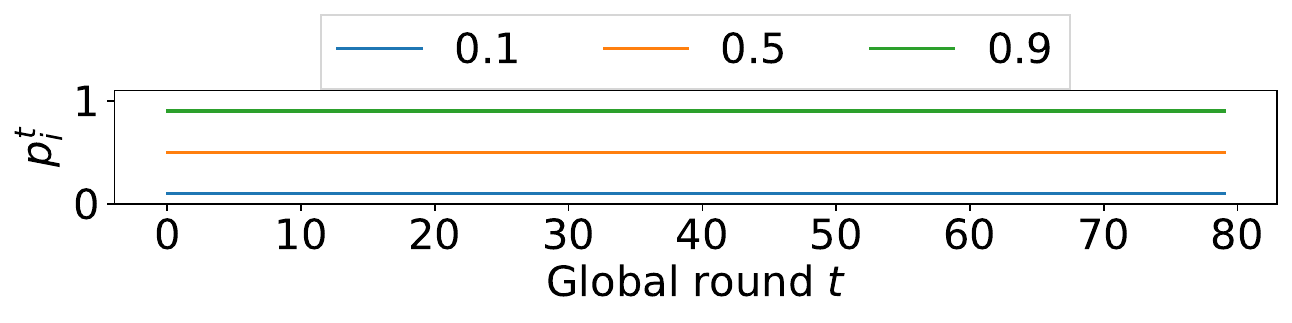}
    \includegraphics[width=\linewidth]{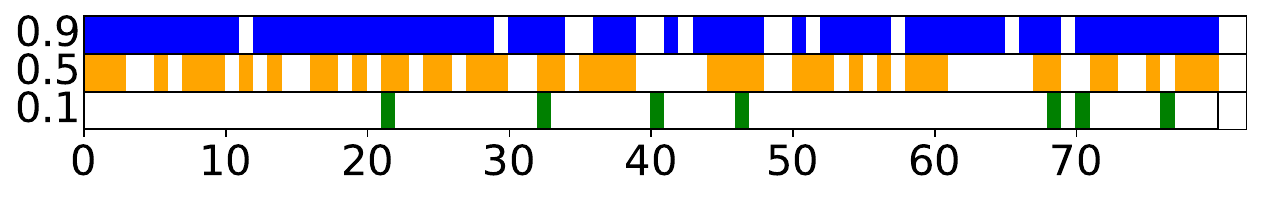}
    \caption{
    \footnotesize
    Stationary}
    \label{fig: constant dynamics}
    \end{subfigure}
    \begin{subfigure}[b]{0.48\textwidth}
    \includegraphics[width=\linewidth]{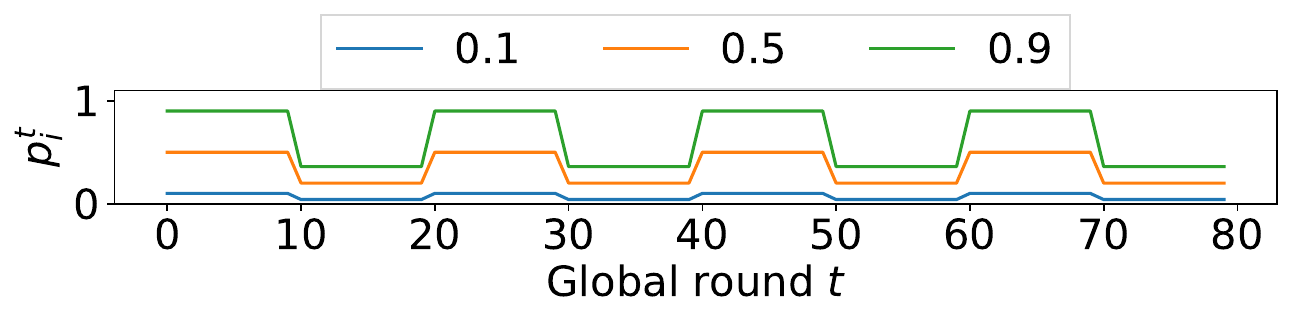}
    \includegraphics[width=\linewidth]{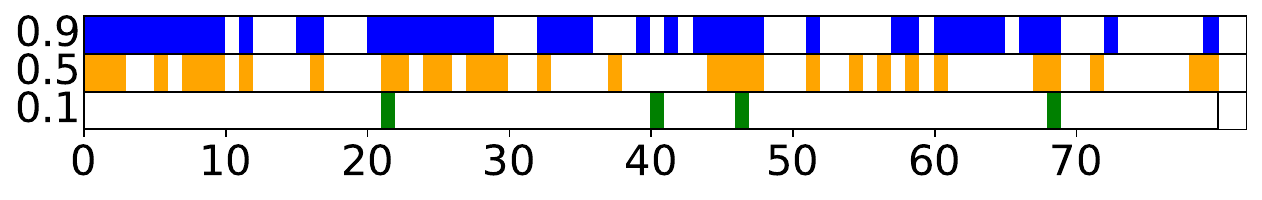}
    \caption{
    \footnotesize
    {\bf Non}-stationary with staircase trajectory}
    \label{fig: staircase dynamics}
    \end{subfigure}
    \begin{subfigure}[b]{0.48\textwidth}
    \includegraphics[width=\linewidth]{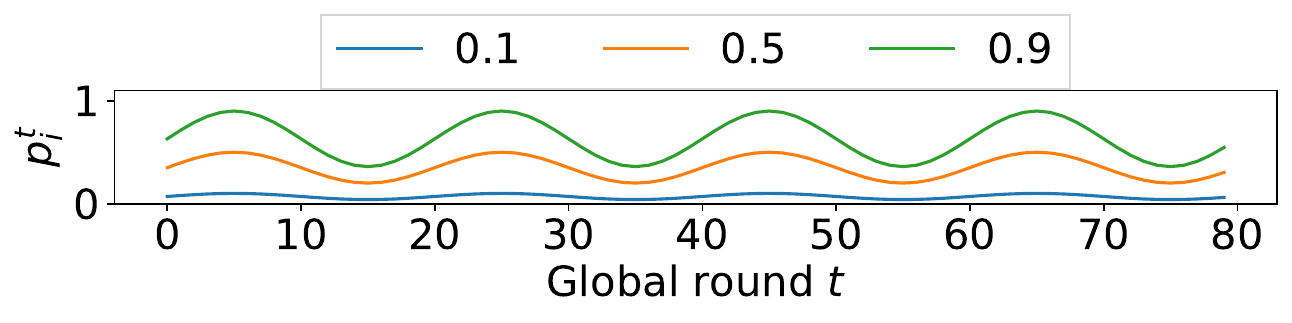}
    \includegraphics[width=\linewidth]{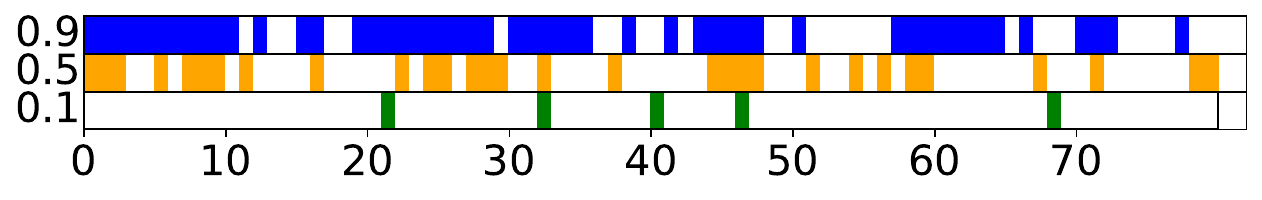}
    \caption{
    \footnotesize
    {\bf Non}-stationary with sine trajectory}
    \label{fig: sine dynamics}
    \end{subfigure}
    \caption{
    \footnotesize
    Examples of client unavailability with probabilistic trajectories.
    The first row in each sub-figure plots the probabilistic trajectory of each dynamics.
    The second row visualizes the simulated client availability 
    by using a colored box to denote that a client is available in that round.
    The y-axis is the base probability $p_i$ to construct $p_i^t$.
    In other words,
    more blank space means that a client is more scarcely available.
    We simulate the cases where $p_i \in \sth{0.1, 0.5, 0.9}$.
    The detailed construction of $p_i^t$ can be found in~\prettyref{app: unavail dynamics}}
    \label{fig: dynamics visualization}
\end{figure}

\subsection{Additional results}

In this section, we provide ablation results on~\FedAPM~with $k = 0$.

\vspace{\baselineskip}
\begin{table}[!t]
\centering
\caption{\footnotesize 
The first round to reach a targeted test accuracy under 
non-stationary of sine trajectory
over 3 random seeds.
We study the first round to reach $1/4$, $1/2$, $3/4$ and $1$ of the best test accuracy of each data set in Table~\ref{tab: exp main text},
which is rounded up to the nearest $10\%$ below for ease of presentation.
In addition,
we sample the mean of test accuracy every 20 global rounds to mitigate %
noisy progress. 
Some algorithms may never attain the targeted accuracy
due to their inferior performance,
where we use ``---'' as a placeholder.
}
\label{tab: slowdown supp}
\resizebox{\linewidth}{!}{
\begin{tabular}{c|cccc|cccc|cccc}
\toprule
{\bf Data sets}&
\multicolumn{4}{|c|}{\bf
SVHN
}&
\multicolumn{4}{|c|}{\bf
CIFAR-10
}
&
\multicolumn{4}{|c}{\bf
CINIC-10
}
\\
\midrule
{\bf Quarters}&
$1/4$&
$1/2$&
$3/4$&
$1$&
$1/4$&
$1/2$&
$3/4$&
$1$&
$1/4$&
$1/2$&
$3/4$&
$1$
\\
\midrule
{\bf Test accuracy}&
$20\%$&
$40\%$&
$60\%$&
$80\%$&
$15\%$&
$30\%$&
$45\%$&
$60\%$&
$10\%$&
$20\%$&
$30\%$&
$40\%$ \\
\midrule
\FedAPM~({\bf ours}, $k = 0$)& 
40 &
120 &
200 &
820&
20 &
60 &
200 &
1360 &
0 &
20 &
120 &
540 
\\
\FedAvg~over {\em active} clients& 
20 &
80 &
160 &
900 &
10 &
20 &
120 &
1060&
0 &
20 &
40 &
800
\\
\FedAvg~over {\em all} clients& 
100 &
420 &
960 &
--- &
20 &
60 &
520 &
---&
0 &
20 &
200 &
---
\\ 
\FedAU & 
60 &
100 &
160 &
840 &
10 &
20 &
100 &
960&
0 &
20 &
80 &
460
\\
\FAST & 
40 &
120 &
200 &
1080&
20 &
40 &
160 &
1300&
0 &
20 &
60 &
540
\\
\noalign{\vspace{.5mm}}
\midrule
\noalign{\vspace{.5mm}}
\FedAvg~with {\em known} $p_i^t$'s 
& 
20 &
40 &
100 &
320 &
10 &
20 &
140 &
620 &
0 &
20 &
40 &
400
\\
\MIFA~({\em memory aided}) & 
20 &
80 &
140&
600 &
10 &
20 &
80&
700 &
0 &
20 &
40&
240 
\\
\bottomrule
\end{tabular}}
\end{table}

\noindent\textbf{Staleness studies.}
\prettyref{tab: slowdown supp} 
illustrates the first round to reach a targeted test accuracy under non-stationary client availability with sine trajectory.
Specifications can be found in the caption.
It can be easily checked that,
during the initial stage (the first three quarters),
\FedAPM~slightly lags behind~\FedAvg~over active
clients.
However,
when reaching the final stage (the last quarter),
\FedAPM~attains the target accuracy in a comparable or lower number of rounds to~\FedAvg~over active clients in the evaluations on SVHN and CINIC-10 data sets.
The slowdown of~\FedAPM~on CIFAR-10 data set is worth further investigation.
In general, we arrive numerically at the conclusion that the staleness incurred by implicit gossiping in~\FedAPM~is mild.

\begin{table}[!t]
\centering
\caption{\footnotesize 
Results after different parameter $\gamma$.
$ p_i^t = p_i \cdot ( \gamma \sin(2 \pi / P  \cdot t) + (1 - \gamma))$.}
\label{tab: exp supp flu gamma}
\resizebox{\linewidth}{!}{
\begin{footnotesize}
\begin{tabular}{c|c|p{1.8cm} p{1.8cm}|p{1.8cm} p{1.8cm}|p{1.8cm} p{1.8cm}}
    \toprule
    \multirow{2}{*}{\begin{tabular}{@{}c@{}}{\bf Unavailable} \\ {\bf Dynamics} \end{tabular}} &
    {\bf Data sets} &
    \multicolumn{2}{c|}{\bf $\gamma = 0.3$} & 
    \multicolumn{2}{c|}{\bf $\gamma = 0.2$} & 
    \multicolumn{2}{c}{\bf $\gamma = 0.1$} \\
    \cline{2-8}
    & 
    {\bf Algorithms}&
    \multicolumn{1}{c}{\bf Train} &
    \multicolumn{1}{c|}{\bf Test} &
    \multicolumn{1}{c}{\bf Train} &
    \multicolumn{1}{c|}{\bf Test} &
    \multicolumn{1}{c}{\bf Train} &
    \multicolumn{1}{c}{\bf Test} \\
    \hline
    \multirow{6}{*}{
\begin{tabular}{@{}c@{}} 
\addlinespace[1ex]
{\bf Non}-stationary  \\ 
({\bf Sine})\\
\adjustbox{width=0.12\linewidth}{\begin{tikzpicture}
\begin{axis}[
    xmin=0, xmax=12*pi,
    ymin=-.6, ymax=1.3,
    domain=1.5*pi:11.5*pi,
    samples=500,
    axis lines=middle,
    hide y axis, %
    xtick=\empty, 
    ytick=\empty,
    ]
    \addplot[blue] {.3*sin(deg(.8*x)) + .6};
    \node[below left] at (axis cs:12*pi,-0.05) {\scalebox{3}{$ 0$}};  
    \node[below right] at (axis cs:0,1) {\scalebox{3}{$p_i^t$}};  
\end{axis}
\end{tikzpicture}}
\end{tabular}} 
& 
\FedAPM~({\bf ours}, $k = 0$)& 
{\bf 85.7} $\pm$ 0.9 \%&
{\bf 85.6} $\pm$ 0.9 \%&
{\bf 85.7} $\pm$ 0.5 \%&
{\bf 85.7} $\pm$ 0.5 \%&
{\bf 85.8} $\pm$ 0.6 \%&
{\bf 85.7} $\pm$ 0.7 \%
\\
& 
\FedAvg~over {\em active} & 
82.1  $\pm$ 1.1 \%&
82.0  $\pm$ 1.3 \%&
82.0  $\pm$ 1.2 \%&
81.9  $\pm$ 1.2 \%&

82.3  $\pm$ 0.9 \%&
82.2  $\pm$ 1.0 \%
\\
& 
\FedAvg~over {\em all} & 
71.3 $\pm$ 2.5 \%&
71.3 $\pm$ 2.8 \%&
73.2  $\pm$ 2.5 \%&
73.2  $\pm$ 2.8 \%&
74.0  $\pm$ 2.1 \%&
74.9  $\pm$ 2.4 \%
\\
& 
\FedAU & 
\underline{82.5} $\pm$ 1.4 \%&
\underline{82.5} $\pm$ 1.3 \%&
\underline{83.5}  $\pm$ 0.3 \%&
\underline{83.4}  $\pm$ 0.4 \%&
\underline{83.7}  $\pm$ 0.3 \%&
\underline{83.6}  $\pm$ 0.3 \%
\\
& 
\FAST & 
82.3  $\pm$ 1.0 \%&
82.3  $\pm$ 1.0 \%&
82.3 $\pm$ 0.9 \%&
82.6  $\pm$ 0.8 \%&
82.9  $\pm$ 0.7 \%& 
82.9 $\pm$ 0.6 \%
\\
\noalign{\vspace{.5mm}}
\cline{2-8}
\noalign{\vspace{.5mm}}
& 
\FedAvg~with {\em known} $p_i^t$'s  & 
86.3  $\pm$ 1.0 \%&
86.0  $\pm$ 1.0 \%&
86.2  $\pm$ 1.2 \%&
86.0  $\pm$ 1.4 \%&
86.4  $\pm$ 0.9 \%&
86.0  $\pm$ 0.8 \%
\\
& 
\MIFA~({\em memory aided}) & 
84.2  $\pm$ 0.4 \%&
84.1  $\pm$ 0.4 \%&
84.6  $\pm$ 0.1 \%&
84.5  $\pm$ 0.1 \%&
84.6  $\pm$ 0.1 \%&
84.4  $\pm$ 0.1 \%
\\
\bottomrule
\end{tabular}
\end{footnotesize}
}
\end{table}
\begin{table}[!t]
\centering
\caption{\footnotesize 
Results after different $\mathsf{Dirichlet}$ parameter $\alpha$.
$ p_i^t = p_i (\gamma \sin(2 \pi / P  \cdot t) + (1 - \gamma))$.
}
\label{tab: exp supp dir alpha}
\resizebox{\linewidth}{!}{
\begin{footnotesize}
\begin{tabular}{c|c|p{1.8cm} p{1.8cm}|p{1.8cm} p{1.8cm}|p{1.8cm} p{1.8cm}}
    \toprule
    \multirow{2}{*}{\begin{tabular}{@{}c@{}}{\bf Unavailable} \\ {\bf Dynamics} \end{tabular}} &
    {\bf Data sets} &
    \multicolumn{2}{c|}{\bf $\alpha = 0.05$} & 
    \multicolumn{2}{c|}{\bf $\alpha = 0.1$} & 
    \multicolumn{2}{c}{\bf $\alpha = 1.0$} \\
    \cline{2-8}
    & 
    {\bf Algorithms}&
    \multicolumn{1}{c}{\bf Train} &
    \multicolumn{1}{c|}{\bf Test} &
    \multicolumn{1}{c}{\bf Train} &
    \multicolumn{1}{c|}{\bf Test} &
    \multicolumn{1}{c}{\bf Train} &
    \multicolumn{1}{c}{\bf Test} \\
    \hline
    \multirow{6}{*}{
\begin{tabular}{@{}c@{}} 
\addlinespace[1ex]
{\bf Non}-stationary  \\ 
({\bf Sine}) \\
\adjustbox{width=0.12\linewidth}{\begin{tikzpicture}
\begin{axis}[
    xmin=0, xmax=12*pi,
    ymin=-.6, ymax=1.3,
    domain=1.5*pi:11.5*pi,
    samples=500,
    axis lines=middle,
    hide y axis, %
    xtick=\empty, 
    ytick=\empty,
    ]
    \addplot[blue] {.3*sin(deg(.8*x)) + .6};
    \node[below left] at (axis cs:12*pi,-0.05) {\scalebox{3}{$ 0$}};  
    \node[below right] at (axis cs:0,1) {\scalebox{3}{$p_i^t$}};  
\end{axis}
\end{tikzpicture}}
\end{tabular}} 
& 
\FedAPM~({\bf ours}, $k = 0$)& 
{\bf 82.5} $\pm$ 2.1 \%&
{\bf 82.5} $\pm$ 2.4 \%&

{\bf 85.7} $\pm$ 0.9 \%&
{\bf 85.6} $\pm$ 0.9 \%&

{\bf 90.6} $\pm$ 0.2 \%&
{\bf 89.7} $\pm$ 0.3 \%
\\
& 
\FedAvg~over {\em active} & 
78.9  $\pm$ 1.6 \%&
78.5  $\pm$ 1.8 \%&

82.1  $\pm$ 1.1 \%&
82.0  $\pm$ 1.3 \%&

88.3  $\pm$ 0.1 \%&
87.5  $\pm$ 0.1 \%
\\
& 
\FedAvg~over {\em all} & 
58.5  $\pm$ 3.0 \%&
58.5  $\pm$ 3.8 \%&

71.3 $\pm$ 2.5 \%&
71.3 $\pm$ 2.8 \%&

82.0  $\pm$ 0.7 \%&
81.9  $\pm$ 0.6 \%
\\
& 
\FedAU & 
\underline{79.5}  $\pm$ 1.6 \%&
\underline{79.5}  $\pm$ 1.7 \%&

\underline{82.5} $\pm$ 1.4 \%&
\underline{82.5} $\pm$ 1.3 \%&

\underline{88.4}  $\pm$ 0.1 \%&
\underline{87.6}  $\pm$ 0.2 \%
\\
& 
\FAST & 
78.9 $\pm$ 1.3 \%&
78.9 $\pm$ 1.3 \%&
82.3  $\pm$ 1.0 \%&
82.3  $\pm$ 1.0 \%&
87.6  $\pm$ 0.1 \%& 
87.0  $\pm$ 0.1 \%
\\
\noalign{\vspace{.5mm}}
\cline{2-8}
\noalign{\vspace{.5mm}}
& 
\FedAvg~with {\em known} $p_i^t$'s  & 
84.2  $\pm$ 1.0 \%&
83.5  $\pm$ 1.0 \%&
86.3  $\pm$ 1.0 \%&
86.0  $\pm$ 1.0 \%&
91.5  $\pm$ 0.3 \%&
90.5  $\pm$ 0.1 \%
\\
& 
\MIFA~({\em memory aided}) & 
82.6  $\pm$ 0.1 \%&
82.6  $\pm$ 0.0 \%&

84.2  $\pm$ 0.4 \%&
84.1  $\pm$ 0.4 \%&

88.4  $\pm$ 0.1 \%&
87.5  $\pm$ 0.1 \%
\\
\bottomrule
\end{tabular}
\end{footnotesize}
}
\end{table}

\vspace{\baselineskip}
\noindent\textbf{Impact of system-design parameters.}
In this part,
we study the impact of system-design parameter
including 
the degree of non-stationarity $\gamma$
and data heterogeneity $\alpha$
under non-stationary with sine trajectory.
The results are in~\prettyref{tab: exp supp flu gamma} and~\prettyref{tab: exp supp dir alpha}.
Overall,~\FedAPM~keeps outperforming the algorithms not assisted by memories or known statistics.

In~\prettyref{tab: exp supp dir alpha}, clients' local data becomes more heterogeneous when $\alpha$ increases.
We can see a clear increasing trend in accuracy.
However,
\FedAPM~remains to attain the best accuracies both train and test when compared to the algorithms not aided by heavy memory or known statistics.
Moreover,
it outperforms~\MIFA,
which consumes a lot of storage space,
when $\alpha = 0.1$ and $1.0$.
The observations confirm the practicality of~\FedAPM.

\end{document}